\documentclass[lettersize,journal]{IEEEtran}
\usepackage{amsmath,amsfonts,amsthm}
\usepackage{array}
\usepackage{url}
\usepackage{amssymb}
\usepackage{multirow}
\usepackage{upgreek}
\usepackage{graphicx}
\usepackage{etoolbox}

\usepackage{cite}
\usepackage[resetlabels,labeled]{multibib}
\newcites{S}{Appendix References}

\usepackage{pifont}

\def\R#1{(\ref{#1})}

\usepackage{wasysym}

\usepackage[ruled,vlined]{algorithm2e}
\usepackage[noend]{algpseudocode}
\usepackage[caption=false,font=normalsize,labelfont=sf,textfont=sf]{subfig}
\usepackage[full]{textcomp}
\usepackage{stfloats}
\usepackage{verbatim}
\usepackage{booktabs}

\usepackage{subcaption}
\usepackage{tikz}
\usepackage{enumitem}
\usepackage{kotex}

\newcommand{\cmark}{\ding{51}}%
\newcommand{\xmark}{\ding{55}}%

\usepackage{stix}
\usepackage[dvipsnames]{xcolor}

\definecolor{FeatGreen}{RGB}{0,128,0}
\colorlet{FeatRed}{red}
\colorlet{MeanOrange}{orange}
\colorlet{WBlue}{blue}

\newcommand{\symSq}[1]{\tikz[baseline=-0.5ex]\draw[#1,thick] (-0.1,-0.1) rectangle (0.1,0.1);}
\newcommand{\symDot}[1]{\tikz[baseline=-0.5ex]\draw[#1,fill=#1] (0,0) circle[radius=0.1];}
\newcommand{\symTri}[1]{\tikz[baseline=-0.5ex]\draw[#1,thick] (-0.1,0)--(0.1,0)--(0,0.17)--cycle;}

\def\R#1{(\ref{#1})}
\newcommand*{\cO}{\mathcal{O}}
  
\newcommand{\be}[1]{\begin{equation} #1 \end{equation}}

\newcommand{\ea}[1]{\begin{align} #1 \end{align}}

\newcommand{\nn}{\nonumber}

\newtheorem{theorem}{Theorem}

\usepackage[pagebackref,breaklinks,colorlinks]{hyperref}
\usepackage{xr-hyper}  
\begin{document}

\title{Equilibrium
contrastive learning for imbalanced image classification}

\author{
Sumin Roh, Harim Kim, Ho Yun Lee$^{\dagger}$, and Il Yong Chun$^{\dagger}$

\thanks{
$^{\dagger}$\emph{Corresponding authors.}
The work was supported in part by NRF Grant RS-2023-00213455 funded by MSIT, 
the Digital Therapeutics Development and Demonstration Support Program Grant H0601-24-1023 funded by MSIT and NIPA,
the BK21 FOUR Project, 
IITP Grant RS-2019-II190421 (AI Graduate School Support Program (Sungkyunkwan University)) 
funded by MSIT, 
KIAT Grant RS-2024-00418086 (HRD Program for Industrial Innovation) funded by MOTIE, 
IBS-R015-D1,
Future Medicine 20*30 Project of the Samsung Medical Center (Grant No.~SMO1240791), 
IITP Grant RS-2021-II212068 funded by MSIT and Artificial Intelligence Innovation Hub, 
SKKU-SMC and SKKU-KBSMC Future Convergence Research Program grants.
}

\thanks{
Sumin Roh is with the Department of Electrical and Computer Engineering (ECE), Sungkyunkwan University (SKKU), Suwon 16419, South Korea (email: 
\href{mailto:sumsumin@g.skku.edu}{\tt sumsumin@g.skku.edu}).
Harim Kim is with the Department of Radiology, Samsung Medical Center (SMC), Seoul 06351, South Korea (email: 
\href{mailto:k3541729@gmail.com}{\tt k3541729@gmail.com}).
Ho Yun Lee is with the Department of Radiology, SMC, Seoul 06351, South Korea, and with School of Medicine, SKKU, Suwon 16419, South Korea (email: \href{mailto:hoyunlee96@gmail.com}{\tt hoyunlee96@gmail.com}).
Il Yong Chun is with the Departments of ECE, Artificial Intelligence, Advanced Display Engineering, and Semiconductor Convergence Engineering, Display Convergence Engineering, SKKU, Suwon 16419, South Korea, and also with the Center for Neuroscience Imaging Research, Institute for Basic Science, Suwon 16419, South Korea (email: \href{mailto:iychun@skku.edu}{\tt iychun@skku.edu}).

}
}

\markboth{Journal of \LaTeX\ Class Files,~Vol.~14, No.~8, February~2026}%
{Shell \MakeLowercase{\textit{et al.}}: A Sample Article Using IEEEtran.cls for IEEE Journals}

\maketitle

\begin{abstract}
The data imbalance issue is critical in image classification.
Contrastive learning (CL) is a predominant technique in image classification,
but they showed limited classification performance when the class distribution in a training dataset is skewed (i.e., an imbalanced dataset). 
Recently, several supervised CL methods have been proposed particularly to promote an ideal regular simplex geometric configuration in the representation space---characterized by intra-class feature collapse and uniform inter-class mean spacing---especially for imbalanced datasets.
In particular, existing prototype-based methods include class representatives, i.e., prototypes, as additional samples to achieve a more balanced treatment of all classes.
However, the existing supervised CL methods for imbalanced datasets still suffer from two major limitations.
First, they do \emph{not} consider the alignment between the class means/prototypes and linear classifiers, which could lead to poor generalization.
Second, existing prototype-based methods treat class prototypes as \emph{only} one additional sample per class, making their influence depend on the number of class instances in a batch and causing unbalanced contributions across classes.
To address these limitations, 
we propose Equilibrium Contrastive Learning (ECL), a supervised CL framework designed to promote geometric equilibrium in the representation space, where class features, means, and classifiers are harmoniously balanced under data imbalance.
The proposed ECL framework uses two main strategies.
First, ECL promotes the representation geometric equilibrium (i.e., a regular simplex geometry characterized by collapsed class samples and uniformly distributed class means), 
while balancing the contributions of class-average features and class prototypes.
Second, ECL establishes a classifier-class center geometric equilibrium by aligning classifier weights and class prototypes in the representation space.
We ran experiments with three long-tailed natural image datasets, 
the CIFAR-10-LT, CIFAR-100-LT and ImageNet-LT benchmark datasets,
and the two imbalanced medical image classification datasets,
the ISIC 2019 benchmark and our constructed lung cancer chest computed tomography dataset.
Results with the five imbalanced datasets show that ECL outperforms existing state-of-the-art supervised CL methods designed for imbalanced/long-tailed classification.
Codes are available at \href{https://github.com/suminRoh/ECL-main}{this link}.
\end{abstract}


\begin{figure}[t]
    \centering
    \includegraphics[width=\columnwidth]{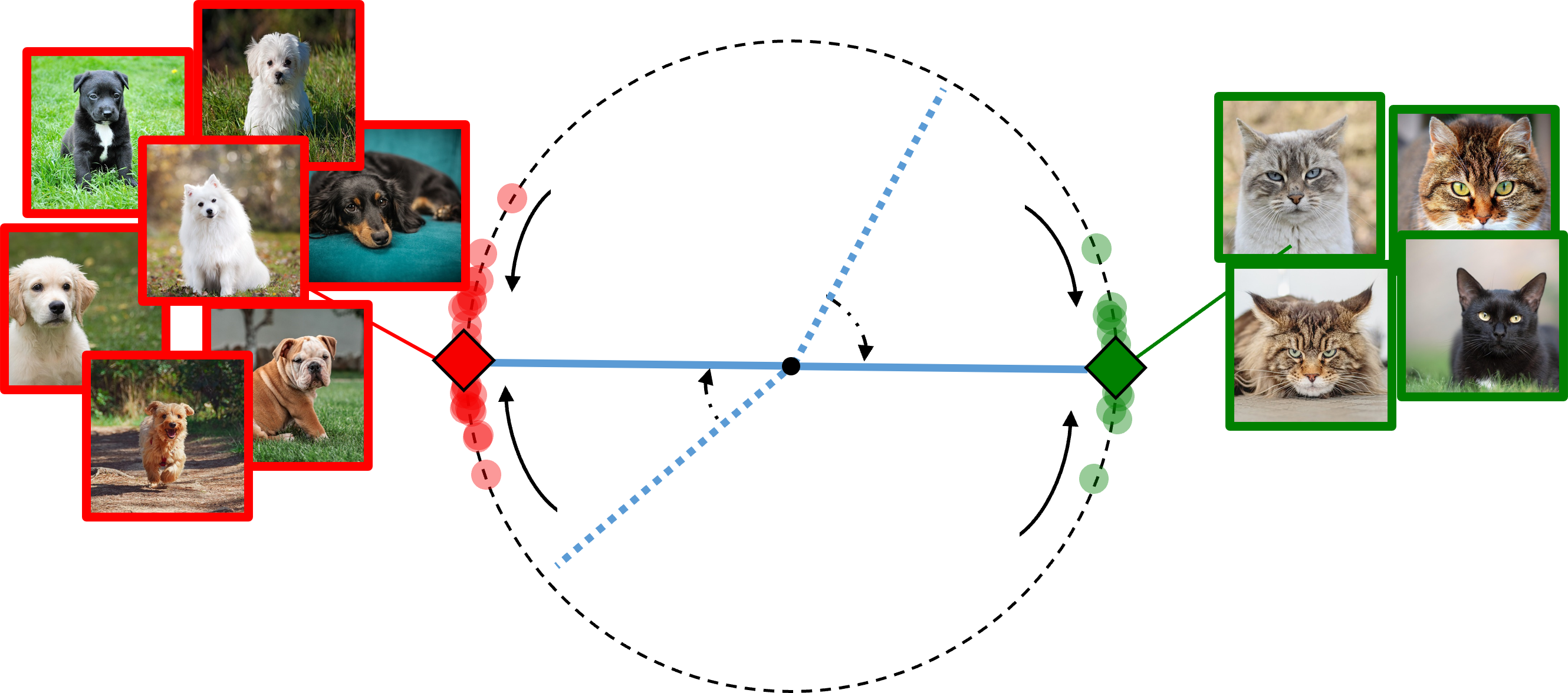}
    
    \caption{
    Illustration of {\bfseries E}quilibrium {\bfseries C}ontrastive {\bfseries L}earning ({\bfseries ECL}).
    ECL promotes the following three key geometric properties that collectively establish geometric equilibrium in the representation space:
    \textit{1)} the collapse of within-class representations (denoted by $\mdsmblkcircle$) to their class means (denoted by $\lozenge$), 
    \textit{2)} equidistant spacing among class means, and
    \textit{3)} alignment between linear classifier weights (denoted by skyblue sticks) and class means.
    It can simplify the geometries in representation and classifier, leading to simple nearest-centroid classifications in the representation space. 
    The simple geometries are useful for improving the generalization capability.}
    \label{fig:configuration}
\vspace{-1pc} 
\end{figure}

\section{Introduction}
\label{sec:intro}

\IEEEPARstart{I}{mbalanced} distributions can easily be found in real-world datasets.
Any dataset that exhibits an unequal distribution among its classes is referred to as an imbalanced dataset \cite{imbalance}.
Image classification becomes challenging when the dataset has imbalanced distributions.
Contrastive learning (CL) is a predominant technique in image classification \cite{simclr,kcl,supcon} and several supervised CL methods have been proposed to handle imbalanced datasets \cite{GLMC,TSC,BCL}.
In particular,
\cite{TSC,BCL}
promote in the representation space that all instances with the same label collapse to points and these points are equidistant from each other, i.e., \emph{intra-class feature collapse} and \emph{uniform inter-class mean spacing} \cite{neuralcollapse}, respectively.
The existing prototype-based CL methods \cite{BCL, PaCo, GPaco,TSC} aim to promote the aforementioned geometric properties by incorporating class prototypes as additional samples, 
which enables a more balanced consideration of all classes under imbalanced settings.
However, the existing supervised CL methods have two limitations.
First, existing CL methods underestimate the alignment between linear classifiers and the class center representations \cite{neuralcollapse}, 
leading to limited generalization performances in imbalanced test sets.
Second, the existing prototype-based CL methods \cite{BCL, PaCo, GPaco, TSC} treat each prototype merely as a single additional sample per class, causing its influence to diminish as the number of class instances in a batch increases. 
Consequently, the prototype signal becomes weak, making the objectives sensitive to mini-batch composition.

To overcome these limitations, 
this paper proposes {\bfseries E}quilibrium {\bfseries C}ontrastive {\bfseries L}earning ({\bfseries ECL}), a new supervised end-to-end (E2E) framework that promotes geometric equilibrium in the representation space by harmoniously balancing class features, means, and classifier weights.
We propose two complementary strategies in the ECL framework.
First, we propose the {\bfseries B}alanced {\bfseries C}lass-wise {\bfseries ECL} ({\bfseries BC-ECL}) formulation that promotes representation geometric equilibrium---a regular simplex geometry characterized by intra-class feature collapse and uniform inter-class mean spacing---in a batch-invariant manner by balancing the contributions of class-average features and prototypes.
Second, we propose the {\bfseries C}lassifier-class {\bfseries C}enter {\bfseries G}eometric {\bfseries E}quilibrium ({\bfseries CC-GE}) promoting loss that aligns class prototypes and classifier weights.

Combined all together, 
proposed ECL can promote the symmetric and very simple geometry in a batch-invariant manner, to improve generalization performance in imbalanced learning.
See its illustration in Fig.~\ref{fig:configuration}.

We ran experiments 
with three long-tailed natural image datasets,
the CIFAR-10-LT, CIFAR-100-LT, and ImageNet-LT benchmark datasets \cite{cifar},
and with two imbalanced medical image datasets,
the International Skin Imaging Collaboration (ISIC) 2019 benchmark dataset \cite{ISIC} for skin lesion image classification 
and our constructed Lung Cancer chest Computed Tomography (LCCT) dataset to determine the risk of cancer recurrence.
Our numerical experiments demonstrate that the proposed ECL method achieves outperforming performance compared to existing state-of-the-art (SOTA) supervised CL methods for handling imbalanced datasets.

We analyzed different supervised CL methods with the three perspectives,
\textit{1)} how well within-class representations collapse to their class means, 
\textit{2)} how far are class means between different classes,
and 
\textit{3)} how well last-layer classifiers are aligned with the class prototypes/means.
We observed that 
proposed ECL improves these key metrics, compared to those from the SOTA methods.

Our contributions can be summarized as follows: 
\begin{itemize}
    \item The proposed ECL framework promotes three key geometric properties in a batch-invariant manner: 
    \textit{1)} intra-class feature collapse, 
    \textit{2)} uniform inter-class mean spacing, and 
    \textit{3)} alignment between classifier weights and class center representations.
    Together, ECL establishes simplified and stable geometric structures jointly among class features, means, and classifier weights, ultimately improving generalization in imbalanced learning.

    \item Proposed ECL achieves the SOTA performance on five imbalanced datasets: ImageNet-LT, CIFAR-10-LT, CIFAR-100-LT, ISIC 2019, and LCCT.
\end{itemize}

\section{Related work}

This section reviews existing supervised CL methods that are proposed to handle imbalanced/long-tailed datasets. 
\cite{PaCo} proposes a two-stage method that introduces parametric class-wise learnable centers to rebalance the learning process from an optimization perspective.
The method improves the gradient balance across classes, making the model more attentive to minority classes.
Global and Local Mixture Consistency Cumulative learning (GLMC) \cite{GLMC} is an E2E method that proposes a global and local mixture consistency loss to improve the robustness of the feature extractor, and a cumulative head-tail soft label reweighted loss to mitigate the majority class bias problem.
Targeted Supervised Contrastive learning (TSC) \cite{TSC} is a two-stage method that introduces
uniformly distributed targets to each class during training to prevent that the representation space is biased by majority classes.
BCL \cite{BCL} is an E2E method that proposes a balanced CL loss to form a regular simplex configuration in the representation space.
BCL additionally employs the logit compensation approach to learn an unbiased classifier.
Parametric Contrastive learning (PaCo) \cite{PaCo} is an E2E method that introduces learnable class-wise centers into the contrastive loss to rebalance optimization across head and tail classes.
Generalized PaCo (GPaCo) \cite{GPaco} is an E2E method that generalizes PaCo by removing the momentum encoder, yielding a single-encoder parametric contrastive framework with learnable class-wise centers for improving optimization stability and generalization.
Probabilistic Contrastive learning (ProCo)~\cite{Probco} is an end-to-end method that models each class using a von Mises–Fisher distribution, allowing the model to learn from class-level distributions rather than relying solely on mini-batch instances.
By comparing feature representations against these class-wise distributions within the contrastive loss, ProCo achieves effective learning without requiring large mini-batch sizes.

Beyond imbalanced learning, Contrast your Neighbors (CoNe) \cite{cone} is an E2E supervised CL method that supervises each instance using its neighbors in the representation space together with a distributional-consistency regularizer, encouraging similar instances to share similar predictive distributions.

TSC \cite{TSC} and BCL \cite{BCL} aims to promote a geometric equilibrium among representations from different classes---namely, intra-class feature collapse and uniform inter-class mean spacing---under imbalanced datasets.
To ensure balanced consideration of all classes in the presence of class imbalance,
TSC \cite{TSC}, BCL \cite{BCL}, PaCo \cite{PaCo}, and GPaCo \cite{GPaco} incorporate class prototypes as extra samples.
The proposed ECL framework differs from the existing methods in the following key aspects:
\begin{itemize}
\item Beyond promoting geometric equilibrium among representations, 
ECL emphasizes the importance of CC-GE---namely, alignment between classifier weights with class-center representations---an aspect overlooked in the existing supervised CL methods \cite{GLMC,cone,PaCo,GPaco,Probco,BCL,TSC}.
This alignment is crucial for improving generalization under class imbalance.
\item ECL promotes representation geometric equilibrium in a batch-invariant manner by balancing class-average features and prototypes contributions.
In contrast, in TSC~\cite{TSC}, BCL~\cite{BCL}, PaCo~\cite{PaCo}, and GPaCo~\cite{GPaco}, 
the influence of prototypes varies across classes depending on their sample counts, potentially causing the geometric structure to fluctuate with mini-batch composition.
\end{itemize}

\section{Methods} 

This section describes the proposed ECL framework.
Section~\ref{sec:prelim} provides some preliminaries to facilitate the understanding of the proposed method.
Section~\ref{preliminary analysis} presents two key insights that motivate the proposed method.
Section~\ref{ECL} presents the details of the proposed ECL framework.

\subsection{Preliminaries} \label{sec:prelim}

The goal of image classification is to learn a mapping from the input space $\mathcal{X}$ to the target space 
$\mathcal{Y} = \{1, 2, \ldots, C\}$, 
where a training dataset is given by $\mathcal{D} = \{ (\mathbf{x}_i, y_i) :  i \in \mathcal{I} \}$ with $\mathbf{x}_i \in \mathcal{X}$ and $y_i \in \mathcal{Y}$ denoting an input image and its corresponding class label, respectively, and $C$ is the number of classes, and $\mathcal{I} = \{ 1,\ldots,N \}$.
The quality of this mapping critically depends on the geometric structure formed by learned representations and a classifier.
To address this challenge, we propose ECL that promotes three key properties that simplify and stabilize the geometry in the representation space, particularly for handling imbalanced datasets.
Under balanced datasets, such configurations empirically confer important benefits, such as improved generalization and robustness~\cite{neuralcollapse}.
We describe the three key properties below:

\begin{itemize}

\item (a) \textbf{Intra-class feature collapse.}
As training progresses,
all the within-class samples collapse to their class means in a representation space in the end.
See the illustration in Fig.~\ref{fig:properties}(a).

\item (b) \textbf{Inter-class mean geometric equilibrium}.
Distances and angles between each class means become equal forming a symmetric representation structure at the terminal training phase.
See the illustration in Fig.~\ref{fig:properties}(b).

\item (c) \textbf{CC-GE: Classifier-class Center Geometric Equilibrium}.
Under a balanced dataset, although the class means and linear classifiers live in dual vector spaces, they can converge to each other with rescaling \cite{neuralcollapse}.
This self-duality convergence property---referred to as the CC-GE---can be formally written as follows:
\begin{equation}
    \left 
    \| 
    \frac{\mathbf{W}^\top}{\left \| \mathbf{W} \right \|_\mathrm{F}}- \frac{\mathbf{M}}{\| \mathbf{M} \|_\mathrm{F}} 
    \right \|_\mathrm{F} 
    \rightarrow  0,
\label{eq:NC3}
\end{equation}
where $\mathbf{W}^\top := [\mathbf{w}_1, \ldots, \mathbf{w}_C]$ is the classifier weights matrix by stacking the weights of classifier and
$
\mathbf{M} 
:= [\boldsymbol{\upmu}_{1} - \dot{\boldsymbol{\upmu}}, \ldots, \boldsymbol{\upmu}_{c} - \dot{\boldsymbol{\upmu}}]
$
is the class means matrix by stacking the class means,
in which 
$\mathbf{w}_c$ is the $c$th class weight vector of a classifier,
$\boldsymbol{\upmu}_{c}$ is the $c$th class mean vector
and $\dot{\boldsymbol{\upmu}}$ is the average of all class means.
See the illustration in Fig.~\ref{fig:properties}(c).
In analyzing different methods, we use (\ref{eq:NC3}) to evaluate how well the self-duality is induced.

\end{itemize}

\begin{figure}[t!]
\centering
\small\addtolength{\tabcolsep}{-5pt}
\renewcommand{\arraystretch}{1}

\begin{tabular}{ccc}
    \centering
    \includegraphics[height=0.34\columnwidth]{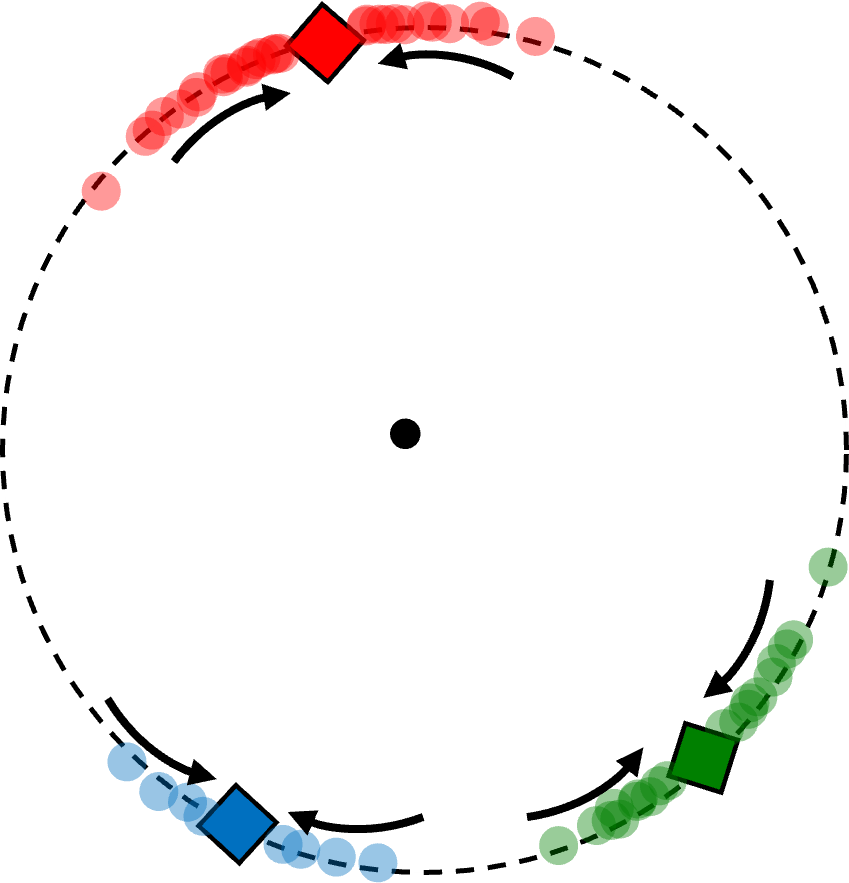} & 
    \includegraphics[height=0.34\columnwidth]{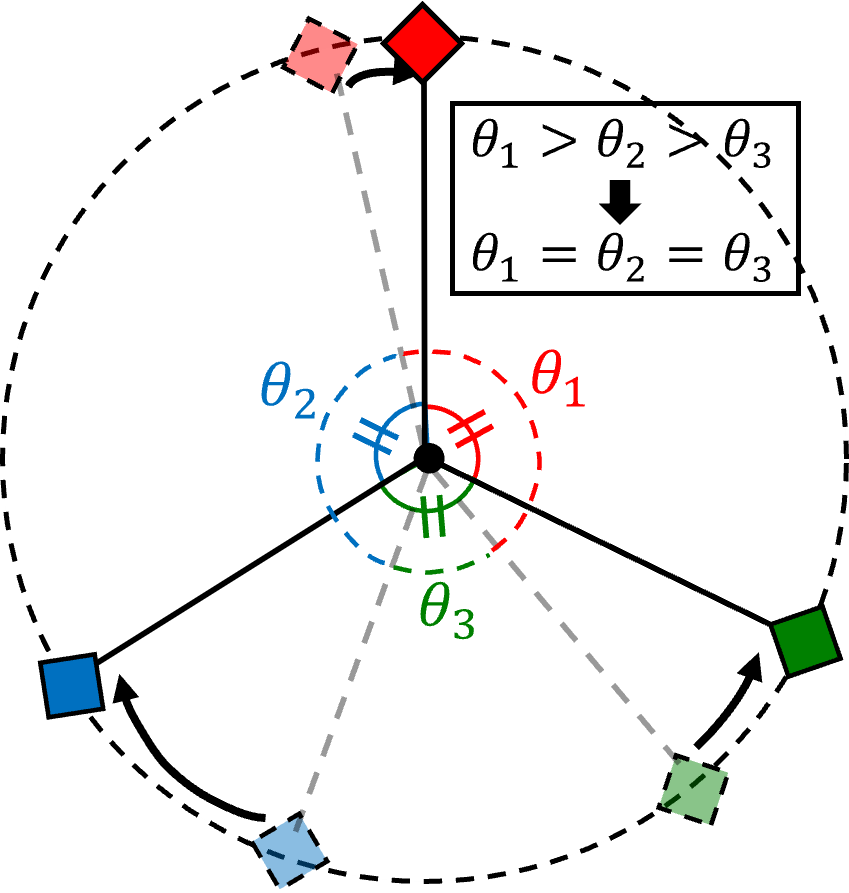} & 
    \includegraphics[height=0.34\columnwidth]{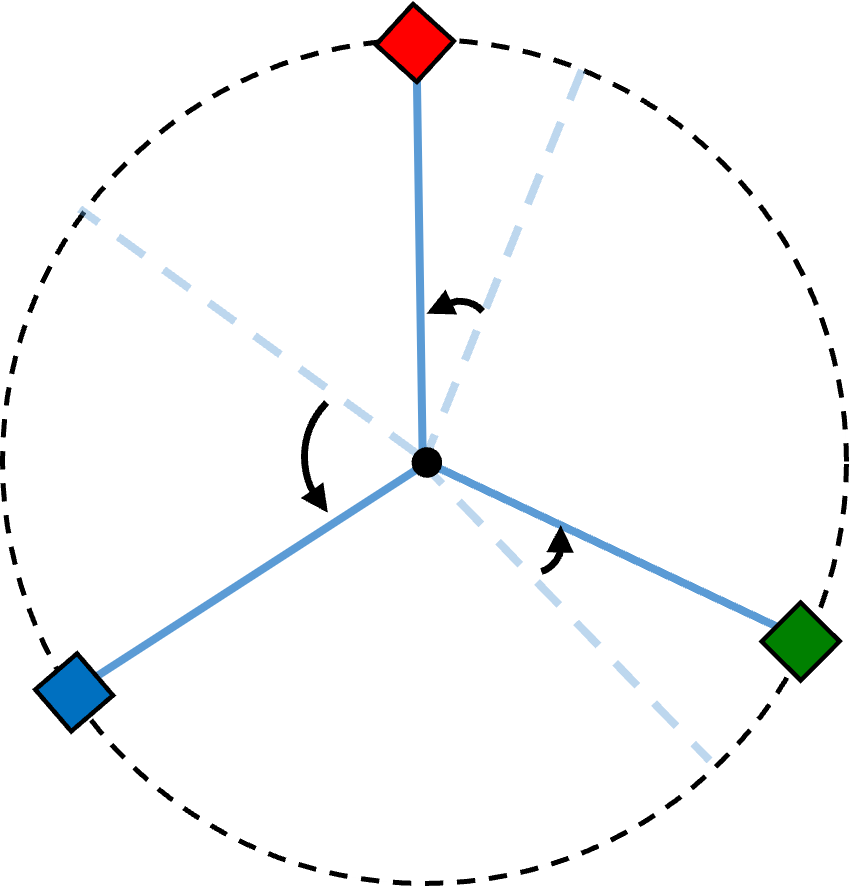} \\
    (a)  & (b) & (c)
\end{tabular}

\caption{
    Illustration of three key properties to simplify geometries via ECL under class imbalance (red and green denote features from different majority classes, and blue denotes features from the minority class).
    (a) {\bfseries Intra-class feature collapse:} all the intra-class representations (denoted by $\mdsmblkcircle$) collapse to their corresponding class means (denoted by $\lozenge$).
    (b) {\bfseries Inter-class mean geometric equilibrium:} distances and angles among class means become equal, forming a symmetric geometric structure. 
    (c) {\bfseries CC-GE: Classifier–class center geometric equilibrium:} the linear classifiers (denoted by skyblue sticks) and class means converge to each other with rescaling.
    }
\label{fig:properties}
\vspace{-1pc}  
\end{figure}

In practice, it is difficult to use (\ref{eq:NC3}) as loss function in training with highly imbalanced datasets,
because samples from minority classes may not be included in a batch so one cannot compute the class means.
In the proposed ECL framework, we extend (\ref{eq:NC3}) to use class prototypes.
We observe that this extension can accelerate training and improve the classification performance (see Section~\ref{sec:mean_vs_proto}).

Combining all these properties,
one can simplify a classifier in the sense that
a classifier converges to the class means with the nearest distance.
For balanced datasets,
supervised CL \cite{supcon} can naturally achieve the three properties above \cite{cp1}.
However, class imbalance inevitably deteriorates the aforementioned three key properties \cite{minoritycollapse,neuralcollapse_imb}:

\noindent
\textbf{Does class imbalance hinder property (a)?} 
In theory, the supervised CL method \cite{supcon} can induce property (a) regardless of whether a dataset is balanced or not \cite{cp1,BCL}.
In practice, however, majority classes often exhibit larger effective intra-class variability due to their overwhelming number of samples.

\noindent
\textbf{How does class imbalance hinder property (b)?}
The theoretical analysis in \cite{cp1, BCL} reveals that data imbalance mainly disrupts the geometric equilibrium of inter-class means because majority classes dominate repulsion terms.
Under class imbalance, each sample moves farther away from majority classes, as the overwhelming number of majority-class instances exerts a disproportionately large repulsive influence on the feature space.
Suppose that in Fig.~\ref{fig:configuration}(b), the red majority class is larger than the green majority class, while the blue minority class is far smaller than both.
In light of the above analysis, this imbalance induces angular separations such that $\theta_1 (\text{green-red}) > \theta_2 (\text{red-blue}) > \theta_3 (\text{blue-green})$, resulting in an asymmetric inter-class mean geometry.
BCL \cite{BCL} is designed to alleviate this issue, but it introduces another limitation that motivates our work (see later Section~\ref{sec:motivation:batch}).

\noindent
\textbf{How does class imbalance hinder property (c)?}
Under a balanced dataset, 
a symmetric geometry of the classifier weights emerges as a consequence of the representation geometric equilibrium satisfying properties (a) and (b),
at the minimal supervised CL loss \cite{cp1}.
However, in practice, and particularly under class imbalance,
it is challenging to satisfy properties (a) and (b) with supervised CL methods,
thereby making it inherently difficult to satisfy property (c) as well.
In other words, even when properties (a) and (b) are partially promoted in an imbalanced learning setup---such as in methods like BCL \cite{BCL}---property (c) does not necessarily emerge as a natural consequence. 
In imbalanced learning, majority classes tend to exhibit larger effective intra-class variability simply due to their overwhelming number of samples.
This broader variability induces larger angular margins for majority classes, causing their classifier weights to drift away from their corresponding class-mean directions.
Conversely, minority classes—with fewer samples and smaller variability—require only modest margins, keeping their weights closer to their means.
These asymmetric angular tendencies distort the decision boundaries, preventing the classifier weights from converging to a symmetric geometric configuration and thereby hindering property (c).
(See later Fig.~\ref{fig:bcl}(a) for empirical evidence.)
Exiting supervised CL methods underestimate property (c), and this limitation motivates our work (see later Section~\ref{sec:motivation:cc-ge}).

\subsection{Motivations: Two key insights from analyzing BCL} \label{preliminary analysis}

This section motivates the proposed method by presenting two key insights from analyzing the predominant supervised CL method for imbalanced datasets, BCL \cite{BCL}.
The main idea of BCL is to include class prototypes in every mini-batch so that all classes are consistently considered even under imbalanced datasets.
BCL encourages the learned representations to collapse toward the vertices of a regular simplex,
a highly symmetric configuration in which all class means are equally spaced.
In the following subsections, we derive two key insights from analyzing BCL:
\begin{itemize}
\item The contribution of prototypes depends on batch composition in forming the  representation geometric equilibrium, which may impede performance improvements particularly for majority classes.
\item The classifier weights and class center representations may become misaligned, leading to poor generalization performance.
\end{itemize}
These naturally extend to other supervised CL methods for imbalanced datasets.

\subsubsection{Prototype contribution depending on batch composition}
\label{sec:motivation:batch}

The BCL loss is given by the following~\cite{BCL}:
\begingroup
\setlength{\thinmuskip}{1.5mu}
\setlength{\medmuskip}{2mu plus 1mu minus 2mu}
\setlength{\thickmuskip}{2.5mu plus 2.5mu}
\ea{
\mathcal{L}_{\text{BCL}}
:= &
 -  \frac{1}{|\mathcal{B}|}  \sum_{i \in \mathcal{B}} \frac{1}{|\mathcal{P}(i)|+1}
\sum\limits_{j \in  {\mathcal{P}(i) \cup \{ p_{y_i} \}} }
\nn
\\
& 
\log
 \frac{\exp( \mathbf{z}_{i}\cdot \mathbf{z}_{j}  / \tau )}{\sum\limits_{c \in \mathcal{Y} }\frac{1}{|\mathcal{B}_c|+1}\sum\limits_{k\in \mathcal{B}_c \cup \{ p_{c} \} } \exp(\mathbf{z}_{i}\cdot \mathbf{z}_{k} / \tau)    },
\label{eq:BCL}
}
\endgroup
where $\mathcal{B}$ is the index set of all training samples in a CL mini-batch, 
$i$ denotes the index of an anchor sample,
$\mathcal{P}(i)$ represents the index set of all positive samples associated with the anchor $i$ within a CL mini-batch, excluding the anchor itself,
${p}_{c}$ denotes the index of the $c$th class prototype,
$\mathcal{B}_{c}$ represents the set of all samples of the $c$th class in a CL mini-batch, 
and $|\cdot|$ denotes the number of samples in a set.
For the $i$th anchor sample,
its feature representation and class label are $\mathbf{z}_{i}$ and $y_i$, respectively.
The temperature parameter $\tau$ adjusts the scale of similarity scores before the exponential operation,
effectively controlling the contrastive strength between an anchor and its positive or negative samples.
BCL performs class-wise averaging in the denominator of \R{eq:BCL}, where the features of samples within each class are averaged. 
In addition, class prototypes are included in every mini-batch.

The behavior of BCL \R{eq:BCL} (with $\tau = 1$) is analyzed through the following bound~\cite{BCL}:
\begingroup
\setlength{\thinmuskip}{1.5mu}
\setlength{\medmuskip}{2mu plus 1mu minus 2mu}
\setlength{\thickmuskip}{2.5mu plus 2.5mu}
\fontsize{9.5pt}{11.4pt}\selectfont
\ea{
&
\mathcal{L}_{\text{BCL}}
\ge
\sum_{i\in \mathcal{B}}
\log\Bigg(
1+(C-1)\exp\Bigg(
\nn
\\
&
\underbrace{\frac{1}{C-1}\sum_{c \in \mathcal{Y} \backslash {y_i} }
\frac{1}{|\mathcal{B}_c|+1}\sum_{k\in \mathcal{B}_c  \cup \{ p_{c} \} }\mathbf z_i \cdot \mathbf z_k}_{\text{repulsion term}}
-
\underbrace{\frac{1}{|\mathcal{P}(i)|+1}\sum_{j \in {\mathcal{P}(i) \cup \{ p_{y_i} \}} } \mathbf z_i \cdot \mathbf z_j}_{\text{attraction term}}
\Bigg)
\Bigg),
\label{eq:BCL:bound}
}
\endgroup

This shows that the class-wise averaging operator in the denominator of \R{eq:BCL} mitigates the dominance of majority classes in the repulsion term of \R{eq:BCL:bound}. 
Without it, majority classes dominate the repulsion term in \R{eq:BCL:bound}, resulting in an asymmetric geometry.
In addition, incorporating class prototypes into every mini-batch in \R{eq:BCL} ensures that inter-class relations are consistently preserved, even when minority classes are more likely to be absent from a batch.

\emph{However,} each class prototype is treated as only a single additional sample and is averaged with the corresponding instances. 
As a result, its contribution diminishes as the number of samples from that class increases within a batch, in both the attraction and repulsion terms of \R{eq:BCL:bound}.
Consequently, the loss in \R{eq:BCL} becomes sensitive to batch composition. 
As the number of instances from a class increases within a batch, the contribution of its prototype diminishes, 
causing variations in batch-wise class means and shifting class representations across batches rather than forming stable geometric relations.
This may impede the performance improvement particularly for majority classes, as the contribution of prototypes from majority classes is weak (see empirical evidence later in Section~\ref{sec:comparison_cl}).

\subsubsection{Misalignment between classifier weights and class center representations}
\label{sec:motivation:cc-ge}

\begin{figure}[t]
\centering
\small\addtolength{\tabcolsep}{-5pt}
\renewcommand{\arraystretch}{1}

\begin{tabular}{cc}
    \centering
    \includegraphics[width=0.5\columnwidth]{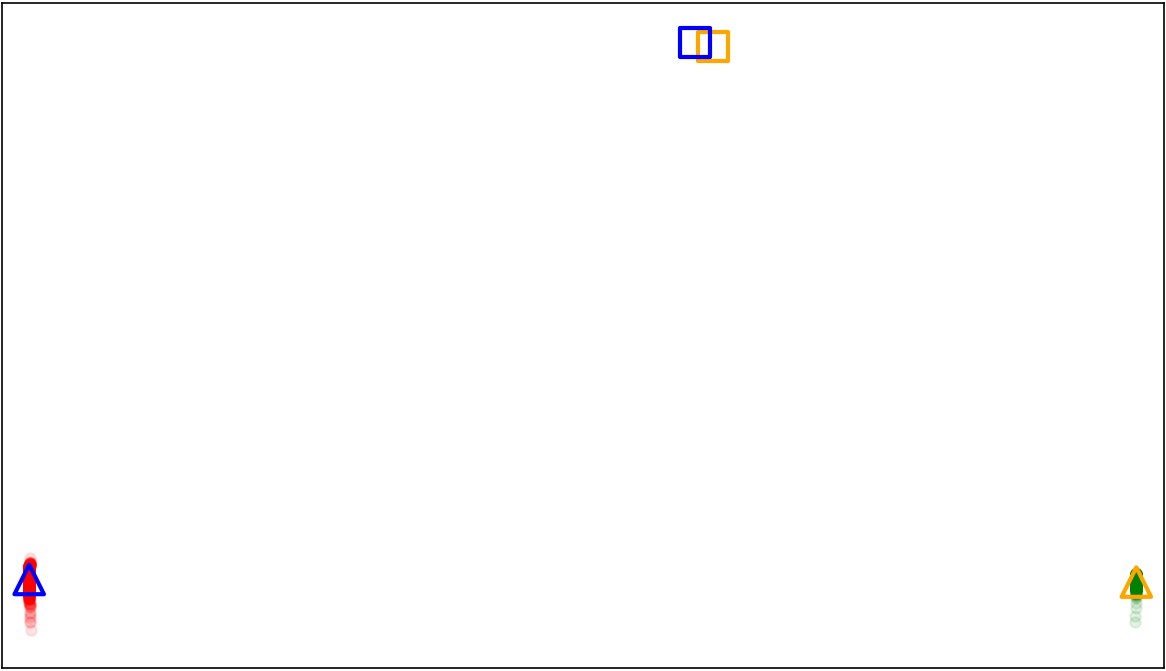} & 
    \includegraphics[width=0.5\columnwidth]{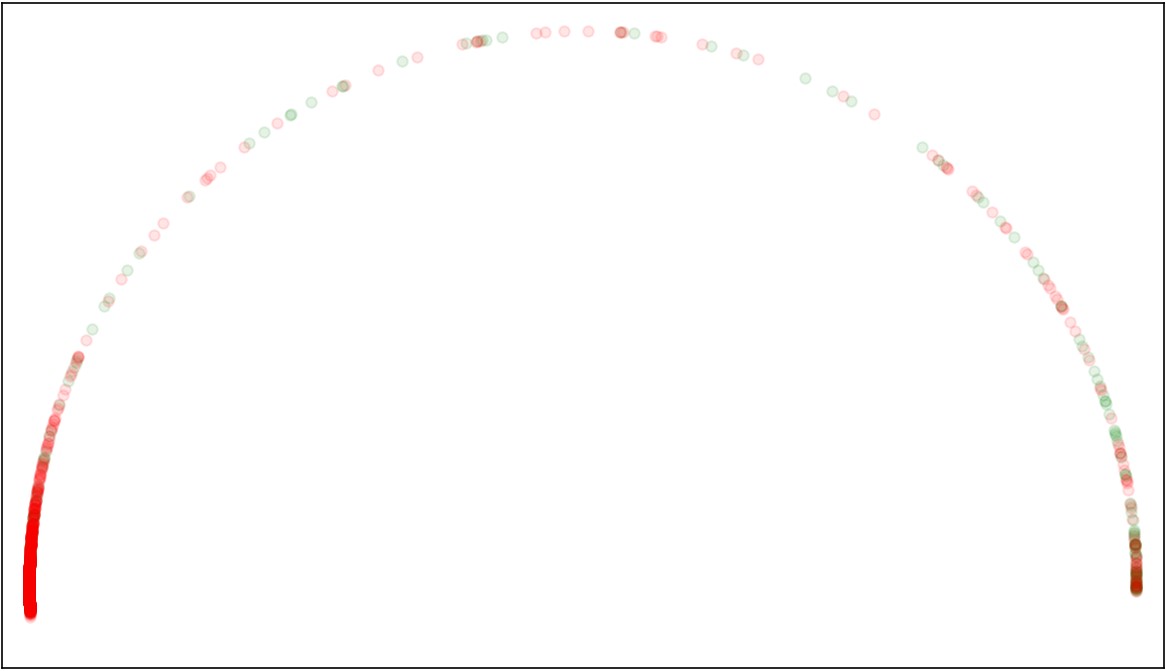} 
    \\
    (a) Training set & (b) Test set \\
\end{tabular}
    \caption{Visualization of feature distribution on the unit sphere obtained by BCL with the LCCT dataset zooming in to the upper half for better visualization.
    The $\triangle$ and $\square$ symbols indicate the class means and classifier weights, while the red and green points represent samples from the majority and minority classes, respectively.
    The results show that the CC-GE property is \emph{not} induced.
    }
\label{fig:bcl}
\vspace{-1pc} 
\end{figure}

Although BCL is designed to promote the geometric equilibrium among representations---namely, properties (a) and (b) in Section~\ref{sec:prelim}---it lacks a mechanism to establish a symmetric geometric configuration for the classifier, i.e., the CC-GE property (c) in Section~\ref{sec:prelim}.
We argue that this missing component plays a critical role in its limited generalization performance under class imbalance.

As discussed in Section~\ref{sec:prelim},
under class imbalance, it becomes difficult even for BCL to fully promote properties (a) and (b) in the representation space, 
thereby making it inherently challenging for property (c) to emerge.  
In particular, 
the imbalance in effective intra-class variability across classes  
induces asymmetric angular margins, 
causing classifier weights to drift away from their 
corresponding class-mean directions and preventing the formation of a symmetric geometric configuration.

The above rationale is supported by our observation.
BCL demonstrated poor generalization on imbalanced LCCT test sets; see Fig.~\ref{fig:bcl}(b).
As illustrated in Fig.~\ref{fig:bcl}(a), 
while the representation geometric equilibrium is \emph{partially} promoted, 
the classifier nevertheless does not form a symmetric geometric configuration—i.e., the CC-GE is \emph{not} induced---with CC-GE measures of $0.51$ and $0.73$ for the LCCT training and test sets, respectively.
In particular, the majority class exhibits larger effective intra-class variability than the minority class, which pushes the decision boundary toward the minority side and results in the formation of asymmetric classifier weights.
Our rationale and observation correspond well with the theoretical claim in \cite{minoritycollapse,neuralcollapse_imb} that imbalanced learning deteriorates inducing the key properties in Section~\ref{sec:prelim} and generalization capability.

If the class representations, class means, and classifier weights are jointly in geometric equilibrium---i.e., if all three properties in Section~\ref{sec:prelim} are simultaneously satisfied---the classification problem reduces to a nearest-centroid–like decision rule that improves generalization. 
In particular, when the last-layer classifier weights are well aligned with their corresponding class means (i.e., when the CC-GE property is induced), classification simply amounts to selecting the class whose weight vector is most closely aligned with the input representation, analogous to nearest-centroid classification.

In the next section, we propose the ECL framework that can address the two aforementioned limitations.
Specifically, the proposed BC-ECL formulation in Section~\ref{sec:BC-ECL} mitigates the batch-dependent prototype contribution in forming the representation geometric equilibrium,
while the proposed CC-GE promoting loss in Section~\ref{sec:ccge} resolves the classifier–class center misalignment.

\subsection{ECL: Equilibrium Contrastive Learning} 
\label{ECL}

\begin{figure*}[ht!]
    \centering
    \begin{tikzpicture}
        \node[anchor=south west, inner sep=0] (image) at (0,0) {\includegraphics[width=0.95\textwidth ]{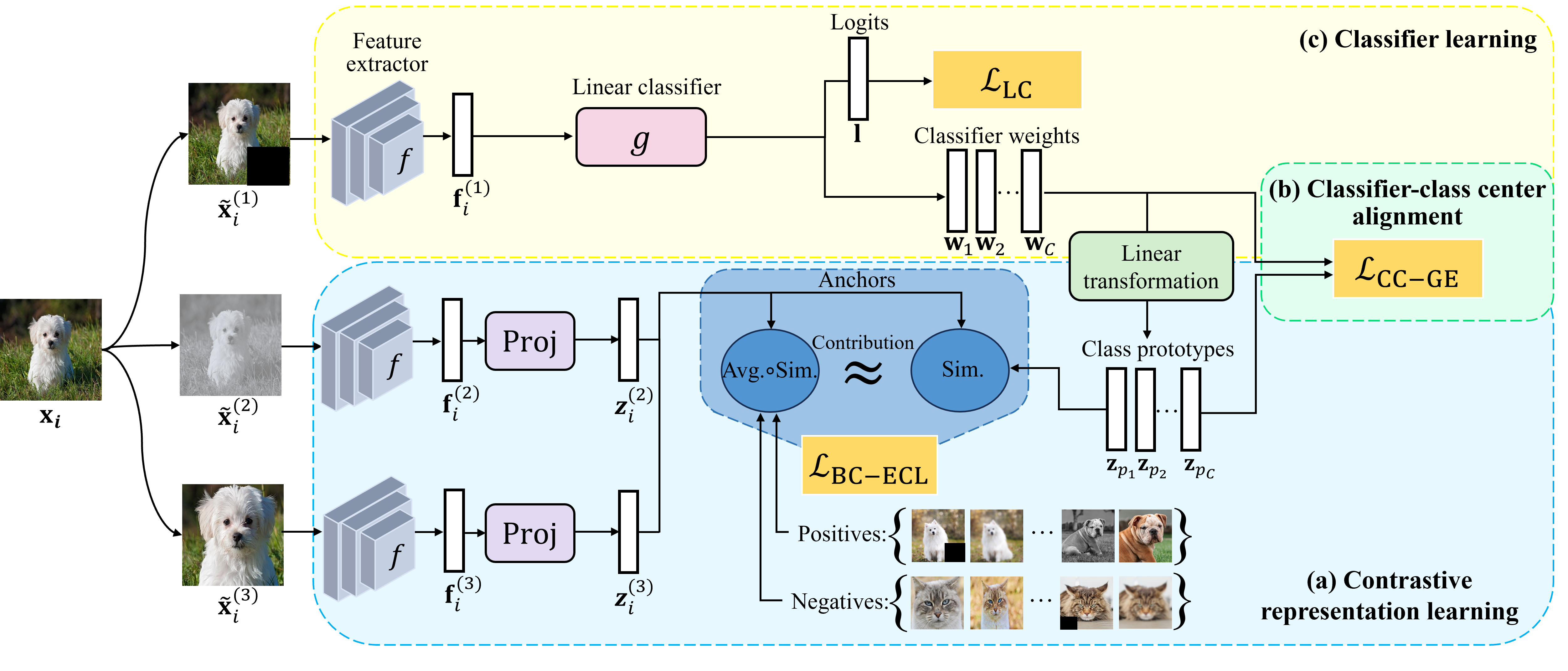}};

    \end{tikzpicture}
    \vspace{-0.5pc} 
     \caption{
    The overall ECL architecture. The proposed ECL consists of three core modules, 
    (a) contrastive representation learning with proposed $\mathcal{L}_{\text{BC-ECL}}$ in Section~\ref{sec:BC-ECL}, 
    (b) classifier-class center alignment with proposed $\mathcal{L}_{\text{CC-GE}}$ in Section~\ref{sec:ccge}, and
    (c) classifier learning with $\mathcal{L}_{\text{LC}}$ introduced in Section~\ref{sec:lc}.
    \textbf{(a)} The proposed $\mathcal{L}_{\text{BC-ECL}}$ formulation equalizes the contribution between class-average features and prototypes in promoting the intra-class feature collapse and geometric equilibrium among class means.
    We utilize the linearly transformed classifier weights as class prototypes, and $\ell _{2}$-normalize the representations $\{\mathbf{z}_{i}^{(2)}, \mathbf{z}_{i}^{(3)} : \forall i \}$.
    The symbols $\text{Sim.}$ and $\text{Avg.}$ denote a similarity computation between two vectors and an averaging operator, respectively.
    \textbf{(b)} The proposed $\mathcal{L}_{\text{CC-GE}}$ loss aligns the classifier weights and class prototypes to promote geometric equilibrium between classifier weights and class center representations.
    \textbf{(c)} For classifier learning, we use a logit-compensated variant of the cross-entropy loss, $\mathcal{L}_{\text{LC}}$.}
    \label{fig:arch}
\vspace{-1pc} 
\end{figure*}

Fig.~\ref{fig:arch} illustrates the overall E2E framework of ECL.
ECL consists of three core modules: 
contrastive representation learning in Fig.~\ref{fig:arch}(a), 
classifier-class center alignment in Fig.~\ref{fig:arch}(b), and 
classifier learning in Fig.~\ref{fig:arch}(c).
The ECL framework uses the following three training losses in each core module: 
\textit{1)} BC-ECL loss that promotes intra-class feature collapse and inter-class mean geometric equilibrium with batch-invariant prototype contributions (Section~\ref{sec:BC-ECL}),
\textit{2)} the CC-GE promoting loss that aligns classifier weights and class center representations (Section~\ref{sec:ccge}), and
\textit{3)} classifier learning loss that can address the bias issue caused by data imbalance (Section~\ref{sec:lc}).

ECL uses the following three major components:
\begin{itemize}
    \item Data augmentation $\widetilde{\mathbf{x}} = \mathrm{Aug}(\mathbf{x})$: Given a batch of samples $\mathbf{x}$, we generate three random augmentation samples $\widetilde{\mathbf{x}}^{(1)}$, $\widetilde{\mathbf{x}}^{(2)}$, and $\widetilde{\mathbf{x}}^{(3)}$, similar to \cite{BCL}.
    In the CL branch (Fig.~\ref{fig:arch}(a)), we use $\widetilde{\mathbf{x}}^{(2)}$ and $\widetilde{\mathbf{x}}^{(3)}$, so the size of each CL mini-batch is $|\mathcal{B}| = 2B$, where $B$ is the number of samples in a mini-batch.
    In the classifier learning branch (Fig.~\ref{fig:arch}(c)), we use $\widetilde{\mathbf{x}}^{(1)}$.

    \item Feature extractor $\mathbf{f}=f(\widetilde{\mathbf{x}})$: It extracts features $\mathbf{f}\in \mathbb{R}^{F}$ from $\widetilde{\mathbf{x}}$, where $F$ is the dimension of features. 
    In the classifier learning branch (Fig.~\ref{fig:arch}(c)), we use $\mathbf{f}$.
    In the CL branch (Fig.~\ref{fig:arch}(a)), following the prevailing convention in CL, we introduce an additional projector $\mathrm{Proj}$ to produce features $\mathbf{z}$ for CL: $\mathbf{z} = \mathrm{Proj}(\mathbf{f})$, where $\mathbf{z} \in \mathbb{R}^{F}$.

    \item Classifier $\mathbf{l} = g(\mathbf{f})$: The classifier maps features $\mathbf{f}$ to the logits $\mathbf{l} \in \mathbb{R}^{C}$, where $C$ is the number of classes.
    We use a linear classifier, i.e.,  $\mathbf{l} = \mathbf{W} \mathbf{f}+\mathbf{b}$, where $\mathbf{W} \in \mathbb{R}^{C \times F}$ is defined as in (\ref{eq:NC3})  and $\mathbf{b} \in \mathbb{R}^C$ is the bias.
\end{itemize}

\subsubsection{BC-ECL:
Balanced Class-wise ECL} 
\label{sec:BC-ECL}

To overcome this limitation arising from batch-dependent prototype contributions, 
we propose a BC-ECL formulation as follows:
\begingroup
\allowdisplaybreaks
\setlength{\thinmuskip}{1.5mu}
\setlength{\medmuskip}{2mu plus 1mu minus 2mu}
\setlength{\thickmuskip}{2.5mu plus 2.5mu}
\ea{
\mathcal{L}_{\text{BC-ECL}} 
:= 
& 
- \frac{1}{|\mathcal{B}|} \sum_{i \in \mathcal{B}} \frac{1}{|\mathcal{P}(i)|} 
\sum\limits_{j\in \mathcal{P}(i)}
\nn
\\
&
\log
 \frac{\exp((\mathbf{z}_{i}\cdot \mathbf{z}_{j} +  \mathbf{z}_{i}\cdot \mathbf{z}_{p_{y_i}} )/ (2\tau) )}{\sum\limits_{c \in \mathcal{Y}} \bigg( \frac{1}{|\mathcal{B}_c|}\sum\limits_{k\in \mathcal{B}_c} \exp(\mathbf{z}_{i}\cdot \mathbf{z}_{k} / \tau ) \bigg) + \exp(\mathbf{z}_{i}\cdot \mathbf{z}_{p_c} / \tau)}.
\label{eq:ECL}
}
\endgroup
Different from BCL \R{eq:BCL}, 
the denominator of the proposed loss \R{eq:ECL} first averages the instances within each class and then takes the sum of the class-average and the corresponding prototype.
In its numerator, we take the average of each positive sample and its corresponding prototype.

The proposed BC-ECL loss \R{eq:ECL} simultaneously balances \textit{1)} the contributions of class-average features and prototypes within each class, and \textit{2)} the overall influence across majority and minority classes, thereby yielding stable and balanced representations across batches.
In addition, \R{eq:ECL} promotes that all the within-class representations collapse to their class means.
We support this claim by deriving its bound whose key idea is to decouple the attraction and repulsion effects~\cite{cp1}:

\begin{theorem}\label{thm:ecl:bound}
The loss in \R{eq:ECL} with $\tau = 1$ is bounded by
\begingroup
\allowdisplaybreaks
\setlength{\thinmuskip}{1.5mu}
\setlength{\medmuskip}{2mu plus 1mu minus 2mu}
\setlength{\thickmuskip}{2.5mu plus 2.5mu}
\ea{
\mathcal{L}_{\text{BC-ECL}}
\ge & 
\frac{1}{|\mathcal{B}|}
\sum_{i\in \mathcal{B}}
\log \Bigg(
1+(C-1)
\exp \Bigg(
\nn
\\
&
\underbrace{\frac{1}{2(C-1)}\sum_{c \in \mathcal{Y} \setminus \{y_i\}}
\bigg( \frac{1}{|\mathcal{B}_c|}\sum_{k\in \mathcal{B}_c}\mathbf{z}_i \cdot \mathbf{z}_k \bigg) 
+ 
\mathbf{z}_i \cdot \mathbf{z}_{p_c}}_{\text{repulsion term}}
\nn
\\
&
-
\underbrace{
\frac{1}{2}
\Bigg(
\bigg(
\frac{1}{|\mathcal{P}(i)|}
\sum_{j\in \mathcal{P}(i)}\mathbf{z}_i \cdot \mathbf{z}_j 
\bigg)
+
\mathbf{z}_i \cdot \mathbf{z}_{p_{y_i}}
\Bigg)
}_{\text{attraction term}}
\Bigg)
\Bigg).
\label{eq:ECL:bound}
}
\endgroup
\end{theorem}
\begin{proof}
See Section~S1 in the Supplementary Material.
\end{proof}

Theorem~\ref{thm:ecl:bound} leads to the following geometric interpretations, revealing how BC-ECL \R{eq:ECL} shapes the feature representation space:

\begin{itemize}
\item \textbf{Balanced repulsion and attraction.} 
    BC-ECL equally weights the class-mean features and class prototypes in both the repulsion and attraction terms, preventing prototype dilution and ensuring balanced class-wise contributions.
    
\item \textbf{Intra-class feature collapse.}
    Each anchor representation is simultaneously pulled toward its positive samples and its class prototype, encouraging intra-class feature collapse.
    
\item \textbf{Inter-class mean geometric equilibrium.}
    An anchor is pushed away from both the mean features and prototypes of non-target classes, maintaining an inter-class geometric equilibrium with well-separated class means.
    
\item \textbf{Batch-invariant geometry.}
    By equalizing contributions between class means and prototypes, BC-ECL reduces sensitivity to batch composition and maintains consistent geometric relations across batches.

\end{itemize}

\noindent
Given an imbalanced dataset, minimizing (\ref{eq:ECL}) encourages all within-class representations to collapse toward their corresponding class means and ultimately to the vertices of a regular simplex, while preserving a stable inter-class geometric equilibrium across batches.
(Note that given a balanced dataset, 
when one obtains the minimum from supervised CL loss \cite{supcon}, 
representations from each class collapse to the vertices of a regular simplex \cite{cp1,neuralcollapse}.)

Compared to the bound in \R{eq:BCL:bound} for BCL, the bound in \R{eq:ECL:bound} for BC-ECL exhibits a weaker contrastive effect due to the additional $1/2$ factor applied to both attraction and repulsion terms.
This attenuation can be offset by reducing the temperature $\tau$ in \R{eq:ECL}, which strengthens the contrastive sharpness despite the $1/2$ factor.
We compare BCL (\ref{eq:BCL}) and BC-ECL (\ref{eq:ECL}) in terms of their performance on classes grouped into many-, medium-, and few-shot regimes (see later Section~\ref{sec:comparison_cl}).

Given that classifier weights are co-linear with the vertices of the simplex \cite{ cp1,neuralcollapse}, we obtain the class prototypes by passing the classifier weights through a linear transformation.
This differs from \cite{BCL} that calculates class prototypes by applying a non-linear multi-layer perceptron (MLP) model to the classifier weights.
We compare these two forms of prototypes in terms of training time and classification performance (see later Section~\ref{sec:mean_vs_proto}).

The computational complexity of the proposed BC-ECL formulation in \R{eq:ECL} is $\cO(|\mathcal{B}| (|\mathcal{B}| + C) F)$,
which is comparable to that of BCL in \R{eq:BCL}.
Computing and incorporating class prototypes only marginally increases the overall cost over the supervised CL loss \cite{supcon} whose complexity is $\cO(|\mathcal{B}|^2 F)$.
When $B \gg C$, the computational complexity of BC-ECL becomes $\cO(|\mathcal{B}|^2 F)$ that is the same order as the supervised CL loss \cite{supcon}.

\subsubsection{Promoting CC-GE: Classifier-class Center Geometric Equilibrium}
\label{sec:ccge}

The proposed BC-ECL formulation in \R{eq:ECL} can promote the representation geometric equilibrium (see the first two properties in Section~\ref{sec:prelim}) in a batch-invariant manner.
However, in imbalanced learning, BC-ECL in \R{eq:ECL} may \textit{not} promote the CC-GE property (see the third property in Section~\ref{sec:prelim}), which, in turn, deteriorates generalization.
See our corresponding observations in Section~\ref{sec:motivation:cc-ge}.

Our aim is to establish a symmetric geometric configuration in both representation and classifier to improve the generalization performance in imbalanced classification.
We propose a new geometric loss that explicitly promotes the CC-GE property, 
by aligning the class center representations, i.e., class prototypes, and classifier weights:
\begin{equation}
\mathcal{L}_{\text{CC-GE}} :=     
    \left 
    \| 
    \frac{\mathbf{W}^\top}{\left \| \mathbf{W} \right \|_\mathrm{F}}- \frac{\mathbf{P}}{\| \mathbf{P} \|_\mathrm{F}} 
    \right \|_\mathrm{F}^2, 
    \label{eq:CC-GE}
\end{equation}
where the the classifier weights matrix $\mathbf{W}$ is defined as in (\ref{eq:NC3}),
and 
$\mathbf{P} \in
\mathbb{R}^{F \times C}$
is the class prototypes matrix defined by stacking the class prototypes, specifically, $\mathbf{P} := \big[ \mathbf{z}_{{p}_{1}}, \ldots, \mathbf{z}_{{p}_{C}} \big]$.

Note in (\ref{eq:CC-GE}) that
as prototypes,
we apply a linear transformation, i.e., single-layer perceptron model,
to class-specific weights of a linear classifier $\{ \mathbf{w}_1, \ldots, \mathbf{w}_C \}$ to project them into the representation space,
rather than the mean representations.
If the classifier weights are co-linear with the vertices of a simplex to which the classes collapse as achieved in balanced supervised CL \cite{cp1,neuralcollapse}, 
the prototypes obtained through the linearly transformed classifier weights are also co-linear with the vertices of a regular simplex.

Under class imbalance, without minimizing \R{eq:CC-GE}, classifier weights do not maintain symmetric geometric relations; see Section~\ref{sec:motivation:cc-ge}.
By minimizing  (\ref{eq:CC-GE}), one can simplify the geometric configuration of both the representation space and the classifier, effectively making the classifier behave like a nearest-centroid classifier.
In other words, the classifier no longer needs to learn complex decision boundaries; it simply determines which classifier weight is most closely aligned with each instance feature.
This simpler decision rule can improve generalization, particularly under class imbalance.

The computational complexity of the proposed CC-GE promoting loss in \R{eq:CC-GE} is $\cO(FC)$.
When combined with the complexity of BC-ECL in \R{eq:ECL}, the overall cost becomes $\cO(|\mathcal{B}| (|\mathcal{B}| + C) F)$, since in general $B \gg 1$ and $B \geq C$.
When $B \gg C$, this becomes $\cO(|\mathcal{B}|^2 F)$ that is the same order as the supervised CL loss \cite{supcon}.

\subsubsection{Logit compensation loss} \label{sec:lc}

In imbalanced learning, bias typically exists in the logits.
Following \cite{BCL},
we use the logit compensation (LC) approach \cite{BCL,lc,lc2,lc3} to address the bias issue caused by data imbalance:
\begingroup
\allowdisplaybreaks
\setlength{\thinmuskip}{1.5mu}
\setlength{\medmuskip}{2mu plus 1mu minus 2mu}
\setlength{\thickmuskip}{2.5mu plus 2.5mu}
\be{
\mathcal{L}_{\text{LC}}
:=
-\frac{1}{|\mathcal{B}|}\sum\limits_{i \in \mathcal{B}} \sum\limits_{c \in \mathcal{Y}} \mathbb{1}_{c = y_i} \cdot \log \! \bigg( \frac{\exp(g(\mathbf{f}_i)_{c}+\log{{h}_{c}})} {\sum_{c' \in \mathcal{Y}} \exp(g(\mathbf{f}_i)_{c'}+\log{{h}_{c'}}) } \bigg),
\label{eq:lc}
}
\endgroup
where ${h}_{c}=|\mathcal{B}_{c}| / |\mathcal{B}| \in [0,1]$ denotes the frequency of the $c$th class, and
where $\mathbb{1}_{c = y_i}$ denotes the indicator function that equals $1$ if $c = y_i$ and $0$ otherwise.
This can address the bias problem by adding $\log{{h}_{c}}$ to the logits.

\begin{algorithm}[t]
\caption{Proposed ECL}
\label{alg:ECL}
\SetKwInOut{Require}{Require}
\SetKwInOut{Initialize}{Initialize}

\Require{
Training set $\mathcal{D}$, number of epochs $\textit{Epoch}$, batch size $B$, loss balancing parameters $\{ \lambda_{\text{BC-ECL}}, \lambda_{\text{CC-GE}}, \lambda_{\text{LC}} \}$
}

\Initialize{
Parameters $\boldsymbol{\uptheta}$ of feature extractor $f$, parameters $\boldsymbol{\upphi}$ of projector $\text{Proj}$, parameters $( \mathbf{W}, \mathbf{b} )$ of linear classifier $g$, transformation matrix $\mathbf{T}$ mapping $\mathbf{W}$ to class prototypes $\mathbf{P}$
}
\For{$\textit{epoch} = 1$ \KwTo $\textit{Epoch}$}{
    
    Shuffle the dataset indices $\mathcal{I}$

    Partition the shuffled indices into mini-batches of size $B$: 
    $\mathcal{I}_B^{(1)}$, $\mathcal{I}_B^{(2)}$, \ldots, $\mathcal{I}_B^{(M)}$, where $M \approx \lceil N /B \rceil$
    
  \For{$\textit{iter} = 1$ \KwTo $M$}{

    $\mathcal{I}_B = \mathcal{I}_B^{(\textit{iter})}$
    \Comment{Mini-batch selection}

    $\mathcal{Z} := \{ \mathbf{z}_k^{(j)} = \mathrm{Proj}_{\boldsymbol{\upphi}}( f_{\boldsymbol{\uptheta}} ( \text{Aug}_j (\mathbf{x}_{k}) ) :k \in \mathcal{I}_B, j =2,3 \}$
    \Comment{Representations for BC-ECL \R{eq:ECL}}

    $\mathbf{P} = \mathbf{T} \mathbf{W}^\top$
    \Comment{Prototypes for BC-ECL \R{eq:ECL} and CC-GE promotion \R{eq:CC-GE}}
        
    $\mathcal{F} := \{ \mathbf{l}_k = g_{\mathbf{W},\mathbf{b}}( f_{\boldsymbol{\uptheta}} ( \text{Aug}_j (\mathbf{x}_{k})) : k \in \mathcal{I}_B, j = 1 \}$
    \Comment{Logits for classifier learning \R{eq:lc}}

    Compute $\nabla_{\boldsymbol{\uptheta}, \boldsymbol{\upphi}, \mathbf{W}, \mathbf{T}} \mathcal{L}_{\text{BC-ECL}}$ with $\{ \mathbf{z}_i(\boldsymbol{\uptheta}, \boldsymbol{\upphi}) \in \mathcal{Z} \}$ and $\mathbf{P}( \mathbf{W}, \mathbf{T} )$
    \Comment{Gradients for \R{eq:ECL}}

    Compute ${\nabla_{\mathbf{W}, \mathbf{T}}\mathcal{L}}_{\text{CC-GE}}$ with $\mathbf{P}( \mathbf{W}, \mathbf{T} )$ and $ \mathbf{W}$
    \Comment{Gradients for \R{eq:CC-GE}}

    Compute $\nabla_{\boldsymbol{\uptheta}, \mathbf{W}, \mathbf{b}} \mathcal{L}_{\text{LC}}$ with $\{ \mathbf{l}_i (\boldsymbol{\uptheta}, \mathbf{W}, \mathbf{b}) \in \mathcal{F} \}$ and $\{ y_i \in \mathcal{I}_B \}$
    \Comment{Gradients for \R{eq:lc}}

    Update $\{ \boldsymbol{\uptheta}, \boldsymbol{\upphi}, \mathbf{W}, \mathbf{b}, \mathbf{T} \}$ with 
    $\lambda_{\text{BC-ECL}} \cdot \nabla_{\boldsymbol{\uptheta}, \boldsymbol{\upphi}, \mathbf{W}, \mathbf{T}} \mathcal{L}_{\text{BC-ECL}} 
    +
    \lambda_{\text{CC-GE}} \cdot {\nabla_{\mathbf{W}, \mathbf{T}}\mathcal{L}}_{\text{CC-GE}}
    + 
    \lambda_{\text{LC}} \cdot \nabla_{\boldsymbol{\uptheta}, \mathbf{W}, \mathbf{b}} \mathcal{L}_{\text{LC}}$
  }
}
\KwOut{$ \{ \boldsymbol{\uptheta}, \mathbf{W}, \mathbf{b}\}$}
\end{algorithm}

\subsubsection{The overall ECL loss}

The overall ECL loss is given by 
\begin{equation}
    \mathcal{L} =
    \lambda_{\text{BC-ECL}} \cdot 
    \mathcal{L}_{\text{BC-ECL}}
    +
    \lambda_{\text{CC-GE}} \cdot 
    \mathcal{L}_{\text{CC-GE}}
    +
    \lambda_{\text{LC}} \cdot
    \mathcal{L}_{\text{LC}},
\label{eq:total}
\end{equation}
where $\lambda_{\text{BC-ECL}}$, $\lambda_{\text{CC-GE}}$, and $\lambda_{\text{LC}}$ are hyper-parameters that control the strength of $\mathcal{L}_{\text{BC-ECL}}$ in (\ref{eq:ECL}),
$\mathcal{L}_{\text{CC-GE}}$ in (\ref{eq:CC-GE}), and $\mathcal{L}_{\text{LC}}$ in (\ref{eq:lc}), respectively.

Under class imbalance, minimizing the overall loss can promote the three key properties in Section~\ref{sec:prelim}, simplifying the geometry of the representation and classifier.
(Note again that $\mathcal{L}_{\text{BC-ECL}}$ promotes intra-class feature collapse and and inter-class mean geometric equilibrium, and $\mathcal{L}_{\text{CC-GE}}$ promotes the CC-GE.)
We verify this with some empirical metrics (see later Section~\ref{sec:metric}).
All the three losses in \R{eq:total} guide a model toward a common geometric objective under class imbalance---the geometric equilibrium in Fig.~\ref{fig:configuration}. 
Consequently, their gradient directions are naturally aligned, leading to cooperative rather than conflicting updates.

We include the pseudo-code of the proposed ECL framework in Algorithm~\ref{alg:ECL}.
In inference, we only use a feature extractor $f_{\boldsymbol{\uptheta}}$ and a linear classifier $g_{\mathbf{W}, \mathbf{b}}$, removing a projector $\mathrm{Proj}_{\boldsymbol{\upphi}}$ and a linear transformation $\mathbf{T}$ producing prototypes used in BC-ECL (\ref{eq:ECL}) and CC-GE promoting loss (\ref{eq:CC-GE}).

\section{Experimental setups}
\label{sec:exp}

This section describes experimental setups.
We compared the Top-1 accuracy of ECL with six SOTA supervised CL methods proposed to handle imbalanced datasets: 
TSC \cite{TSC}, 
BCL \cite{BCL}, 
GLMC \cite{GLMC}, 
PaCo \cite{PaCo}, 
GPaCo \cite{GPaco}, and 
ProCo \cite{Probco}.
We additionally compared the proposed method with the recent supervised CL method, CoNe \cite{cone}, that outperformed the representative supervised CL method \cite{PaCo}.
Furthermore, we compared proposed ECL with several supervised learning baselines with the natural image benchmark datasets.

We performed several ablation studies to investigate the primary components of the proposed method.
(ECL is an E2E learning approach so using only $\mathcal{L}_{\text{BC-ECL}}$ (\ref{eq:ECL}) that needs two-stage learning is omitted in its ablation studies.)

Section~\ref{sec:data} introduces the five imbalanced datasets, CIFAR-10-LT, CIFAR-100-LT, ImageNet-LT, ISIC 2019, and LCCT. 
Section~\ref{sec:impl} includes implementation details.
Section~\ref{sec:metric} introduces the metrics and a visualization tool for evaluating how well the key properties in Section~\ref{sec:prelim} are satisfied.

\subsection{Imbalanced datasets} \label{sec:data}

\noindent\textbf{CIFAR-10-LT and CIFAR-100-LT datasets.}
The datasets are long-tailed versions of CIFAR-10 and CIFAR-100 \cite{cifar}. 
They consist of $10$ and $100$ classes, respectively.
We used their standard dataset split that divides each dataset into training, validation, and test sets, with $50$K, $10$K, and $10$K images, respectively \cite{BCL,GLMC}.
Following \cite{BCL,TSC}, we ran experiments with the imbalance factors ($\rho) = \{ 100, 50, 10 \}$.
The $\rho = N_{\text{max}} / N_{\text{min}}$ reflects the imbalance degree in the data, where $N_{\text{max}}$ and $N_{\text{min}}$ is the number of samples in the majority and minority classes, respectively.

\noindent\textbf{ImageNet-LT dataset.} 
The dataset is a long-tailed version of vanila ImageNet \cite{ImageNet-LT, imagenet} consisting of $115.8$K images. The ImageNet-LT dataset has $1,\!000$ classes in total, with five to $1,\!280$ images per class (i.e., $\rho \!=\! 256$). 
We used its standard dataset split that divides ImageNet-LT into a training, validation, and test set, approximately with the $11\!\!:\!\!2\!\!:\!\!5$ ratio \cite{BCL,GLMC}.
We report the classification accuracy for Many-shot classes (specifically, the number of training samples $> 100$), Medium-shot classes (specifically, the number of training samples is between $[20,100]$), Few-shot classes (specifically, the number of training samples $\le 20$), and overall classes.

\noindent\textbf{ISIC 2019 dataset.}
The ISIC 2019 dataset consists of 25,331 dermatoscopic images for skin lesion classification with eight classes \cite{ISIC}.
Following \cite{ECL}, we randomly split the dataset into training, validation, and test sets with a ratio of $3\!:\!1\!:\!1$. 
The imbalance factor $\rho$ is $54.02$.

\noindent\textbf{LCCT dataset.}
We collected chest CT scans of lung adenocarcinoma images from $55$ patients, in Samsung Medical Center.
This dataset was registered as a prospective study \cite{lcct} and received IRB approval (IRB No.~2018-01-099), where patient dataset were encrypted to prevent their direct identification.
Lung adenocarcinoma is a subtype of cancer that consists of heterogeneous histology. 
Micropapillary and solid pattern (MPsol) is a well-known risk factor for recurrence \cite{mpsol,mpsol2,recurrence}. 
We labeled CT images as high-risk when the presence of MPsol is pathologically proven \cite{grading1,grading2,grading3}.
We used a semi-automatic approach in the AVIEW software (developed by Coreline Soft Co., Ltd., Seoul, South Korea) in drawing the tumor region of interests.
The constructed dataset consists of 687 images, with high-risk and low-risk classes, approximately of the $3 \!:\! 1$ ratio, i.e., $\rho \approx 3$.
We set the ratio between the training and test sets as approximately $5\!:\!1$.

\subsection{Implementation details}
\label{sec:impl}

This section lists hyperparameters of all methods evaluated in this paper.
We trained and tested them with the PyTorch ver.~2.0.0 \cite{pytorch} and NVIDIA GeForce RTX 4090 and A100 GPUs.
For the CIFAR-10-LT, CIFAR-100-LT, and ImageNet-LT benchmark datasets, we did not re-train competing methods unless stated otherwise, and instead relied on the officially reported results from their original publications for fair comparisons. We omit the hyperparameters of competing methods for these three datasets, assuming that the hyperparameters of each existing method had already been finely tuned for these benchmarks in their original studies.

\subsubsection{Common hyperparameters}

We first list hyperparameters commonly used throughout the methods.

For the CIFAR-10-LT and CIFAR-100-LT datasets, 
we used ResNet-$32$ \cite{resnet} for the feature extraction backbone $f$, and set feature dimension $F$ and number of epochs $N_{\text{epoch}}$ as $64$, $200$, respectively.
For the ImageNet-LT dataset, we used ResNeXt-50-32x4d \cite{resnext} for $f$, 
and set $F$, batch size $B$, and $N_{\text{epoch}}$ as $2048$, $256$, and $180$, respectively.
For the ISIC 2019 and LCCT datasets, 
we used ResNet-$50$ \cite{resnet} for $f$ and set $F$ as $2048$, respectively.
For the ISIC 2019 dataset, we set $B$ and $N_\text{epoch}$ as $128$ and $600$, respectively, and used the same augmentation strategy with \cite{ECL}.
For the LCCT dataset, we set $B$ and $N_\text{epoch}$ as $256$ and $2500$, respectively, and used \textsf{RandomCrop}, \textsf{RandomHorizontalFlip}, \textsf{GaussianBlur}, and \textsf{RandomErasing} modules for data augmentation.

\subsubsection{Existing methods}

We ran two existing methods, PaCo and GPaCo, on the CIFAR-10-LT dataset using their default setups for CIFAR-100-LT, as their original papers report results only on CIFAR-100-LT.
We ran five existing methods, TSC, CoNe, GLMC, BCL, and GPaCo, with their default setups on the ISIC2019 and LCCT datasets unless stated otherwise.
To achieve the best performance for the ISIC 2019 benchmark and our LCCT datasets, 
we tuned a learning rate of each existing method.
For the ISIC 2019 dataset, 
we set the learning rate as $0.02$ for TSC, CoNe, GLMC, and BCL, and $0.04$ for GPaCo.
For the LCCT dataset, 
we set the learning rate as $0.01$ for GLMC and CoNe, $0.03$ for BCL and TSC, and $0.04$ for GPaCo.

\subsubsection{Proposed ECL}

Following \cite{BCL}, we used the stochastic gradient descent method with a momentum of $0.9$ and decayed the learning rate at the $160$th and $180$th epochs with a step size of $0.1$. 
We used a fixed set of ECL hyperparameters in (\ref{eq:total}) for all datasets.
We set them as $\{ \lambda_\text{BC-ECL} \!=\! 0.5, \lambda_\text{CC-GE} \!=\! 3, \lambda_\text{LC} \!=\! 0.5 \}$, 
by adjusting them so that the initial magnitudes of the three weighted loss terms lie on a similar scale on CIFAR-10-LT ($\rho=100$).
See further details of the fine tuning process in Section~S2 of the Supplementary Material.
Following \cite{BCL}, we implemented the additional projector $\mathrm{Proj}$ in CL as a two-layer MLP composed with a Rectified Linear Unit (ReLU) activation between two linear layers (without bias), and we used the same hidden dimension $H$ as in \cite{BCL}, unless otherwise specified.

For the CIFAR-10-LT and CIFAR-100-LT datasets, we set temperature $\tau$, $B$, learning rate, and weight decay as $0.05$, $256$, $0.3$, and $5 \!\times\! 10^{-4}$ respectively, and used the same augmentation schemes as in \cite{BCL}. 
For the ImageNet-LT dataset, we set $\tau$, $B$, learning rate, and  weight decay as $0.1$, $256$, $0.3$,  and $3 \!\times\! 10^{-4}$, respectively, and used the same augmentation schemes as in \cite{Probco}.
For the ISIC 2019 dataset, we set $\tau$, $H$, learning rate, and weight decay as $0.1$, $1024$, $0.15$,  and $2 \!\times\! 10^{-4}$ respectively.
For the LCCT dataset, we set $\tau$, $H$, learning rate,  and weight decay as $0.1$, $1024$, $0.15$,  and $1 \!\times\! 10^{-4}$ respectively.

Note that throughout all datasets, 
we used half of the temperature $\tau$ employed in BCL to compensate for the reduced contrastive effect analyzed in Theorem~\ref{thm:ecl:bound}. 
See our remarks following Theorem~\ref{thm:ecl:bound}.

\subsection{Evaluation metrics and a visualization tool for representation learning}\label{sec:metric}

In analyzing representation learning with the perspectives of the three key properties in Section~\ref{sec:prelim},
we use the two widely used key metrics \cite{simclr,momentum,TSC,au}, 
and the alignment between mean representations and classifier weights in (\ref{eq:NC3}).

\subsubsection{Intra-class Feature Collapse (FC)} 

The FC metric averages distances between features from the same class \cite{au,TSC}:
\begin{equation}
    \text{FC} := \frac{1}{C}\sum_{c=1}^{C}\frac{1}{|\mathcal{F}_{c}|^{2}}\sum_{\mathbf{z}_{i}, \mathbf{z}_{j}\in \mathcal{F}_{c}}\left\| \mathbf{z}_{i}-\mathbf{z}_{j}\right\|_{2},
\end{equation}
where $\mathcal{F}_{c}$ denotes the set of all features belonging to $c$th class. 
This metric quantifies the extent to which features from the same class collapse toward a point, i.e., the degree of intra-class feature collapse---property (a) in Section~\ref{sec:prelim}.
A smaller FC value indicates greater compactness among samples within the same class, i.e., better feature collapse.

\subsubsection{Inter-class Mean Spacing (MS)}

The MS metric averages distances between different class means in the representation space \cite{au,TSC}:
\begin{equation}
    \text{MS} := \frac{1}{C(C-1)}\sum_{c=1}^{C}\sum_{c'=1,c' \neq c}^{C}\left\| \mathbf{m}_{c}-\mathbf{m}_{c'}\right\|_{2},
\end{equation}
where $\mathbf{m}_{c}$ denotes the $c$th class mean.
This metric quantifies how well the normalized class means are uniformly distributed on a unit hypersphere, reflecting the degree of inter-class mean geometric equilibrium---property (b) in Section~\ref{sec:prelim}.
The higher, the better separation between different classes, i.e., closer to a regular simplex configuration.

\subsubsection{Self-Duality (SD) between mean embeddings and classifier weights}
The SD metric evaluates the alignment between mean embeddings and classifier weights:

\begin{equation}
    \text{SD} :=   
    \left 
    \| 
    \frac{\mathbf{W}^\top}{\left \| \mathbf{W} \right \|_\mathrm{F}}- \frac{\mathbf{M}}{\| \mathbf{M} \|_\mathrm{F}} 
    \right \|_\mathrm{F},
\end{equation}
where $\mathbf{W}$ and $\mathbf{M}$ are defined as in (\ref{eq:NC3}).
A smaller $\text{SD}$ value indicates better alignment between class means and classifier weights,
i.e., better CC-GE---property (c) in Section~\ref{sec:prelim}.

\subsubsection{Visualization tool}

We qualitatively analyze representation learning by visualizing representations. We first normalizing each feature vector to have unit $\ell_2$ norm and then projecting them onto a two-dimensional plane using principal component analysis.
This visualization tool---that projects high-dimensional features onto the two-dimensional unit sphere---is widely used in many CL studies \cite{cp1, BCL}.

\begin{table}[t]
\centering
\caption{Top-1 accuracy (\%) comparisons between different supervised (C)L methods with CIFAR-10-LT and CIFAR-100-LT datasets.}
\label{tab:cifar}
\resizebox{\columnwidth}{!}{%
\begin{tabular}{c|ccc|ccc}
\toprule
\hline
\multirow{2}{*}{Methods} & \multicolumn{3}{c|}{CIFAR-10-LT}                 & \multicolumn{3}{c}{CIFAR-100-LT}                 \\ \cline{2-7} 
                           & $\rho$=100 & 50    & 10    & $\rho$=100 & 50    & 10    \\ \hline
CB-Focal \cite{CB-Focal}   & 74.5      & 79.2 & 87.1 & 39.6      & 45.1 & 57.9 \\
BBN \cite{BBN}             & 79.8      & 81.1 & 88.3 & 42.5      & 47.0 & 59.1 \\
ResLT \cite{ResLT}         & 82.4      & 85.1 & 89.7 & 48.2      & 52.7 & 62.0 \\
Hybrid-SC \cite{hybrid-sc} & 81.4      & 85.3 & 91.1 & 46.7      & 51.8 & 63.0 \\
TSC \cite{TSC}             & 79.7       & 82.9  & 88.7  & 42.8       & 46.3  & 57.6  \\
CoNe \cite{cone}         & 25.2           & 37.2           & 41.1          & -$^\dagger$    & -$^\dagger$    & -$^\dagger$    \\
GLMC \cite{GLMC}           & 87.7      & 90.1 & 94.0 & 55.8      & 61.0 & 70.7 \\
BCL \cite{BCL}             & 84.3      & 87.2 & 91.1 & 51.9      & 56.5 & 64.8 \\
PaCo \cite{PaCo}           &     77.2       &   79.8    &   83.4    & 52.0       & 56.0  & 64.2  \\
ProCo \cite{Probco}         & 85.9       & 88.2  & 91.9  & 52.8       & 57.1  & 65.5  \\
GPaCo \cite{GPaco}         & 79.2      & 81.6 & 85.9 & 52.3       & 56.4  & 65.4  \\ \hline
\textbf{ECL} (ours)      & \textbf{92.1} & \textbf{92.5} & \textbf{94.4} & \textbf{66.3} & \textbf{67.5} & \textbf{71.0} \\ \hline
\bottomrule
\end{tabular}%
}
\vspace{0.5pc}

{\centering\scriptsize $^\dagger$We failed to run CoNe with CIFAR-100-LT -- we observed the \textsf{NaN} in training loss. \par}
\end{table}

\begin{table}[t]
\centering
\caption{Top-1 accuracy (\%) comparisons between different (C)L methods with ImageNet-LT dataset.}
\label{tab:imagenet}
\resizebox{7cm}{!}{%
\begin{tabular}{c|ccc|c}
\toprule
\hline
Methods                  & Many          & Medium        & Few           & All           \\ \hline
CB-Focal \cite{CB-Focal} & 39.6          & 32.7          & 16.8          & 33.2          \\
LDAM \cite{LDAM}         & 60.4          & 46.9          & 30.7          & 49.8          \\
LADE \cite{LADE}         & 65.1          & 48.9          & 33.4          & 53.0          \\
ResLT \cite{ResLT}       & 63.6          & 55.7          & 38.9          & 56.1          \\
KCL \cite{kcl}           & 61.8          & 49.4          & 30.9          & 51.5          \\
TSC \cite{TSC}           & 63.5          & 49.7          & 30.4          & 52.4          \\
GLMC \cite{GLMC}         & \textbf{70.1} & 52.4          & 30.4          & 56.3          \\
BCL \cite{BCL}           & 67.9          & 54.2          & 36.6          & 57.1          \\
PaCo \cite{PaCo}         & 64.4          & 55.7          & 33.7          & 56.0          \\
ProCo \cite{Probco}       & 68.2          & 55.1          & 38.1          & 57.8          \\
GPaCo \cite{GPaco}       & 64.7          & 56.2          & 39.5          & 58.0          \\ \hline
\textbf{ECL} (ours)      & 68.6          & \textbf{56.3} & \textbf{40.1} & \textbf{58.8} \\ \hline
\bottomrule
\end{tabular}%
}
\end{table}

\begin{table}[t]
\centering
\caption{Top-1 accuracy (\%) comparisons between different supervised CL methods with ISIC 2019 and LCCT datasets.}
\label{tab:medical}
\resizebox{5cm}{!}{%
\begin{tabular}{c|c|c}
\toprule
\hline
Methods          & \begin{tabular}[c]{@{}c@{}}ISIC 2019\\ ($\rho \approx 54$)\end{tabular} & \begin{tabular}[c]{@{}c@{}}LCCT\\ ($\rho \approx 3$)\end{tabular} \\ \hline
TSC \cite{TSC}              & 50.23                                                                   & 63.6                                                               \\
CoNe \cite{cone}            & 63.57                                                                   & 74.93                                                              \\
GLMC \cite{GLMC}            & 70.49                                                                   & 59.33                                                              \\
BCL \cite{BCL}             & 83.76                                                                   & 80.24                                                               \\
GPaCo \cite{GPaco}             & 85.66                                                                   & 74.52                                                               \\

\hline
\textbf{ECL} (ours) & \textbf{87.31}                                                          & \textbf{92.36}                                                     \\ \hline
\bottomrule
\end{tabular}%
}
\end{table}

\begin{table}[t!]
\centering
\caption{Comparisons of $\text{FC}$, $\text{MS}$, and $\text{SD}$ values (see their definitions in Section~\ref{sec:metric}) between different supervised CL methods with CIFAR-10-LT and LCCT datasets (the $\downarrow$ and $\uparrow$ symbols indicate that the lower the better and 
higher the better, respectively).}
\label{tab:metric}
\resizebox{8cm}{!}{%
\begin{tabular}{c|ccc|ccc}
\toprule
\hline
\multirow{2}{*}{Methods} & \multicolumn{3}{c|}{CIFAR-10-LT ($\rho = 100$)}                 & \multicolumn{3}{c}{LCCT}                       \\ \cline{2-7} 
                         & $\text{FC}^{\downarrow}$              & $\text{MS}^{\uparrow}$              & $\text{SD}^{\downarrow}$              & $\text{FC}^{\downarrow}$             & $\text{MS}^{\uparrow}$             & $\text{SD}^{\downarrow}$              \\ \hline
TSC \cite{TSC}                     & 0.95          & 0.30          & 1.42          & 0.91          & 0.56          & 0.97          \\
CoNe \cite{cone}                    & 1.20          & 0.44          & 0.62          & 0.47          & 0.62          & 0.38          \\
GLMC \cite{GLMC}                    & 0.86          & 0.89          & 0.71          & 0.88          & 0.71          & 0.69          \\
BCL \cite{BCL}                     & 0.64          & 0.62          & 0.55          & 0.41          & 0.82          & 0.73          \\ 
GPaCo \cite{GPaco}                     & 0.82          & 0.75          & 0.42          & 0.75          & 0.98          & 0.27          \\ \hline
\textbf{ECL} (ours)     & \textbf{0.36} & \textbf{0.94} & \textbf{0.25} & \textbf{0.23} & \textbf{1.61} & \textbf{0.21} \\ \hline
\bottomrule
\end{tabular}%
}
\end{table}

\begin{figure*}[t!]
    \centering
    \begin{tikzpicture}
        \node[anchor=center, inner sep=0] at (-0.25\textwidth,0) {\includegraphics[width=0.25\textwidth]{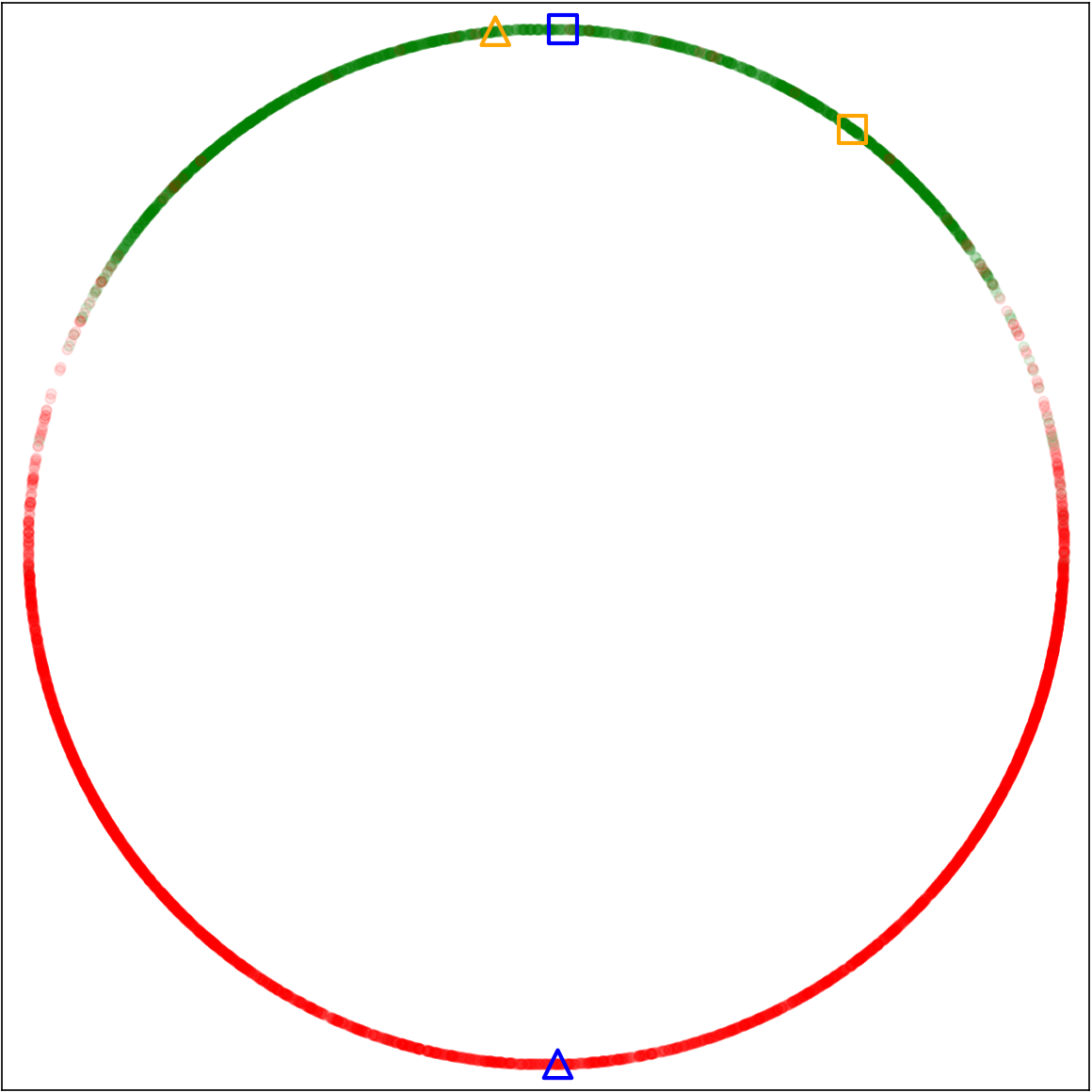}};
        \node[align=center] at (-0.25\textwidth,0.2) {\small $\text{FC}$: 0.66};
        \node[align=center] at (-0.25\textwidth,-0.15) {\small $\text{MS}$: 1.27};
        \node[align=center] at (-0.25\textwidth,-0.5) {\small $\text{SD}$: 0.81};
        \node[anchor=center, inner sep=0] at (0,0) 
        {\includegraphics[width=0.25\textwidth]{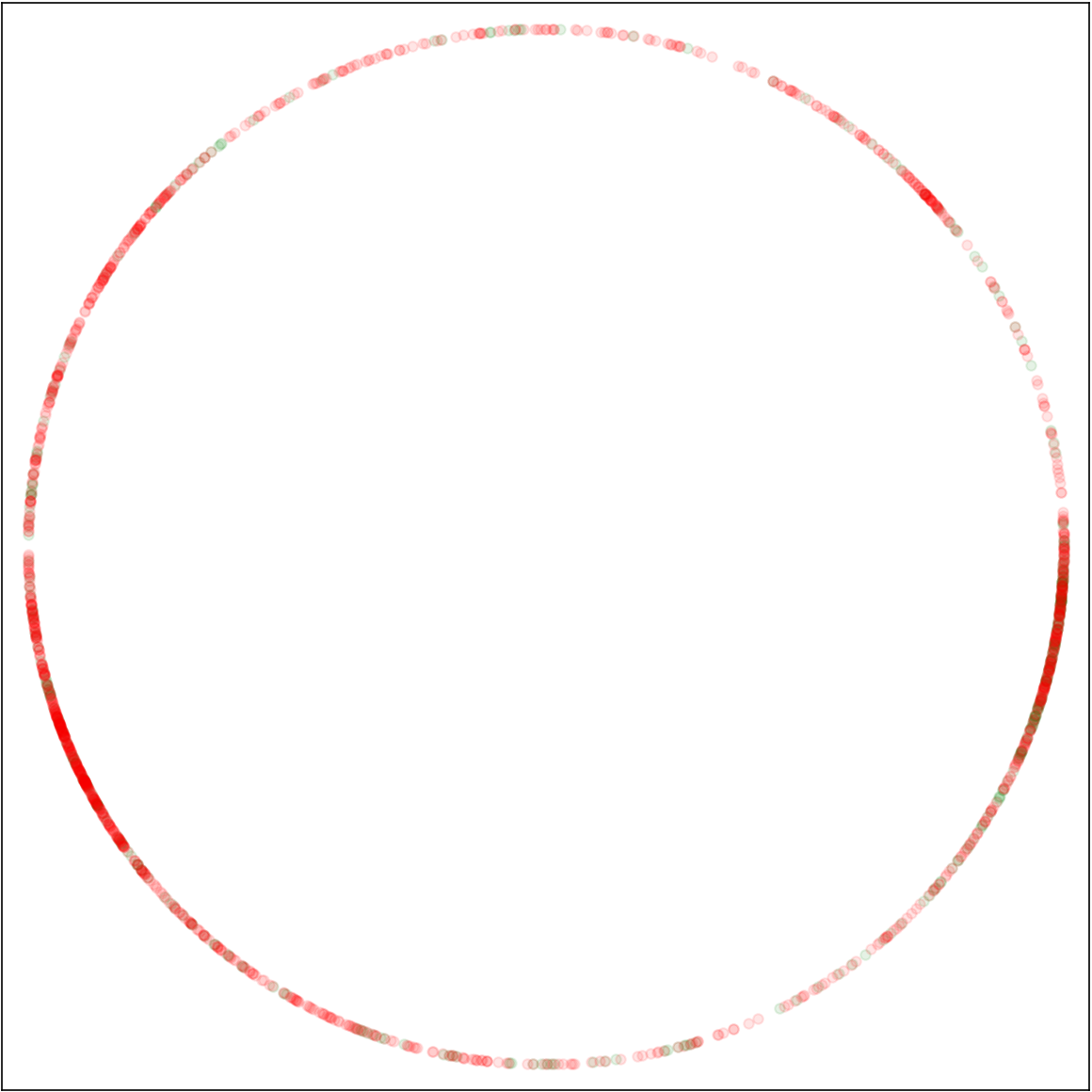}};
        \node[align=center] at (0,0.2) {\small $\text{FC}$: 0.91};
        \node[align=center] at (0,-0.15) {\small $\text{MS}$: 0.56};
        \node[align=center] at (0,-0.5) {\small $\text{SD}$: 0.97};
        \node[align=center] at (-0.125\textwidth,-2.6) {\small (a) TSC \cite{TSC}};
        
        \node[anchor=center, inner sep=0] at (0.25\textwidth,0) 
        {\includegraphics[width=0.25\textwidth]{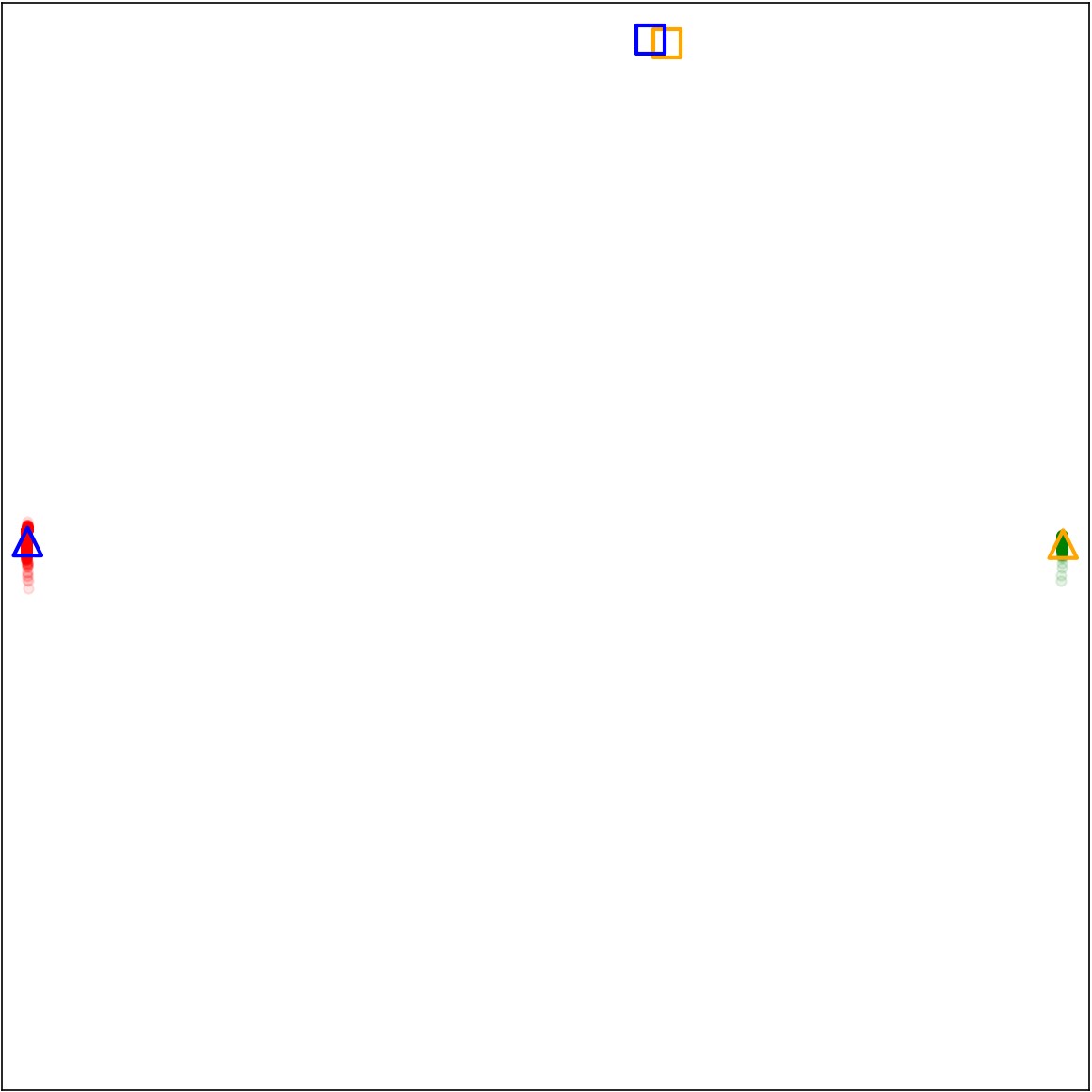}};
        \node[align=center] at (0.25\textwidth,0.2) {\small $\text{FC}$: 0.02};
        \node[align=center] at (0.25\textwidth,-0.15) {\small $\text{MS}$: 1.95};
        \node[align=center] at (0.25\textwidth,-0.5) {\small $\text{SD}$: 0.51};
        
        \node[anchor=center, inner sep=0] at (0.5\textwidth,0) {\includegraphics[width=0.25\textwidth]{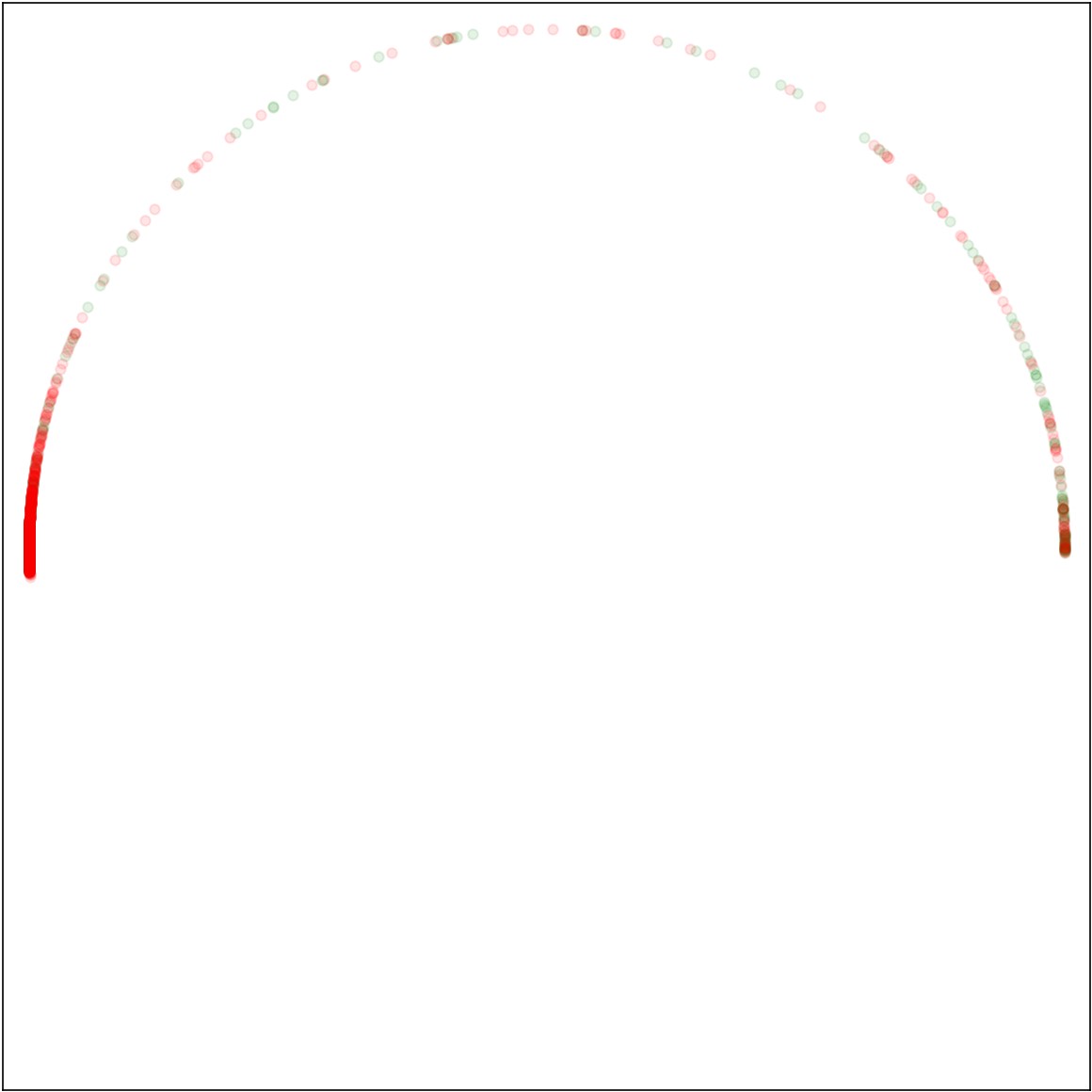}};
        \node[align=center] at (0.5\textwidth,0.2) {\small $\text{FC}$: 0.41};
        \node[align=center] at (0.5\textwidth,-0.15) {\small $\text{MS}$: 0.82};
        \node[align=center] at (0.5\textwidth,-0.5) {\small $\text{SD}$: 0.73};
        \node[align=center] at (0.375\textwidth,-2.6) {\small (b) BCL \cite{BCL}};
        
        \node[anchor=center, inner sep=0] at (-0.25\textwidth,-5.2) {\includegraphics[width=0.25\textwidth]{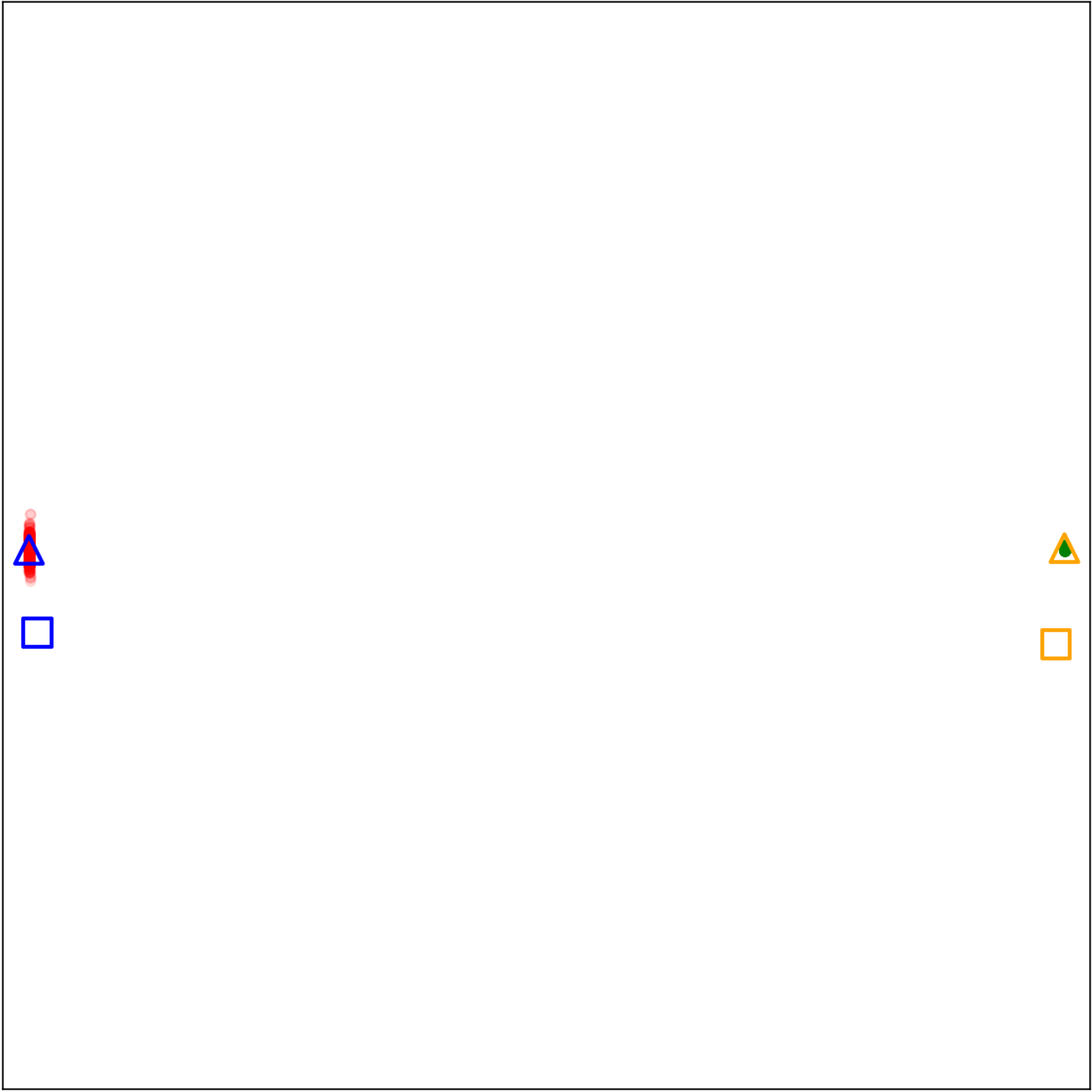}};
        \node[align=center] at (-0.25\textwidth,-5.0) {\small $\text{FC}$: 0.02};
        \node[align=center] at (-0.25\textwidth,-5.35) {\small $\text{MS}$: 1.97};
        \node[align=center] at (-0.25\textwidth,-5.7) {\small $\text{SD}$: 0.24};
        
        \node[anchor=center, inner sep=0] at (0,-5.2) {\includegraphics[width=0.25\textwidth]{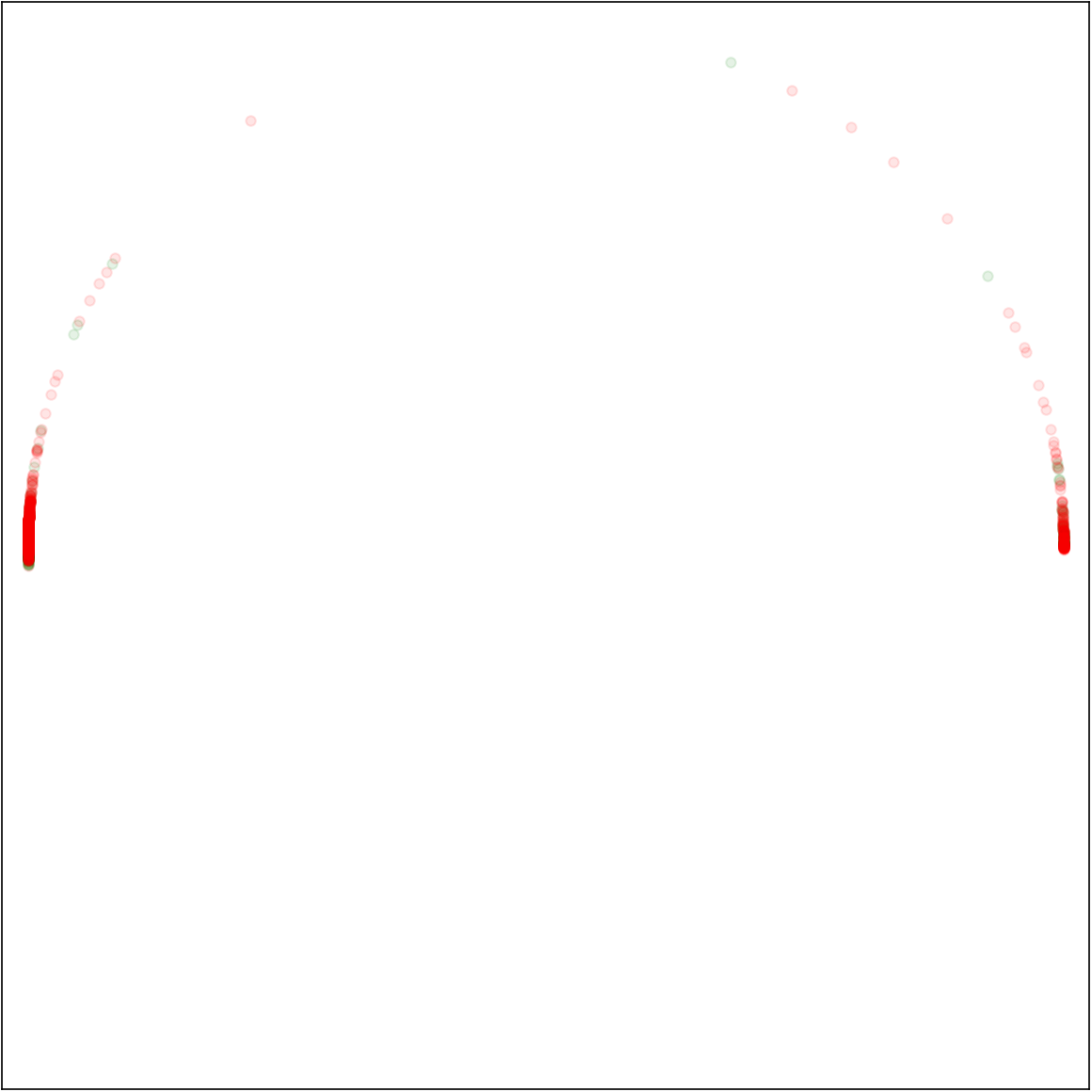}};
        \node[align=center] at (0,-5.0) {\small $\text{FC}$: 0.47};
        \node[align=center] at (0,-5.35) {\small $\text{MS}$: 0.62};
        \node[align=center] at (0,-5.7) {\small $\text{SD}$: 0.38};
        \node[align=center] at (-0.125\textwidth,-7.8) {\small (c) CoNe \cite{cone}};

        \node[anchor=center, inner sep=0] at (0.25\textwidth,-5.2) {\includegraphics[width=0.25\textwidth]{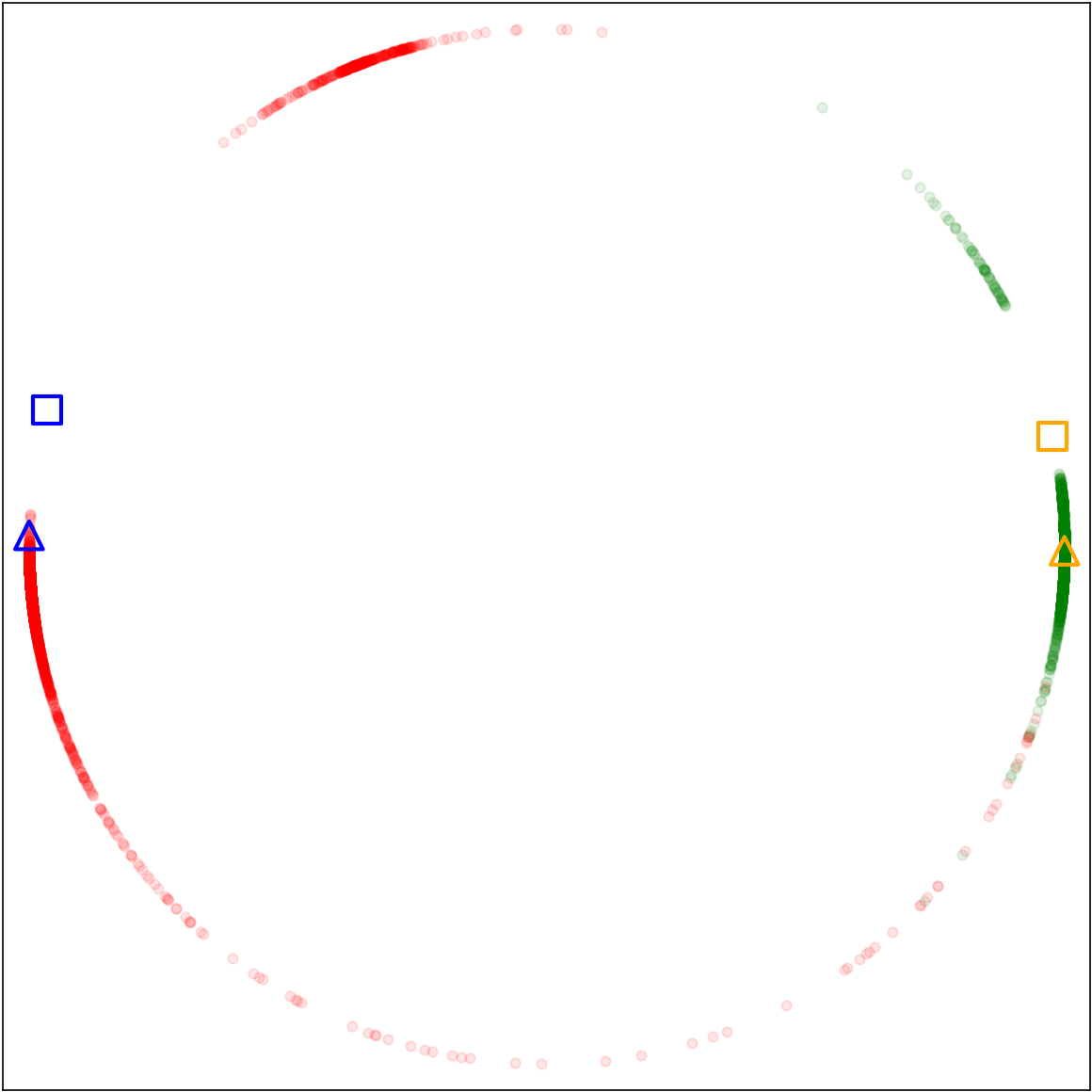}};
        \node[align=center] at (0.25\textwidth,-5.0) {\small $\text{FC}$: 0.49};
        \node[align=center] at (0.25\textwidth,-5.35) {\small $\text{MS}$: 1.76};
        \node[align=center] at (0.25\textwidth,-5.7) {\small $\text{SD}$: 0.56};
        
        \node[anchor=center, inner sep=0] at (0.5\textwidth,-5.2) {\includegraphics[width=0.25\textwidth]{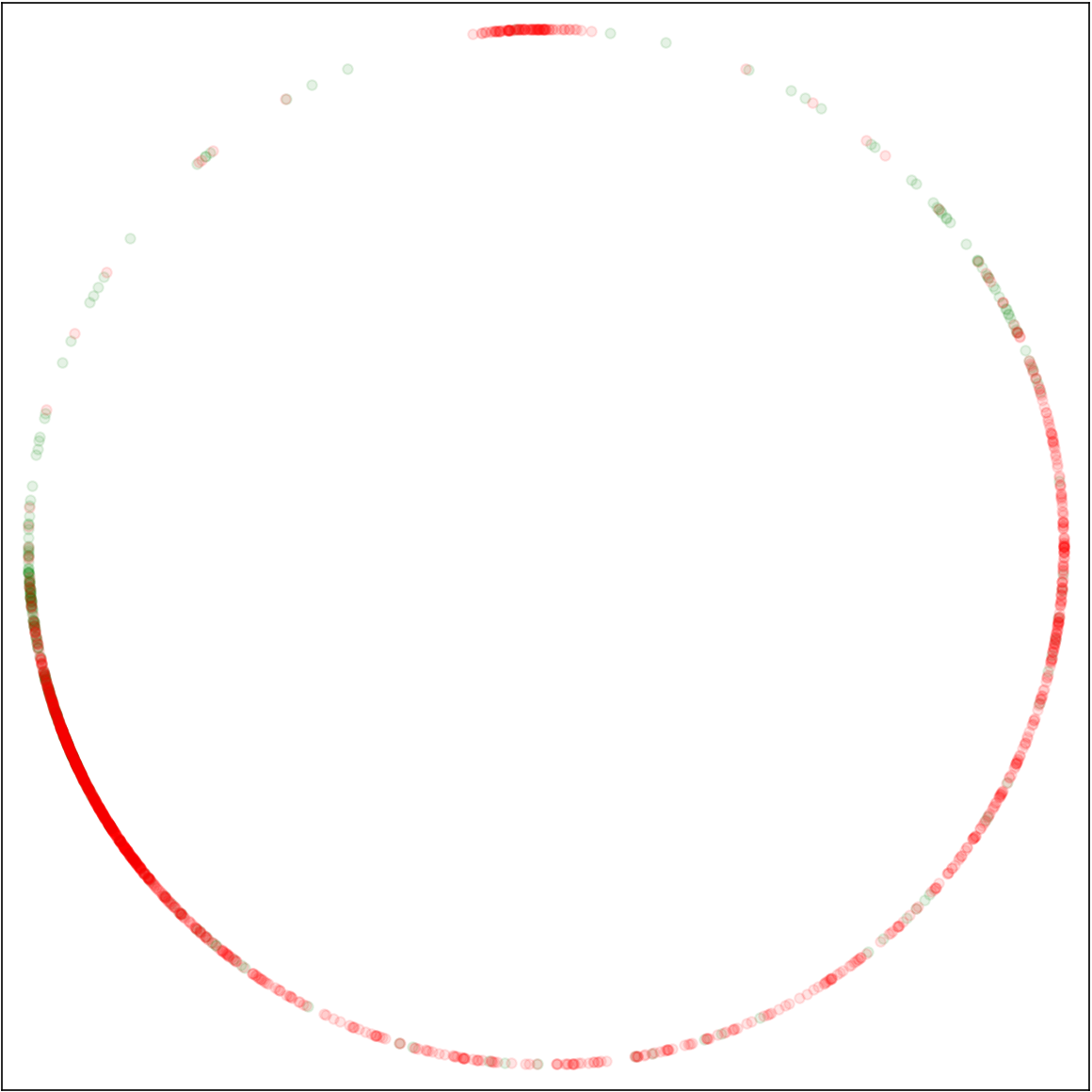}};
        \node[align=center] at (0.5\textwidth,-5.0) {\small $\text{FC}$: 0.88};
        \node[align=center] at (0.5\textwidth,-5.35) {\small $\text{MS}$: 0.71};
        \node[align=center] at (0.5\textwidth,-5.7) {\small $\text{SD}$: 0.69};
        \node[align=center] at (0.375\textwidth,-7.8) {\small (d) GLMC \cite{GLMC}};

        \node[anchor=center, inner sep=0] at (-0.25\textwidth,-10.4) {\includegraphics[width=0.25\textwidth]{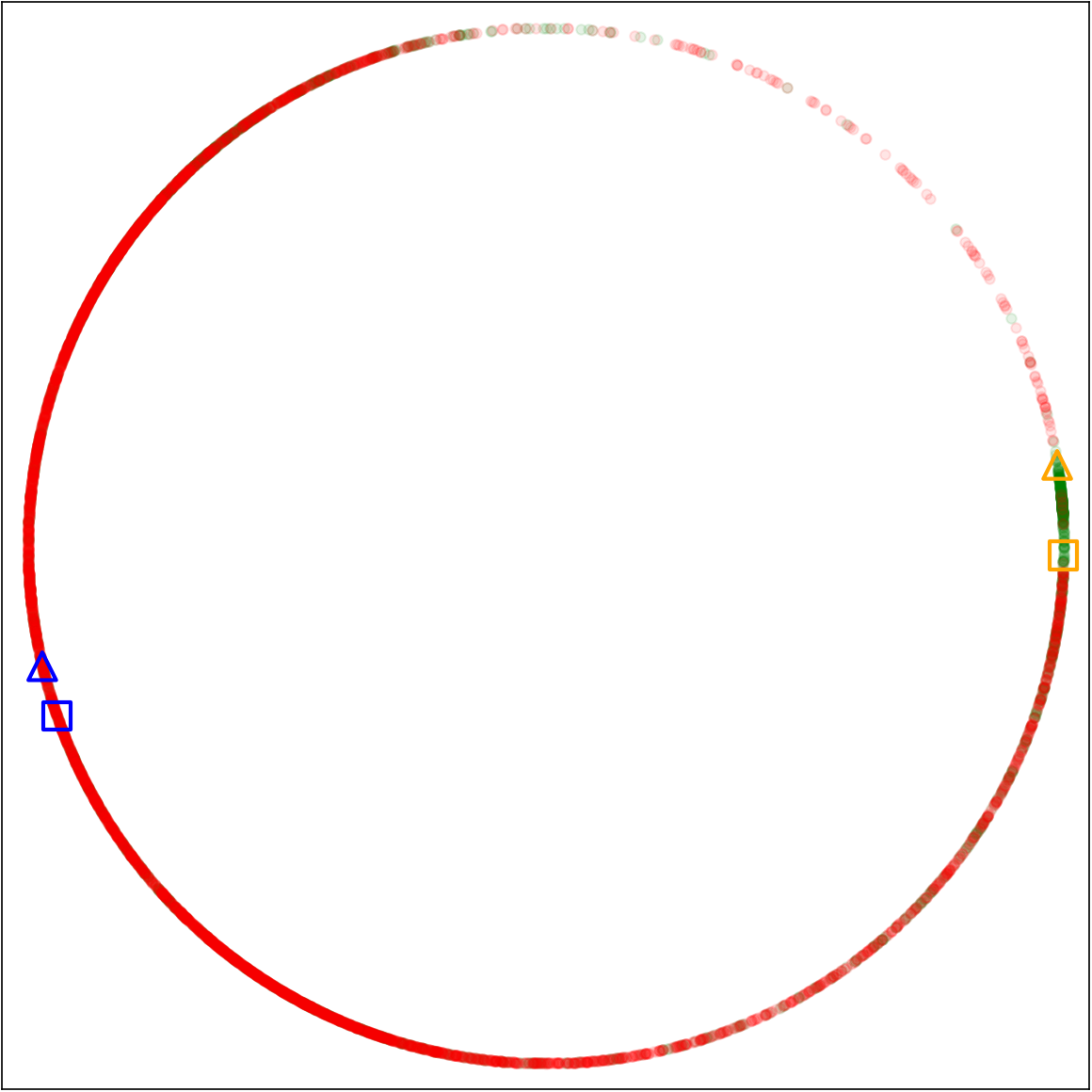}};
        \node[align=center] at (-0.25\textwidth,-10.2) {\small $\text{FC}$: 0.54};
        \node[align=center] at (-0.25\textwidth,-10.55) {\small $\text{MS}$: 1.63};
        \node[align=center] at (-0.25\textwidth,-10.9) {\small $\text{SD}$: 0.19};
        
        \node[anchor=center, inner sep=0] at (0,-10.4) {\includegraphics[width=0.25\textwidth]{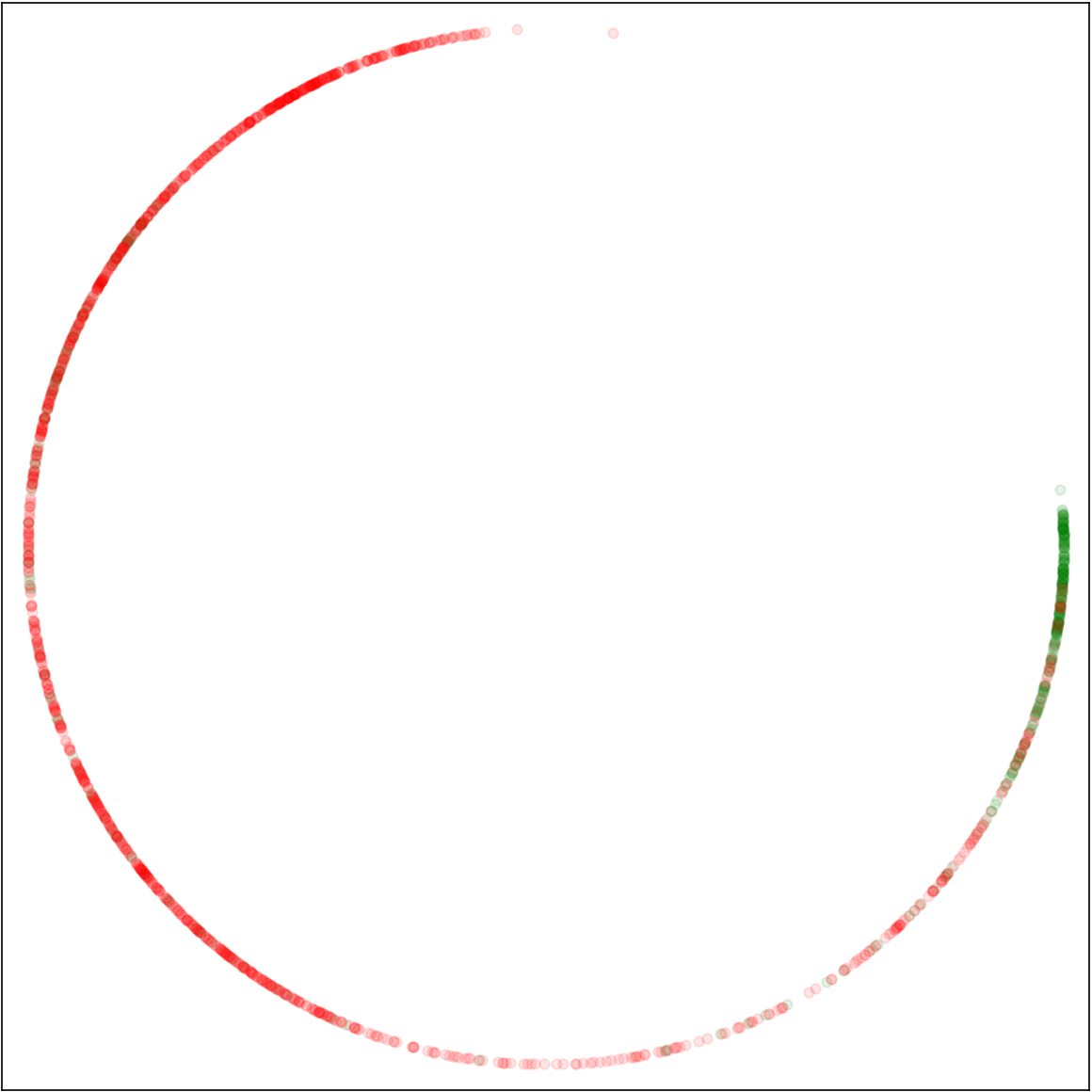}};
        \node[align=center] at (0,-10.2) {\small $\text{FC}$: 0.75};
        \node[align=center] at (0,-10.55) {\small $\text{MS}$: 0.98};
        \node[align=center] at (0,-10.9) {\small $\text{SD}$: 0.27};
        \node[align=center] at (-0.125\textwidth,-13) {\small (e) GPaCo \cite{GPaco}};

        \node[anchor=center, inner sep=0] at (0.25\textwidth,-10.4) {\includegraphics[width=0.25\textwidth]{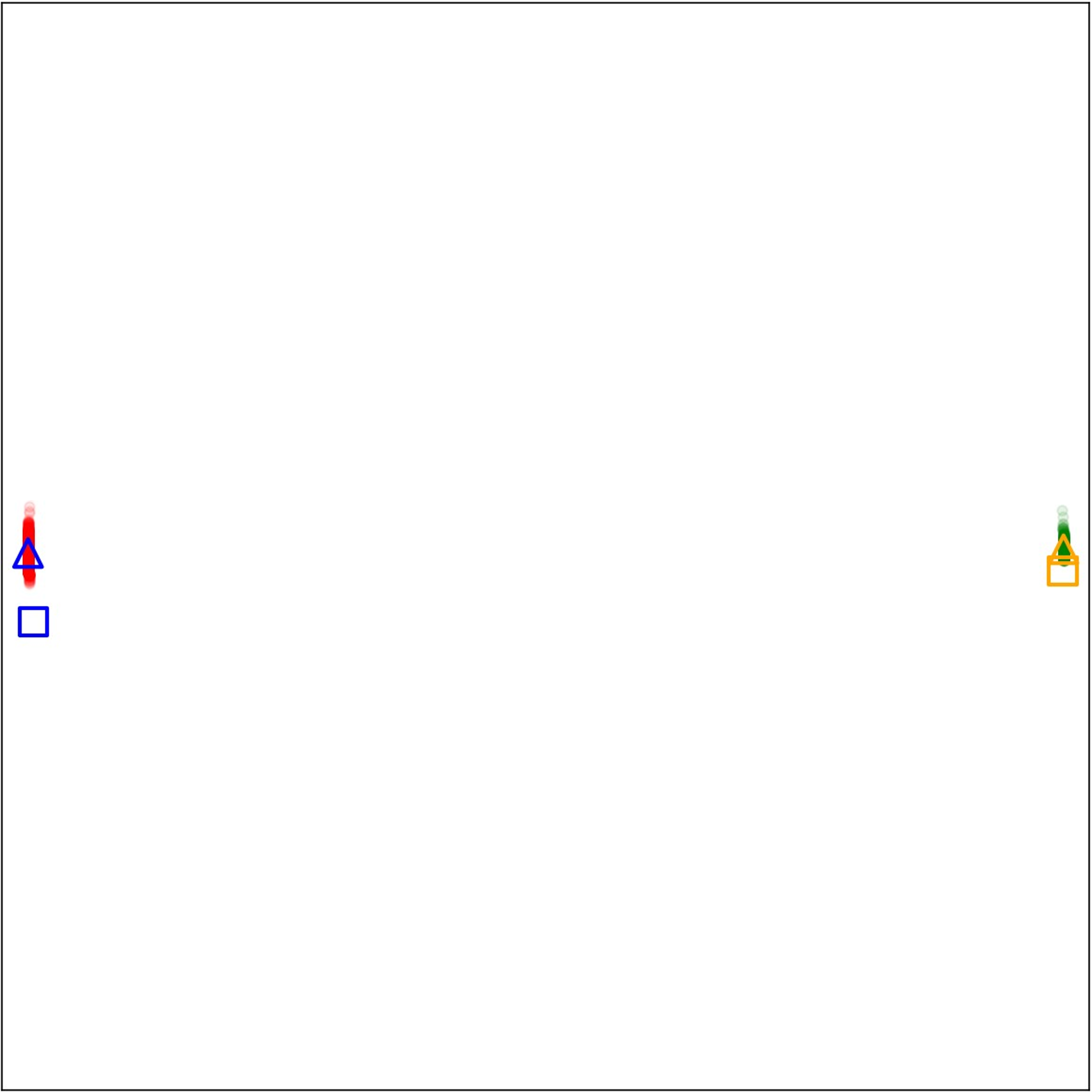}};
        \node[align=center] at (0.25\textwidth,-10.2) {\small $\text{FC}$: 0.04};
        \node[align=center] at (0.25\textwidth,-10.55) {\small $\text{MS}$: 1.95};
        \node[align=center] at (0.25\textwidth,-10.9) {\small $\text{SD}$: 0.12};
        
        \node[anchor=center, inner sep=0] at (0.5\textwidth,-10.4) {\includegraphics[width=0.25\textwidth]{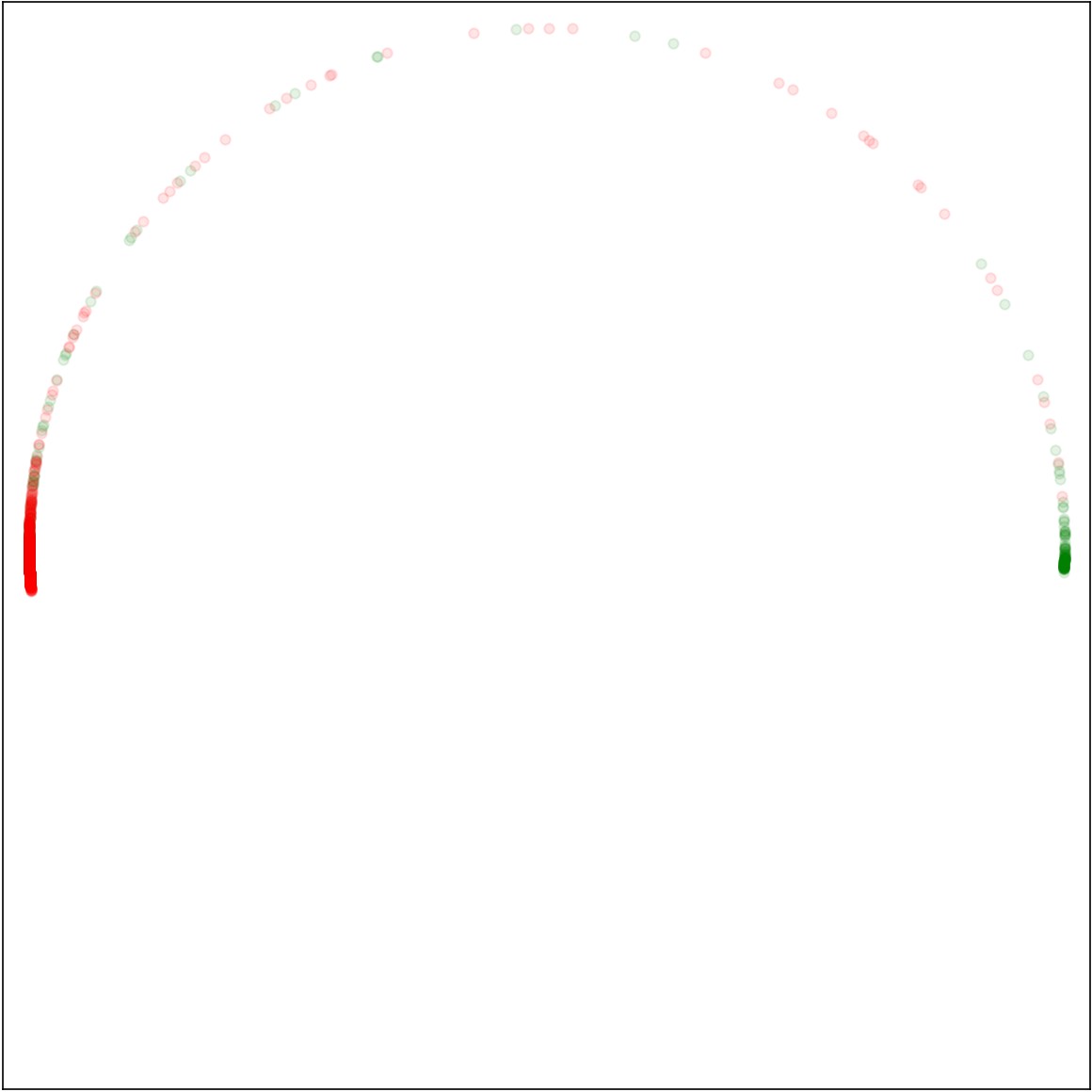}};
        \node[align=center] at (0.5\textwidth,-10.2) {\small $\text{FC}$: 0.23};
        \node[align=center] at (0.5\textwidth,-10.55) {\small $\text{MS}$: 1.61};
        \node[align=center] at (0.5\textwidth,-10.9) {\small $\text{SD}$: 0.21};
        \node[align=center] at (0.375\textwidth,-13) {\small (f) \textbf{ECL}};

        \node[align=center] at (-0.25\textwidth,2.5) {\small Training set};
        \node[align=center] at (0,2.5) {\small Test set};
        \node[align=center] at (0.25\textwidth,2.5) {\small Training set};
        \node[align=center] at (0.5\textwidth,2.5) {\small Test set};

    \end{tikzpicture}

  \caption{
  Visualization of feature distribution on the unit hypersphere for different supervised CL methods with LCCT dataset; $\{ x \in [-1, 1], y \in [-1, 1] \}$. 
  If class-wise features are clustered (see Fig.~\ref{fig:properties}(a)) to form a regular simplex configuration (see Fig.~\ref{fig:properties}(b)) and their class means are aligned with the classifier weights (see Fig.~\ref{fig:properties}(c)), then good representations and classifier are learned.
  \textit{Legends:} 
  \protect\symDot{FeatGreen} denotes minority class samples;
  \protect\symDot{FeatRed} denotes majority class samples;
  \protect\symTri{MeanOrange} denotes minority class mean;
  \protect\symTri{WBlue} denotes majority class mean;
  \protect\symSq{MeanOrange} denotes minority classifier weight; and
  \protect\symSq{WBlue} denotes majority classifier weight.
}
    \label{fig:visualization}
    \vspace{-1pc} 
\end{figure*}
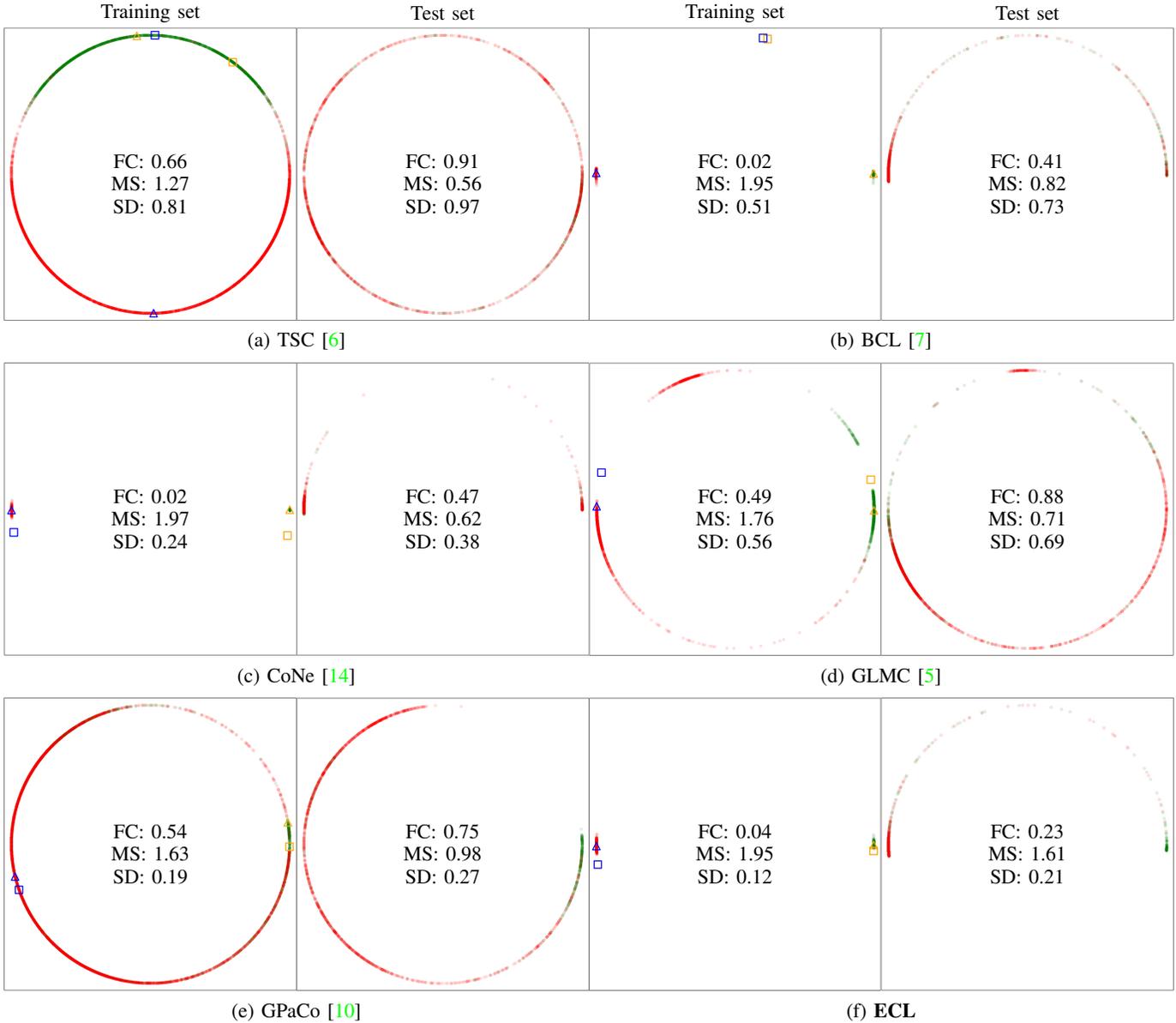

\section{Results and discussion}

This section presents and discusses the results obtained with the five imbalanced/long-tailed datasets in Section~\ref{sec:data}.
Section~\ref{sec:acc} compares the Top-1 classification accuracy
between different supervised CL methods, and analyzes their representation learning with the three metrics in Section~\ref{sec:metric}.
Section~\ref{sec:ablation} conducts ablation studies of proposed ECL (\ref{eq:total}), particularly to evaluate the contribution of each component in (\ref{eq:total}) and study different forms of CL and prototypes within the ECL framework.
Section~\ref{sec:clf} investigates how the CC-GE property behaves with linear and nonlinear classifiers under class imbalance.

\subsection{Comparisons between different supervised (C)L methods with imbalanced datasets}
\label{sec:acc}

\subsubsection{Classification performance comparisons between different supervised (C)L methods}

TABLES~\ref{tab:cifar}--\ref{tab:medical} demonstrate that the proposed ECL framework achieves significantly better classification performance compared to 
several existing SOTA supervised CL methods (listed in Section~\ref{sec:exp}) and several supervised learning baselines, for all the five imbalanced datasets.

\noindent\textbf{CIFAR-10-LT and CIFAR-100-LT.}
Comparing results with the CIFAR-10-LT and CIFAR-100-LT datasets with different $\rho$ values shows that the performance gap between existing methods and ECL becomes larger when $\rho$ is larger.
See TABLE~\ref{tab:cifar}.
This implies that ECL is in general particularly useful for handling extremely long-tailed datasets.
ECL also significantly improved precision and recall performance over existing supervised CL methods; see the CIFAR-10-LT results in Section~S3-A of the Supplementary Material.

\noindent\textbf{ImageNet-LT.}
Comparing results with ImageNet-LT in TABLE~\ref{tab:imagenet} shows that the proposed ECL outperforms existing SOTA supervised CL methods and supervised learning baselines.
In particular, ECL achieved significantly higher classification performance on Few-shot classes, implying that ECL is useful for handling minority classes.

\noindent\textbf{ISIC2019 and LCCT.}
TABLE~\ref{tab:medical} shows the outperforming performance of the proposed ECL over the five aforementioned SOTA methods.
The confusion matrices in Section~S3-B of the Supplementary Material show that ECL achieves low false-negative rates for minority classes while avoiding excessive false positives, reflecting its ability to maintain a more stable trade-off between sensitivity and specificity.
In addition, the receiver operating characteristic (ROC) and precision-recall (PR) curves in Section~S3-C of the Supplementary Material show that ECL effectively bridges the gap between general classification and minority-class detection, 
by limiting false positives while preserving a favorable precision–recall trade-off for both majority and minority classes.

\subsubsection{Representation learning comparisons with different supervised CL methods}

TABLE~\ref{tab:metric} with the two imbalanced datasets shows that ECL better induces the three key properties in Section~\ref{sec:prelim} compared to the five SOTA methods.
The results in TABLE~\ref{tab:metric} well correspond to the classification performances in TABLES~\ref{tab:cifar}--\ref{tab:medical}, underlining the importance of promoting three key properties in Section~\ref{sec:prelim} under class imbalance.

The visualizations with the representation learning perspective in Fig.~\ref{fig:visualization} can explain the outperforming performance of ECL over existing supervised CL methods under class imbalance.
First, we observe that ECL better encourages the feature distribution similarity between training and test sets compared to the existing methods, improving the performance in imbalanced test sets.
Second, ECL better promotes all three key properties compared to the existing methods.
Comparing Fig.~\ref{fig:visualization}(b) with Fig.~\ref{fig:visualization}(e) particularly demonstrates that ECL significantly improves the self-duality over BCL.

\subsection{Ablation studies for proposed ECL} \label{sec:ablation}

This section studies the followings:
\begin{itemize}
\item Section~\ref{subsec:ablation} studies the contribution of each component in the proposed ECL framework (\ref{eq:total}); 
\item Section~\ref{sec:comparison_cl} studies ECL with its variant that replaces proposed BC-ECL loss (\ref{eq:ECL}) with the existing BCL loss (\ref{eq:BCL});
\item Section~\ref{sec:mean_vs_proto} studies different forms of prototypes in the the proposed CC-GE promoting loss (\ref{eq:CC-GE}).
\end{itemize}

\subsubsection{Ablation study of ECL (\ref{eq:total})} 
\label{subsec:ablation}

Several ablation studies in TABLE~\ref{tab:ablation} investigate the contribution of each component in the proposed ECL loss (\ref{eq:total}) by calculating the Top-1 classification accuracy and the $\text{FC}$, $\text{MS}$, and $\text{SD}$ measures defined in Section~\ref{sec:metric}.

TABLE~\ref{tab:ablation} shows that 
in imbalanced learning,
to improve all the $\text{FC}$, $\text{MS}$, $\text{SD}$ measures,
i.e., to simultaneously promote the intra-class feature collapse, inter-class  mean geometric equilibrium, and CC-GE properties in Section~\ref{sec:prelim}, 
one needs to use all the components $\mathcal{L}_{\text{BC-ECL}}$ (\ref{eq:ECL}), $\mathcal{L}_{\text{LC}}$ (\ref{eq:lc}), and $\mathcal{L}_{\text{CC-GE}}$ (\ref{eq:CC-GE}).
For example, comparing the ``$\mathcal{L}_{\text{LC}}$'' setup with the ``$\mathcal{L}_{\text{LC}}$+$\mathcal{L}_{\text{CC-GE}}$'' combination in TABLE~\ref{tab:ablation} demonstrates that 
adding only $\mathcal{L}_{\text{CC-GE}}$ to $\mathcal{L}_{\text{LC}}$ does not improve the $\text{SD}$ measure (even if it is designed to improve $\text{SD}$).
Comparing ``$\mathcal{L}_{\text{BC-ECL}}$+$\mathcal{L}_{\text{LC}}$'' with 
``$\mathcal{L}_{\text{BC-ECL}}$+$\mathcal{L}_{\text{LC}}$+$\mathcal{L}_{\text{CC-GE}}$'' in TABLE~\ref{tab:ablation} verifies the effectiveness of the proposed CC-GE loss (\ref{eq:CC-GE}) in improving not only the $\text{SD}$ measures, but also $\text{FC}$ and $\text{MS}$ measures.

In addition, TABLE~\ref{tab:ablation} supports that in imbalanced learning, promoting all three properties can significantly improve the classification performance over the ablated variants.

By comparing the first and last rows of TABLE~\ref{tab:ablation},
we observe that the performance gain of full ECL over the single-classifier learning loss $\mathcal{L}_{\text{LC}}$ is more pronounced on CIFAR-10-LT ($\rho = 100$) than on the LCCT datasets.
This is because the geometric equilibrium in the representation space---illustrated in Fig.~\ref{fig:configuration}---is more easily promoted on LCCT than on the highly imbalanced CIFAR-10-LT ($\rho = 100$).

\subsubsection{Can the proposed $\mathcal{L}_{\text{BC-ECL}}$ \R{eq:ECL} address the limitation of the existing $\mathcal{L}_{\text{BCL}}$ \R{eq:BCL}?}
\label{sec:comparison_cl}

In a nutshell, 
the full ECL model using the proposed CL formulation $\mathcal{L}_{\text{BC-ECL}}$ \R{eq:ECL} consistently outperforms its variant that replaces it with the BCL formulation $\mathcal{L}_{\text{BCL}}$ \R{eq:BCL} across Many, Medium, and Few-shot classes as well as on All classes. 
See the last two rows of TABLE~\ref{tab:comparison_cl}.

By comparing the baseline BCL \cite{BCL}, the aforementioned variant of ECL, and the full ECL model,
we now examine the impact of batch-dependent prototype contributions in BCL.
We first compare the baseline model BCL \cite{BCL}
with the variant of ECL---namely, a model whose differences from BCL are the
addition of the proposed CC-GE loss and linearly transformed prototypes---in TABLE~\ref{tab:comparison_cl}.
The results show that, 
while the variant yields remarkable improvements on Few-shot classes, 
it offers \emph{no} performance gain on the Many- and Medium-shot classes.
This supports our conjecture in Section~\ref{sec:motivation:batch} that
batch-dependent prototype contributions can limit performance improvements for classes with many instances, e.g., Many- and Medium-shot groups.
In contrast, comparing full ECL with its variant in
TABLE~\ref{tab:comparison_cl} confirms that the proposed CL loss,
$\mathcal{L}_{\text{BC-ECL}}$ \R{eq:ECL}, effectively resolves this limitation inherent in $\mathcal{L}_{\text{BCL}}$ \R{eq:BCL}.

\begin{table}[t!]
\centering
\caption{
Ablation study for the primary components of proposed ECL (the $\downarrow$ and $\uparrow$ symbols indicate that the lower the better and 
and higher the better, respectively). 
}
\label{tab:ablation}
\resizebox{\columnwidth}{!}{
\begin{tabular}{ccc|cccc|cccc}
\toprule
\hline
$\mathcal{L}_{\text{BC-ECL}}$ & $\mathcal{L}_{\text{CC-GE}}$ & $\mathcal{L}_{\text{LC}}$ & \multicolumn{4}{c|}{CIFAR-10-LT ($\rho = 100$)}                                  & \multicolumn{4}{c}{LCCT}                                          \\ \cline{4-11} 
(\ref{eq:ECL}) & (\ref{eq:CC-GE})  & (\ref{eq:lc}) & Top-1 acc. & \multicolumn{1}{c}{$\text{FC}^{\downarrow}$} & \multicolumn{1}{c}{$\text{MS}^{\uparrow}$} & $\text{SD}^{\downarrow}$    & Top-1 acc. & \multicolumn{1}{c}{$\text{FC}^{\downarrow}$} & \multicolumn{1}{c}{$\text{MS}^{\uparrow}$} & $\text{SD}^{\downarrow}$    \\ \hline
\xmark &
  \xmark &
  \cmark &
  82.27 \% &
  0.90 &
  0.39 &
  0.85 &
  79.20 \% &
  0.88 &
  0.52 &
  0.74 \\
\xmark &
  \cmark &
  \cmark &
  87.26 \% &
  0.81 &
  0.37 &
  0.80 &
  81.62\% &
  0.86 &
  0.56 &
  0.72 \\
\cmark &
  \xmark &
  \cmark &
  88.61 \% &
  0.64 &
  0.65 &
  0.86 &
  88.52 \% &
  0.37 &
  0.76 &
  0.71 \\
\cmark &
  \cmark &
  \cmark &
  \textbf{92.13 \%} &
  \textbf{0.36} &
  \textbf{0.94} &
  \textbf{0.35} &
  \textbf{92.36 \%} &
  \textbf{0.23} &
  \textbf{1.61} &
  \textbf{0.21} \\ \hline
\bottomrule
\end{tabular}
}
\end{table}

\begin{table}[t!]
\centering
\caption{Top-1 accuracy (\%) comparisons between ECL and its variant by replacing 
$\mathcal{L}_\text{BC-ECL}$ (\ref{eq:ECL}) with $\mathcal{L}_\text{BCL}$ (\ref{eq:BCL}), with ImageNet-LT dataset.}
\label{tab:comparison_cl}
\resizebox{0.95\columnwidth}{!}{%
\begin{tabular}{c|ccc|c}
\toprule
\hline
Methods                    & Many          & Medium        & Few           & All           
\\ 
\hline
BCL \cite{BCL} (baseline)           & 67.9          & 54.2          & 36.6          & 57.1          \\
Variant of ECL by replacing  & \multirow{2}{*}{67.6}          & \multirow{2}{*}{54.1}          & \multirow{2}{*}{39.6}          & \multirow{2}{*}{57.3}          \\  
$\mathcal{L}_\text{BC-ECL}$ (\ref{eq:ECL}) w/ $\mathcal{L}_\text{BCL}$ (\ref{eq:BCL}) & & & \\
\textbf{ECL} (ours)  & \textbf{68.6} & \textbf{56.3} & \textbf{40.1} & \textbf{58.8} \\
\hline
\bottomrule
\end{tabular}%
}
\end{table}

\begin{table}[t!]
\centering
\caption{Performance comparisons between different forms of prototypes $\mathbf{P}$ in (\ref{eq:CC-GE}) with different datasets.}
\label{tab:prototype}
\resizebox{\columnwidth}{!}{%
\begin{tabular}{c|cc|cc}
\toprule
\hline
\multirow{2}{*}{Methods}                                                      & \multicolumn{2}{c|}{CIFAR-10-LT ($\rho=100$)} & \multicolumn{2}{c}{LCCT}   \\ \cline{2-5} 
                                                                              & Training time    & Top-1 acc.    & Training time & Top-1 acc. \\ \hline
Means (\ref{eq:NC3})                                                                         & -$^\dagger$                & -$^\dagger$             & 11 h           & 81.61 \%     \\
\begin{tabular}[c]{@{}c@{}}Nonlinearly transf.\\ weights of clf.~\cite{BCL}\end{tabular} & 2.5 h            & 90.68 \%        & 7.5 h          & 90.06 \%     \\
\begin{tabular}[c]{@{}c@{}}{\bfseries Linearly transf.}\\ {\bfseries weights of clf.} (ours) \end{tabular}   & \textbf{2 h}            & \textbf{92.13 \%}       & \textbf{7 h}            & \textbf{92.36 \%}     \\ \hline
\bottomrule
\end{tabular}%
}
\vspace{0.5pc}

{\centering\scriptsize $^\dagger$We failed to run experiments with CIFAR-10-LT due to high imbalance factor -- initial minibatch did not include samples from some minority classes. \par}
\end{table}

\subsubsection{Is a linear transformation more effective than a nonlinear one for constructing prototypes in the CC-GE promoting loss (\ref{eq:CC-GE})?}
\label{sec:mean_vs_proto}

We compare different forms of prototypes in the proposed CC-GE loss (\ref{eq:CC-GE}) that are implemented as follows:
\textit{1)} the class means that average the within-class representations;
\textit{2)} the class prototypes obtained by passing the classifier weights through a non-linear MLP transformation \cite{BCL}; and
\textit{3)} the class prototypes obtained by passing the classifier weights through a linear transformation.

Comparing the second and third schemes in TABLE~\ref{tab:prototype} shows that
in producing class prototypes from classifier weights,
using a linear transformation improves both training efficiency and classification performance.
A plausible explanation is as follows.
Constructing prototypes via a linear transformation of classifier weights preserves their directional structure, ensuring that prototype directions remain geometrically consistent with the weight vectors.
In contrast, applying an MLP with piecewise activations such as ReLU may introduce nonlinear distortions that warp angular relationships and create misalignment between prototypes and classifier weights.
Such misalignment weakens the CC-GE property (see Fig.~\ref{fig:properties}(c)), ultimately degrading generalization performance.
Our conclusion is that a linear transformation is preferable to a nonlinear one for constructing prototypes in ECL: it better facilitates CC-GE by avoiding the potential directional distortions that nonlinear mappings can introduce, while also reducing computational overhead.

Comparing the first with the second and third schemes in TABLE~\ref{tab:prototype} supports our claim in Section~\ref{sec:prelim} that using class prototypes, rather than class means, can accelerate training and improve classification performance.

\begin{table}[t]
\centering
\caption{Comparisons of Top-$1$ accuracy, $\text{FC}$, $\text{MS}$, and $\text{SD}$ values between a nonlinear classifier and a linear classifier with CIFAR-10-LT and LCCT datasets.}
\label{tab:clf}
\resizebox{\columnwidth}{!}{%
\begin{tabular}{c|cccc|cccc}
\toprule
\hline
\multirow{2}{*}{Methods} & \multicolumn{4}{c|}{CIFAR-10-LT ($\rho = 100$)}                                                 & \multicolumn{4}{c}{LCCT}                                                                        \\
                         & Top-1 acc.       & $\text{FC}^{\downarrow}$ & $\text{MS}^{\uparrow}$ & $\text{SD}^{\downarrow}$ & Top-1 acc.       & $\text{FC}^{\downarrow}$ & $\text{MS}^{\uparrow}$ & $\text{SD}^{\downarrow}$ \\ \hline
Nonlinear                & 90.65\%          & 0.38                     & 0.88                   & 0.36                     & 89.71\%          & 0.36                     & 1.27                   & 0.33                     \\
\textbf{Linear}          & \textbf{92.13\%} & \textbf{0.36}            & \textbf{0.94}          & \textbf{0.25}            & \textbf{92.36\%} & \textbf{0.23}            & \textbf{1.61}          & \textbf{0.21}            \\ \hline
  \bottomrule
\end{tabular}%
}
\end{table}

\subsection{Can the CC-GE property hold with nonlinear classifiers?}
\label{sec:clf}

Many existing theoretical studies of the three key geometric properties in Section~\ref{sec:prelim}, e.g., \cite{neuralcollapse}, assume a \emph{linear} classifier.
The work \cite{nuclear} theoretically show in \emph{balanced} learning that a two-layer MLP classifier using ReLU activation and no bias can satisfy CC-GE, yet this result requires a strong assumption about nuclear norm equality.
Another theoretical study \cite{NC_deep} demonstrates that in balanced learning, a two- or three-layer MLP classifier with ReLU activation can empirically satisfy CC-GE, 
where they measure (\ref{eq:NC3}) by replacing the classifier weight matrix $\mathbf{W}$ with the product of the weight matrices across layers.

We compare a linear classifier (following the standard convention) with a nonlinear classifier implemented as a three-layer MLP, where ReLU activations and batch normalization layers are inserted between the linear layers, and a bias term is used only in the final linear layer \cite{NC_deep}.
In contrast to the empirical findings of \cite{NC_deep} under balanced learning, 
our results in TABLE~\ref{tab:clf} indicate that the nonlinear classifier does not sufficiently promote CC-GE under class imbalance and likewise fails to adequately promote the two properties of representation geometric equilibrium.
Ultimately, the nonlinear classifier degraded the classification performance relative to our linear classifier.

We conjecture that using a nonlinear classifier hinders the achievement of an overall geometric equilibrium in Fig.~\ref{fig:configuration} that harmoniously balances class representations, class means, and classifier weights in imbalanced learning, ultimately leading to performance degradation.

While nonlinear classifiers do not sufficiently promote the desired geometric equilibrium under the aforementioned experimental setup in ECL, 
future work may explore architectural modifications or training schemes that allow nonlinear classifiers to preserve this equilibrium while still benefiting from their expressive capacity.

\section{Conclusion} 

Imbalanced distributions between classes are easily found in real-world datasets.
Class imbalance degrades generalization performance because it disrupts the geometric equilibrium among representations and classifier weights, leading models to overfit majority classes and overlook rare but societally important minority classes.
This issue becomes even more severe when the dataset is highly imbalanced, resulting in unreliable predictions in domains where safety and trustworthy decision-making are essential.
Several existing supervised CL methods aim to promote a geometric equilibrium among representations from different classes by optionally including class prototypes to more evenly account for all classes.

We propose a new E2E supervised CL framework, ECL, that promotes geometric equilibrium in the representation space by harmoniously balancing class features, class means, and classifier weights.
In ECL, we introduce two complementary strategies that, to the best of our knowledge, have not been explored in prior work.
First, 
we promote a consistent geometric equilibrium of representations across batches by balancing the contributions of class-average features and class prototypes; 
see the proposed BC-ECL formulation in Section~\ref{sec:BC-ECL}.
Second, 
we promote geometric equilibrium between a classifier and class-center representations by mutually aligning classifier weights and class prototypes; 
see the proposed CC-GE promoting loss in Section~\ref{sec:ccge}. 
These two strategies effectively stabilize the geometry of representations and prevent classifier–class mean misalignment, yielding improved generalization.
ECL achieves new benchmarking SOTA performances with four real-world long-tailed benchmark datasets,
and outperforming performance over existing supervised CL methods with a real-world imbalanced medical image dataset that suffers from poor generalization.
By effectively addressing class imbalance, ECL contributes to more reliable and equitable decision-making in safety-critical applications---such as medical diagnosis, autonomous driving, and fraud detection---where failures on minority classes can lead to disproportionately harmful consequences.

Future research directions of ECL can be pursued along several dimensions.
First, we aim to address a limitation of ECL: it still shares the common drawback of supervised CL methods in requiring a large batch size to form diverse positive and negative pairs.
Motivated by ProCo \cite{ProCo}, we plan to modify ECL so that its contrastive objectives are computed against class-wise probability distributions rather than mini-batch samples, while retaining its prototype-based contrasting (which may incur negligible computational overhead).
This would enable ECL to function effectively with a smaller batch size while retaining its geometric equilibrium behavior.
Second, we plan to extend ECL beyond visual domains to general imbalanced modalities such as text and tabular data. 
This can be achieved by incorporating a modality-specific encoder and selecting appropriate augmentation strategies within the proposed ECL framework, 
potentially enabling ECL to serve as a unified representation learning framework across data types.
We believe these extensions will further broaden the applicability and expressiveness of ECL in real-world imbalanced learning scenarios.

{\small
\bibliographystyle{IEEEtran}
\bibliography{ref}
}

\begingroup
\raggedbottom

\begin{IEEEbiography}[{\includegraphics[width=1in,height=1.25in,clip,keepaspectratio]{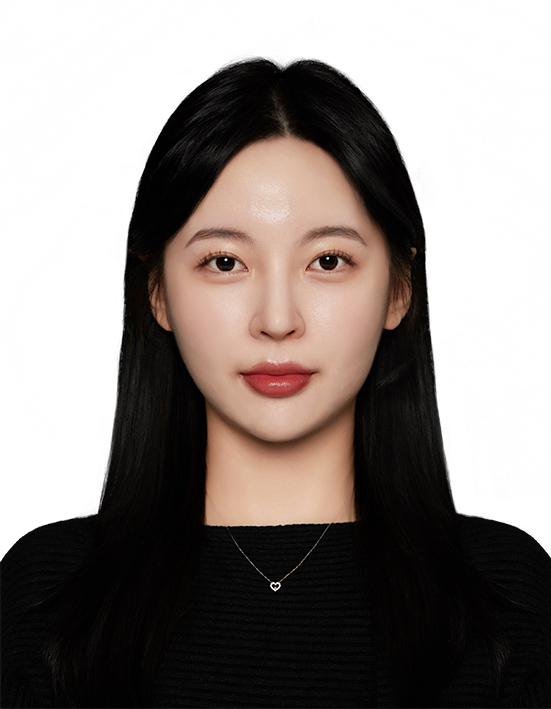}}]{Sumin Roh} 
received the M.Eng.~degree in Electrical and Computer Engineering (ECE), Sungkyunkwan University (SKKU), Suwon, South Korea, in 2025.
She is currently pursuing the Ph.D.~degree in ECE, SKKU, Suwon, South Korea.
Her research interests include long-tailed recognition, contrastive learning, and multimodal learning.
\end{IEEEbiography}

\begin{IEEEbiography}[{\includegraphics[width=1in,height=1.25in,clip,keepaspectratio]{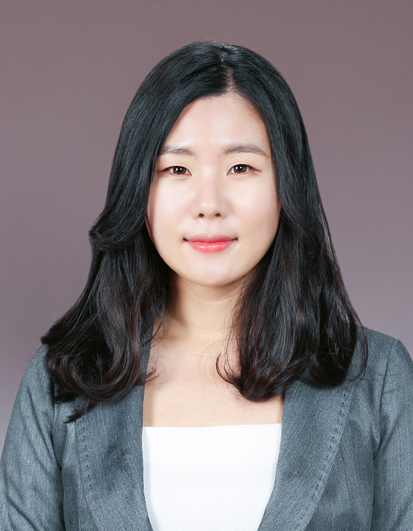}}]{Harim Kim} 
is a 4th year resident currently undergoing training at 
the Department of Radiology, Samsung Medical Center, Seoul, South Korea.
She received her M.D.~degree from Ewha Womans University, Seoul, South Korea (2015–2019).
She also holds a B.A.~degree in English Literature (2011–2015). 
Her research interests focus on tumor
imaging and applications of artificial intelligence to medical imaging.
\end{IEEEbiography}

\begin{IEEEbiography}[{\includegraphics[width=1in,height=1.25in,clip,keepaspectratio]{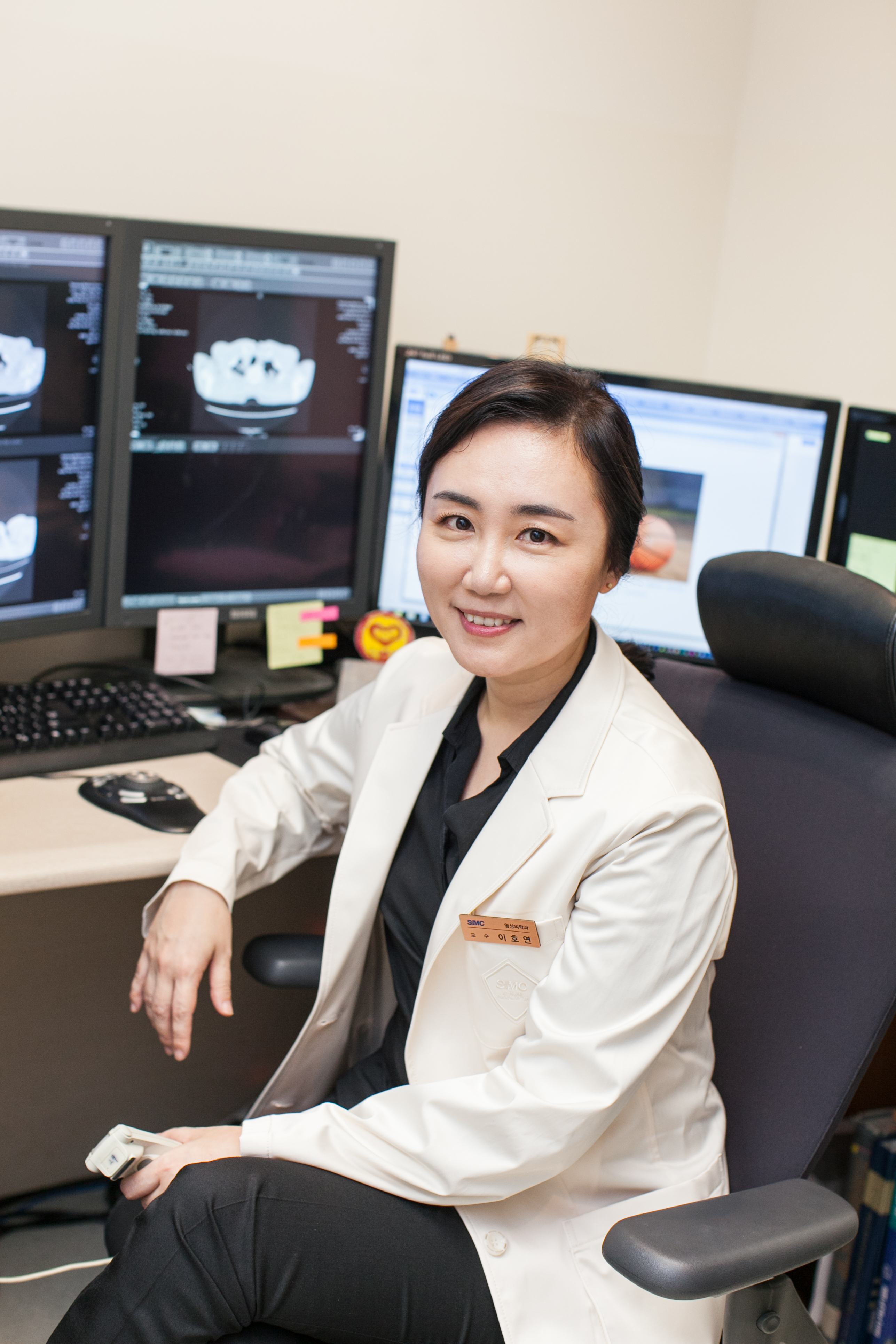}}]{Ho Yun Lee} is a Professor of Radiology (Thoracic Imaging Section) at Samsung Medical Center, Seoul, South Korea, and a Professor at the School of Medicine, Sungkyunkwan University, Suwon, South Korea.
Her research interests include thoracic oncologic imaging and radiomics, with a particular focus on quantitative image analysis and its integration with genomics for prognostic stratification and assessment of treatment response. She currently serves as Chair of the International Liaison Committee of the Korean Society of Thoracic Radiology and Chair of the Planning Committee for the Korean Workshop on Pulmonary Functional Imaging, and is a member of the Fleischner Society.
\end{IEEEbiography}

\begin{IEEEbiography}[{\includegraphics[width=1in,height=1.25in,clip,keepaspectratio]{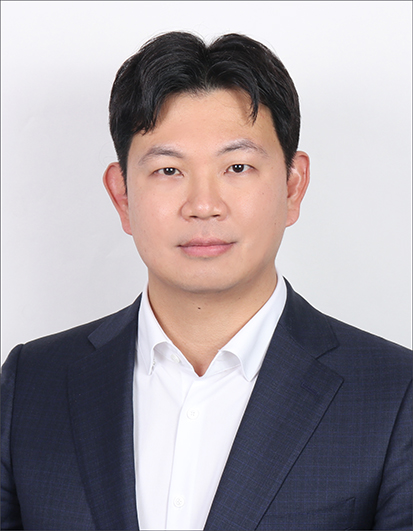}}]{Il Yong Chun} (Member, IEEE) 
received the B.Eng.~degree in Electrical Engineering from Korea University, Seoul, South Korea,
in 2009, and the Ph.D.~degree in Electrical and Computer Engineering from Purdue University, West Lafayette, Indiana, in 2015. 
He is an Associate Professor of
Electrical and Computer Engineering (ECE) at Sungkyunkwan University (SKKU), Suwon, South Korea, from 2024.
He is also affiliated with the Departments of Artificial Intelligence, Advanced Display Engineering, Semiconductor Convergence Engineering, and Display Convergence Engineering, 
SKKU, Suwon, South Korea,
and the Center for Neuroscience Imaging Research, Institute for Basic Science, Suwon, South Korea.
He joined the School of Electronic and Electrical Engineering (now ECE) at SKKU in 2021, as an Assistant Professor. 
Prior to joining SKKU, he was an Assistant Professor of Electrical and Computer Engineering at the University of Hawai’i, Manoa,
a Research Fellow in Electrical Engineering and Computer Science at The University of Michigan, and a Postdoctoral Research Associate in Mathematics at Purdue University, from 2019-2021, 2016-2019, and 2015-2016, respectively.
His research interests include machine learning and artificial intelligence,  
applied to computational imaging and vision.
\end{IEEEbiography}
\endgroup

\newpage


\renewcommand\citeform[1]{S#1}

\renewcommand{\thetable}{S\arabic{table}}
\renewcommand{\thefigure}{S\arabic{figure}}
\renewcommand{\theequation}{S\arabic{equation}}
\renewcommand{\thesection}{S\arabic{section}}
\setcounter{table}{0}
\setcounter{figure}{0}
\setcounter{equation}{0}

\clearpage
\twocolumn[
\begin{@twocolumnfalse}
\begin{center}
\fontsize{20}{22}\selectfont
Supplementary Material for
``Equilibrium contrastive learning for imbalanced image classification''
\vspace{1pc}
\end{center}
\end{@twocolumnfalse}
]

For clarity and completeness, this supplementary material provides the followings:
\begin{itemize}
\item Section~\ref{sup:proof} presents the proof of Theorem~1;
\item Section~\ref{sup:hyperparameter} details our hyperparameter settings;
\item Section~\ref{sup:more results} reports precision and recall results for different methods on a benchmark dataset, and confusion matrices for ECL on the medical image datasets.
\end{itemize}

\section{Proofs for Theorem~1}
\label{sup:proof}

In this section, we derive the lower bound for BC-ECL (4).
Following \cite{cp1, BCL}, we aim to decompose the loss into \emph{attraction} and  \emph{repulsion} terms.
(For notational brevity, we set $\tau=1$.)
We first rewrite the overall BC-ECL los as follows:
\begingroup
\setlength{\thinmuskip}{1.5mu}
\setlength{\medmuskip}{2mu plus 1mu minus 2mu}
\setlength{\thickmuskip}{2.5mu plus 2.5mu}
\fontsize{9.5pt}{11.4pt}\selectfont
\ea{
&~\mathcal{L}_{\text{BC-ECL}}
\nn
\\
&=
-  \frac{1}{|\mathcal{B}|}  \sum_{i \in \mathcal{B}} \frac{1}{|\mathcal{P}(i)|} 
\sum\limits_{j\in \mathcal{P}(i)}
\!\log\!
 \frac{\exp( (\mathbf{z}_{i}\cdot \mathbf{z}_{j} +  \mathbf{z}_{i}\cdot \mathbf{z}_{p_{y_i}}) / 2 )}{\sum\limits_{c \in \mathcal{Y}}\frac{1}{|\mathcal{B}_c|}\sum\limits_{k\in \mathcal{B}_c} \exp(\mathbf{z}_{i}\cdot \mathbf{z}_{k} ) + \exp(\mathbf{z}_{i}\cdot \mathbf{z}_{p_c} )}
 \nn
\\
&=
\frac{1}{|\mathcal{B}|} \sum_{i \in \mathcal{B}}
\left(
\frac{1}{|\mathcal{P}(i)|} \sum\limits_{j\in \mathcal{P}(i)} 
\log\left(
\sum\limits_{c \in \mathcal{Y}}\frac{1}{|\mathcal{B}_c|}\sum\limits_{k\in \mathcal{B}_c} \exp(\mathbf{z}_{i}\cdot \mathbf{z}_{k} ) + \exp(\mathbf{z}_{i}\cdot \mathbf{z}_{p_c} )
\right) \right.
 \nn
\\
& \hspace{4pc} 
\left.
- \frac{1}{|\mathcal{P}(i)|} \sum\limits_{j\in \mathcal{P}(i)} 
\frac{\mathbf{z}_{i}\cdot \mathbf{z}_{j} +  \mathbf{z}_{i}\cdot \mathbf{z}_{p_{y_i}}}{2}
\right)
 \nn
\\
&=
\frac{1}{|\mathcal{B}|} \sum_{i \in \mathcal{B}}
\left(
\log\left(
\sum\limits_{c \in \mathcal{Y}}\frac{1}{|\mathcal{B}_c|} \sum\limits_{k\in \mathcal{B}_c} \exp(\mathbf{z}_{i}\cdot \mathbf{z}_{k} ) + \exp(\mathbf{z}_{i}\cdot \mathbf{z}_{p_c} )
\right) \right.
\nn
\\
&\hspace{4pc}
\left.
- \frac{1}{|\mathcal{P}(i)|} \sum\limits_{j\in \mathcal{P}(i)} 
\frac{\mathbf{z}_{i}\cdot \mathbf{z}_{j} +  \mathbf{z}_{i}\cdot \mathbf{z}_{p_{y_i}}}{2}
\right)
 \nn
\\
&
=\frac{1}{|\mathcal{B}|} \sum_{i \in \mathcal{B}}
\left(
\log\left(
\sum\limits_{c \in \mathcal{Y}}\frac{1}{|\mathcal{B}_c|}\sum\limits_{k\in \mathcal{B}_c} \exp(\mathbf{z}_{i}\cdot \mathbf{z}_{k} ) + \exp(\mathbf{z}_{i}\cdot \mathbf{z}_{p_c} )
\right) \right.
 \nn
\\
&\hspace{4pc} 
\left.
-\log \left( \exp \left(\frac{1}{|\mathcal{P}(i)|} \sum\limits_{j\in \mathcal{P}(i)} 
\frac{\mathbf{z}_{i}\cdot \mathbf{z}_{j}+\mathbf{z}_{i}\cdot \mathbf{z}_{p_{y_i}}}{2}
\right)
\right)
\right)
 \nn
\\
&
=
\frac{1}{|\mathcal{B}|}  \sum_{i \in \mathcal{B}}
\log 
\frac{
 \sum\limits_{ c \in \mathcal{Y}}
     \frac{1}{|\mathcal{B}_c|}
    \sum\limits_{k \in \mathcal{B}_c}  
    \exp(\mathbf{z}_i \cdot \mathbf{z}_k) +  \exp(\mathbf{z}_i \cdot \mathbf{z}_{p_c})    
}{
    \exp\big(\frac{1}{2|\mathcal{P}(i)|}\sum\limits_{j \in \mathcal{P}(i) } \big(\mathbf{z}_i \cdot \mathbf{z}_j  + \mathbf{z}_i \cdot \mathbf{z}_{p_{y_i}}
    \big) \big)
},
\label{eq:sup1}
}
\endgroup
where the second inequality holds by $\log(a/b) = \log a - \log b$,
the third inequality holds by $ \sum_{j \in \mathcal{A}} a_i = |\mathcal{A}| a_i$.

Now, we split the numerator $N(i)$ in (\ref{eq:sup1}) into intra-class and inter-class terms:
\begingroup
\allowdisplaybreaks
\setlength{\thinmuskip}{1.5mu}
\setlength{\medmuskip}{2mu plus 1mu minus 2mu}
\setlength{\thickmuskip}{2.5mu plus 2.5mu}
\fontsize{9.5pt}{11.4pt}\selectfont
\ea{
N(i)
&:=
\underbrace{
     \frac{1}{|\mathcal{P}(i)|}
    \sum\limits_{j \in \mathcal{P}(i)} 
    \exp(\mathbf{z}_i \cdot \mathbf{z}_j) +  \exp(\mathbf{z}_i \cdot \mathbf{z}_{p_{y_i}}) 
}_{\text{(a) intra-class} }
 \nn
\\
&
\hspace{1pc}
+
\sum\limits_{ c \in \mathcal{Y} \setminus \{y_i\}}
\underbrace{  
   \frac{1}{|\mathcal{B}_c|}
    \sum\limits_{k \in \mathcal{B}_c}  
    \exp(\mathbf{z}_i \cdot \mathbf{z}_k) +  \exp(\mathbf{z}_i \cdot \mathbf{z}_{p_c}) 
}_{\text{(b) inter-class}}
.
\label{eq:sup2}
}
\endgroup

We obtain the lower bound of the term (a) in (\ref{eq:sup2}) as follows:
\begingroup
\setlength{\thinmuskip}{1.5mu}
\setlength{\medmuskip}{2mu plus 1mu minus 2mu}
\setlength{\thickmuskip}{2.5mu plus 2.5mu}
\fontsize{9.5pt}{11.4pt}\selectfont
\ea{
&
\frac{1}{|\mathcal{P}(i)|}
    \sum\limits_{j \in \mathcal{P}(i)} 
    \exp(\mathbf{z}_i \cdot \mathbf{z}_j) +  \exp(\mathbf{z}_i \cdot \mathbf{z}_{p_{y_i}}) 
 \nn
\\
&
\ge
\frac{2}{|\mathcal{P}(i)|} \sum_{j \in \mathcal{P}(i)}
\exp\left( \frac{\mathbf{z}_i \cdot \mathbf{z}_j + 
\mathbf{z}_i \cdot \mathbf{z}_{p_{y_i}}}{2} 
\right)
 \nn
\\
&
\ge
2\exp\left(
  \frac{1}{|\mathcal{P}(i)|}
  \sum_{j \in \mathcal{P}(i)} 
  \frac{
\mathbf{z}_i \cdot \mathbf{z}_j + 
\mathbf{z}_i \cdot \mathbf{z}_{p_{y_i}}
}{2} 
\right).
\label{eq:sup3}
}
\endgroup
where the first inequality in (\ref{eq:sup3}) holds by the arithmetic mean--geometric mean (AM--GM) inequality, 
i.e., for any positive numbers $x$ and $y$, 
$x + y \ge 2\sqrt{xy}$,
and the second inequality holds by the Jensen's inequality for a convex function, i.e.,
for any real numbers $\{a_i: i \in \mathcal{A} \}$,
$|\mathcal{A}|^{-1} \sum_{i\in\mathcal{A}} \exp(a_i)\ge 
\exp(|\mathcal{A}|^{-1} \sum_{i\in\mathcal{A}} a_i).$
Similarly, we obtain the lower bound of the term (b) in (\ref{eq:sup2}) as follows:
\begingroup
\setlength{\thinmuskip}{1.5mu}
\setlength{\medmuskip}{2mu plus 1mu minus 2mu}
\setlength{\thickmuskip}{2.5mu plus 2.5mu}
\fontsize{9.5pt}{11.4pt}\selectfont
\ea{
&
\frac{1}{|\mathcal{B}_c|}
\sum\limits_{k \in \mathcal{B}_c}  
\exp(\mathbf{z}_i \cdot \mathbf{z}_k) +  \exp(\mathbf{z}_i \cdot \mathbf{z}_{p_c})
\nn
\\
&
\ge 
2\exp \left(
   \frac{1}{|\mathcal{B}_{c}|}
    \sum\limits_{k \in \mathcal{B}_{c}} 
\frac{ \mathbf{z}_i \cdot \mathbf{z}_k + \mathbf{z}_i \cdot \mathbf{z}_{p_{c}} }{2}
\right).
\label{eq:sup4}
}
\endgroup

\noindent

By combining (\ref{eq:sup3}) and (\ref{eq:sup4}), $N(i)$ in (\ref{eq:sup2}) is lower bounded as follows:
\begingroup
\setlength{\thinmuskip}{1.5mu}
\setlength{\medmuskip}{2mu plus 1mu minus 2mu}
\setlength{\thickmuskip}{2.5mu plus 2.5mu}
\fontsize{9.5pt}{11.4pt}\selectfont
\ea{
N(i)
&
\ge
2\exp \left(
   \frac{1}{|\mathcal{P}(i)|}
    \sum\limits_{j \in \mathcal{P}(i)} 
\frac{\mathbf{z}_i \cdot \mathbf{z}_j + \mathbf{z}_i \cdot \mathbf{z}_{p_{y_i}} }{2}
\right)
\nn
\\
&
\hspace{0.7pc}
+
\underbrace{
\sum\limits_{ c \in \mathcal{Y} \setminus \{y_i\}}
2\exp \left(
   \frac{1}{|\mathcal{B}_{c}|}
    \sum\limits_{k \in \mathcal{B}_{c}} 
\frac{ \mathbf{z}_i \cdot \mathbf{z}_k + \mathbf{z}_i \cdot \mathbf{z}_{p_{c}} }{2}
\right)}_{\text{(c)}}
.
\label{eq:sup5}
}
\endgroup
\noindent
By applying the Jensen's inequality (for a convex function) again to the term (c), we obtain its lower bound:
\begingroup
\setlength{\thinmuskip}{1.5mu}
\setlength{\medmuskip}{2mu plus 1mu minus 2mu}
\setlength{\thickmuskip}{2.5mu plus 2.5mu}
\fontsize{9.5pt}{11.4pt}\selectfont
\ea{
&
\sum\limits_{ c \in \mathcal{Y} \setminus \{y_i\}}
2\exp \left(
   \frac{1}{|\mathcal{B}_{c}|}
    \sum\limits_{k \in \mathcal{B}_{c}} 
\frac{ \mathbf{z}_i \cdot \mathbf{z}_k + \mathbf{z}_i \cdot \mathbf{z}_{p_{c}} }{2}
\right)
\nn
\\
&
\ge 2(C-1)
\exp \left(
\frac{1}{(C-1)}\sum_{c \in \mathcal{Y} \setminus \{y_i\}}
\frac{1}{|\mathcal{B}_c|}\sum_{k\in \mathcal{B}_c}
\frac{\mathbf z_i\!\cdot\!\mathbf z_k
+ \mathbf z_i\!\cdot\!\mathbf{z}_{p_c}}{2}
\right).
\label{eq:sup6}
}
\endgroup
By combining the results in (\ref{eq:sup5}) and (\ref{eq:sup6}), we obtain the final lower bound of $N(i)$ in (\ref{eq:sup2}):
\begingroup
\setlength{\thinmuskip}{1.5mu}
\setlength{\medmuskip}{2mu plus 1mu minus 2mu}
\setlength{\thickmuskip}{2.5mu plus 2.5mu}
\fontsize{9.5pt}{11.4pt}\selectfont
\ea{
N(i)
&
\ge
2\exp\left(
  \frac{1}{2|\mathcal{P}(i)|}
  \sum_{j \in \mathcal{P}(i)} 
(\mathbf{z}_i \cdot \mathbf{z}_j + 
\mathbf{z}_i \cdot \mathbf{z}_{p_{y_i}})
\right)
\nn
\\
&
\hspace{0.6pc}
+
2(C-1)
\exp \left(
\frac{1}{2(C-1)}\sum_{c \in \mathcal{Y} \setminus \{y_i\}}
\frac{1}{|\mathcal{B}_c|}\sum_{k\in \mathcal{B}_c}
(\mathbf z_i\!\cdot\!\mathbf z_k
+ \mathbf z_i\!\cdot\!\mathbf{z}_{p_c})
\right).
\label{eq:sup7}
}
\endgroup

Finally, we obtain the lower bound of $\mathcal{L}_{\text{BC-ECL}}$ in (\ref{eq:sup1}).
For simplicity, we define the followings in (\ref{eq:sup7}):
\begingroup
\setlength{\thinmuskip}{1.5mu}
\setlength{\medmuskip}{2mu plus 1mu minus 2mu}
\setlength{\thickmuskip}{2.5mu plus 2.5mu}
\fontsize{9.5pt}{11.4pt}\selectfont
\ea{
A(i) 
&:= 
\frac{1}{2|\mathcal{P}(i)|}
  \sum_{j \in \mathcal{P}(i)} 
(\mathbf{z}_i \cdot \mathbf{z}_j + 
\mathbf{z}_i \cdot \mathbf{z}_{p_{y_i}})
\label{eq:Ai}
\\
R(i)
&:=
\frac{1}{2(C-1)}
\sum_{c \in \mathcal{Y} \setminus \{y_i\}}
\frac{1}{|\mathcal{B}_c|}\sum_{k\in \mathcal{B}_c}
(\mathbf z_i\!\cdot\!\mathbf z_k
+ \mathbf z_i\!\cdot\!\mathbf{z}_{p_c})
\label{eq:Ri}
}
By applying the result in (\ref{eq:sup7}) and the definitions in (\ref{eq:Ai})--(\ref{eq:Ri}) to  (\ref{eq:sup1}),
we derive the following lower bound of $\mathcal{L}_{\text{BC-ECL}}$:
\begingroup
\setlength{\thinmuskip}{1.5mu}
\setlength{\medmuskip}{2mu plus 1mu minus 2mu}
\setlength{\thickmuskip}{2.5mu plus 2.5mu}
\fontsize{9.5pt}{11.4pt}\selectfont
\ea{
&~\mathcal{L}_{\text{BC-ECL}}
\nn
\\
&=
\frac{1}{|\mathcal{B}|}  \sum_{i \in \mathcal{B}}
\log 
\frac{
 \sum\limits_{ c \in \mathcal{Y}}
     \frac{1}{|\mathcal{B}_c|}
    \sum\limits_{k \in \mathcal{B}_c}  
    \exp(\mathbf{z}_i \cdot \mathbf{z}_k) +  \exp(\mathbf{z}_i \cdot \mathbf{z}_{p_c})    
}{
    \exp\big(\frac{1}{2|\mathcal{P}(i)|}\sum\limits_{j \in \mathcal{P}(i) } \big(\mathbf{z}_i \cdot \mathbf{z}_j  + \mathbf{z}_i \cdot \mathbf{z}_{p_{y_i}}
    \big) \big)
}
\quad \text{in (\ref{eq:sup1})}
\nn
\\
&
\ge
\frac{1}{|\mathcal{B}|}  \sum_{i \in \mathcal{B}}
\log 
\frac{
2 \cdot \exp (A(i)) + 2(C-1) \cdot \exp (R(i))
}{
\exp (A(i))}
\nn
\\
& 
\ge
\frac{1}{|\mathcal{B}|}  \sum_{i \in \mathcal{B}}
\log 
\left(
2\left(
1 + (
C-1)\exp (R(i)-A(i))
\right)
\right)
\nn
\\
&
\ge
\frac{1}{|\mathcal{B}|}  \sum_{i \in \mathcal{B}}
\log2+
\log 
\left(
1 + (C-1)
\exp (R(i)-A(i)
)
\right)
\nn
\\
&
\ge
\frac{1}{|\mathcal{B}|}  \sum_{i \in \mathcal{B}}
\log 
\left(
1 + (C-1)
\exp (R(i)-A(i)
)
\right)
\label{eq:sup8}
}
\endgroup
Substituting (\ref{eq:Ai})--(\ref{eq:Ri}) into (\ref{eq:sup8}) completes the proofs of Theorem~1.

\section{Hyperparameter settings}
\label{sup:hyperparameter}

\begin{table}[t]
\centering
\caption{Top-1 accuracy comparisons of ECL under varying balancing parameter combinations with CIFAR-10-LT ($\rho=100$).}
\label{tab:hyperparameters}
\resizebox{6cm}{!}{%
\begin{tabular}{ccc|c}
\toprule
\hline
$\lambda_{\text{BC-ECL}}$ & $\lambda_{\text{LC}}$ & $\lambda_{\text{CC-GE}}$ & CIFAR-10-LT  \\ \hline
0.5 & 1   & 1 & 91.9          \\
0.5 & 0.5 & 1 & 91.5          \\
0.5 & 1   & 3 & 92.0          \\
0.5 & 0.5 & 3 & \textbf{92.1} \\
1   & 1   & 1 & 91.3          \\
1   & 0.5 & 1 & 91.7          \\
1   & 1   & 3 & 91.8          \\
1   & 0.5 & 3 & 91.6          \\ \hline
\bottomrule
\end{tabular}%
}
\end{table}

To tune the hyperparameters $\{\lambda_\text{BC-ECL}, \lambda_\text{CC-GE}, \lambda_\text{LC}\}$ of the overall ECL loss (8), 
we selected their values such that the three weighted loss terms 
$\lambda_\text{BC-ECL}\mathcal{L}_{\text{BC-ECL}}$, 
$\lambda_\text{CC-GE}\mathcal{L}_{\text{CC-GE}}$, and 
$\lambda_\text{LC}\mathcal{L}_{\text{LC}}$ 
have similar magnitudes at the first training epoch on the CIFAR-10-LT ($\rho=100$) benchmark dataset. 
We then applied this set of hyperparameters consistently across all datasets. 
The hyperparameters selected from our simple hyperparameter tuning scheme yielded stable and reliable performance across all datasets.

Table~\ref{tab:hyperparameters} reports the top-1 accuracy of ECL on CIFAR-10-LT ($\rho=100$) under different balancing parameter combinations 
$\{\lambda_{\text{BC-ECL}}, \lambda_{\text{LC}}, \lambda_{\text{CC-GE}}\}$.
Overall, the performance remains stable across the tested combinations, and 
we obtained the best result with 
$\{\lambda_\text{BC-ECL}=0.5,\ \lambda_\text{LC}=0.5,\ \lambda_\text{CC-GE}=3\}$.
Other combinations yielded comparable accuracy, indicating that ECL is not overly sensitive to the choice of these balancing parameters.
This sensitivity analysis demonstrates that ECL is robust to a wide range of parameter choices and consistently outperforms existing state-of-the-art supervised methods on CIFAR-10-LT ($\rho=100$) across all combinations we evaluated.

\section{More experimental results}
\label{sup:more results}

Section~\ref{sup:precision&recall} compares precision and recall between different supervised CL methods on the CIFAR-10-LT benchmark dataset.
Section~\ref{sup:confusion} presents confusion matrix results of ECL with the two medical image datasets, LCCT and ISIC 2019.
Section~\ref{sup:pr&roc} includes receiver operating characteristic (ROC) and precision-recall (PR) curves of ECL with the two medical image datasets, LCCT and ISIC 2019.

\subsection{Comparisons of precision and recall between different supervised CL methods}
\label{sup:precision&recall}

\begin{table}[t]
\centering
\caption{Comparisons of precision and recall with CIFAR-$10$-LT dataset ($\rho=100$).}
\label{tab:prec&rec}
\resizebox{0.55\columnwidth}{!}{%
\begin{tabular}{ccc}
\toprule
\hline
Methods             & Precision      & Recall         \\ \hline
BCL \cite{BCL}                & 86.65          & 84.14          \\
PaCO \cite{PaCo}               & 77.21          & 73.72          \\
GPaCo \cite{GPaco}              & 79.73          & 77.56          \\ \hline
\textbf{ECL} & \textbf{94.41} & \textbf{91.60} \\ \hline
\bottomrule
\end{tabular}%
}
\end{table}

Table~\ref{tab:prec&rec} reports precision and recall on the CIFAR-$10$-LT benchmark, comparing ECL with supervised CL baselines, BCL \cite{BCL}, PaCo \cite{PaCo}, and GPaCo \cite{GPaco}, to assess their robustness to class imbalance. 
The proposed method achieved the highest precision and recall among all baselines, demonstrating more balanced performance across both majority and minority classes. 
These results suggest that ECL is more robust to class imbalance than existing supervised CL methods.

\subsection{Confusion matrices for ECL}
\label{sup:confusion}

\begin{figure}[t]
\centering
\small\addtolength{\tabcolsep}{-5pt}
\renewcommand{\arraystretch}{1}

\begin{tabular}{cc}
    \centering
    \includegraphics[width=0.5\columnwidth]{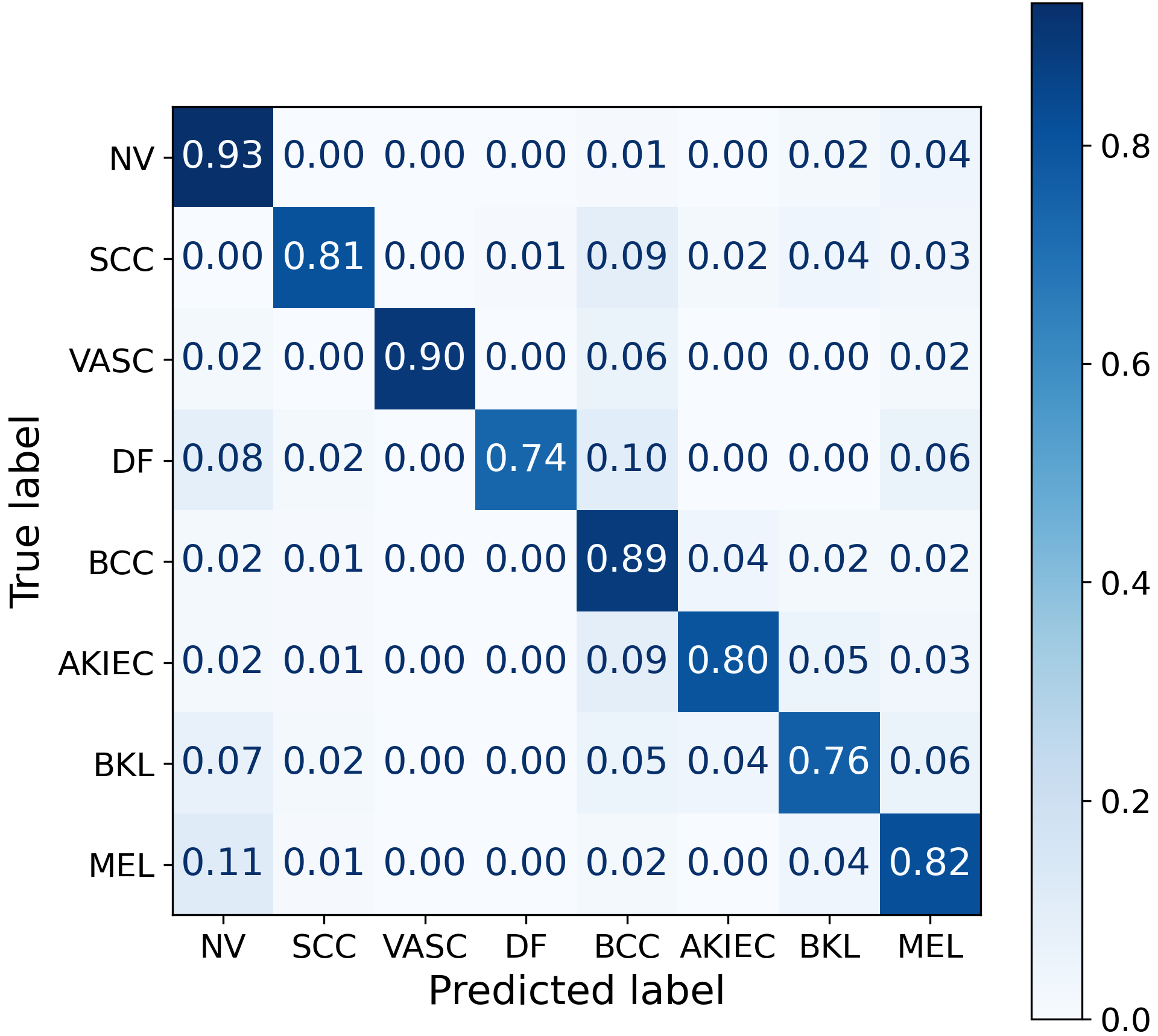} & 
    \includegraphics[width=0.5\columnwidth]{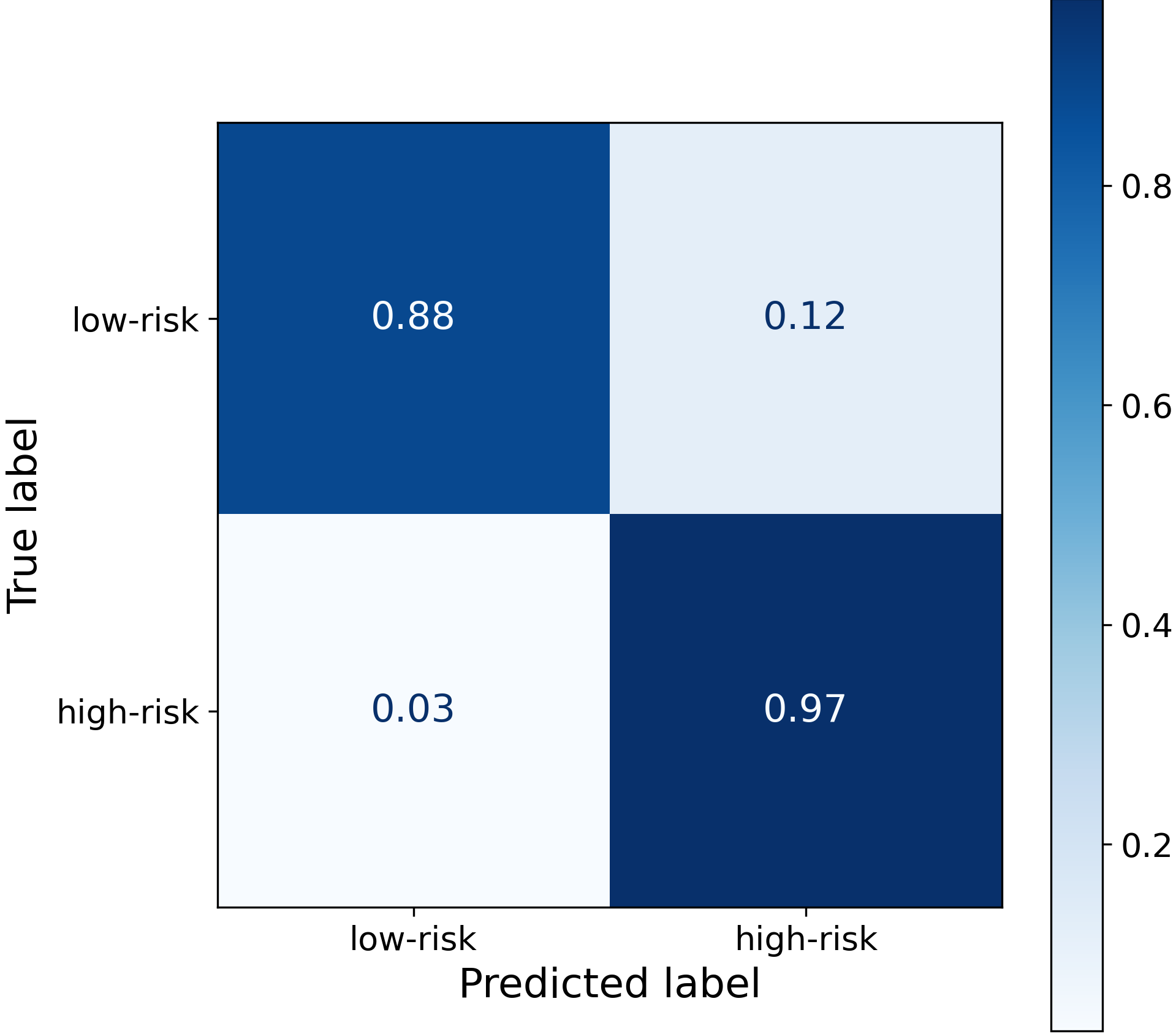} 
    \\
    (a) ISIC 2019 & (b) LCCT \\
\end{tabular}
  \caption{Confusion matrices for ECL on two different medical image datasets. The bar on the right indicates the classification accuracy.}
\label{fig:confusion}
\end{figure}

\begin{figure}[t]
\centering
\small\addtolength{\tabcolsep}{-5pt}
\renewcommand{\arraystretch}{1}

\begin{tabular}{c}
    \includegraphics[width=0.8\columnwidth]{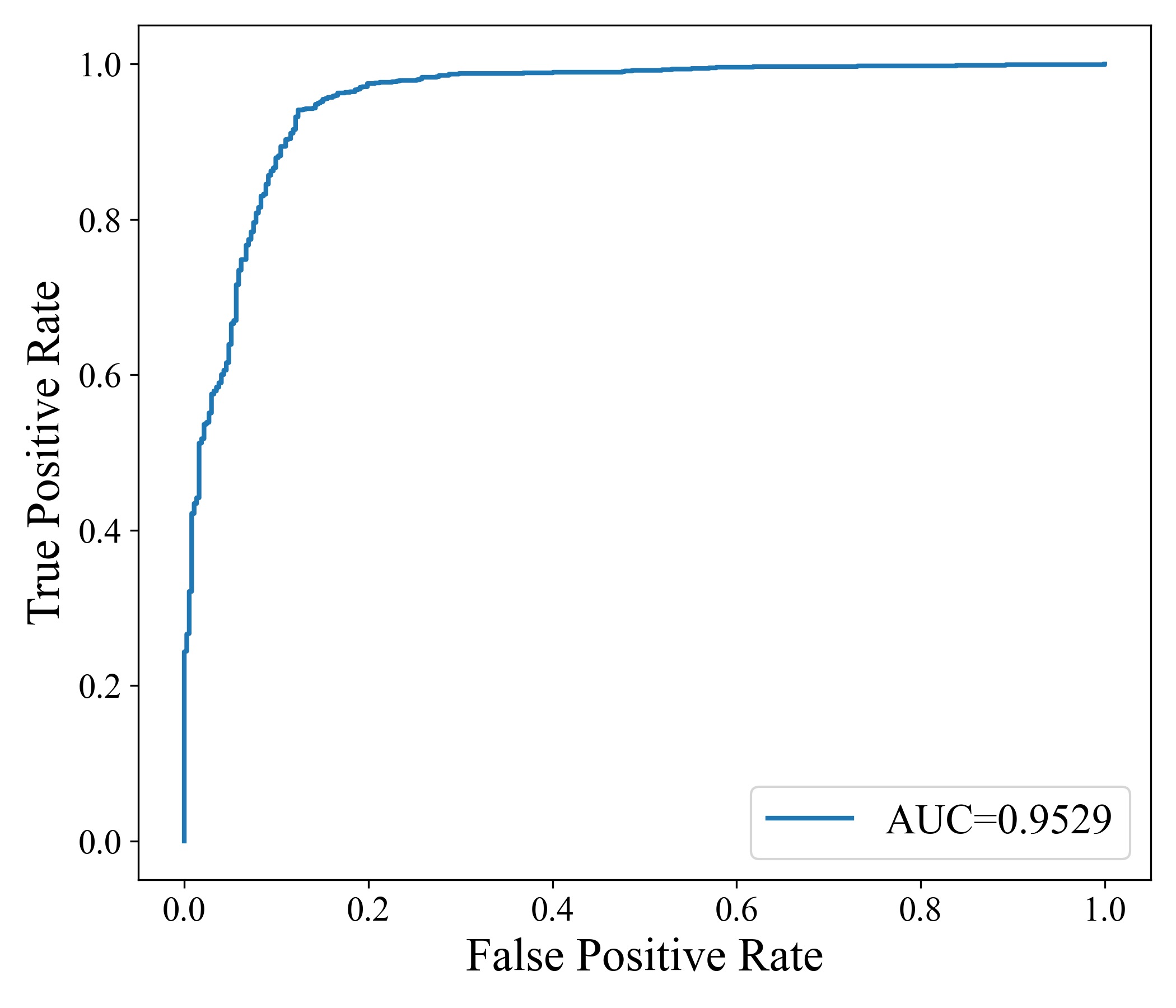} \vspace{-0.5pc} \\
    (a) ROC curve \\
    \includegraphics[width=0.8\columnwidth]{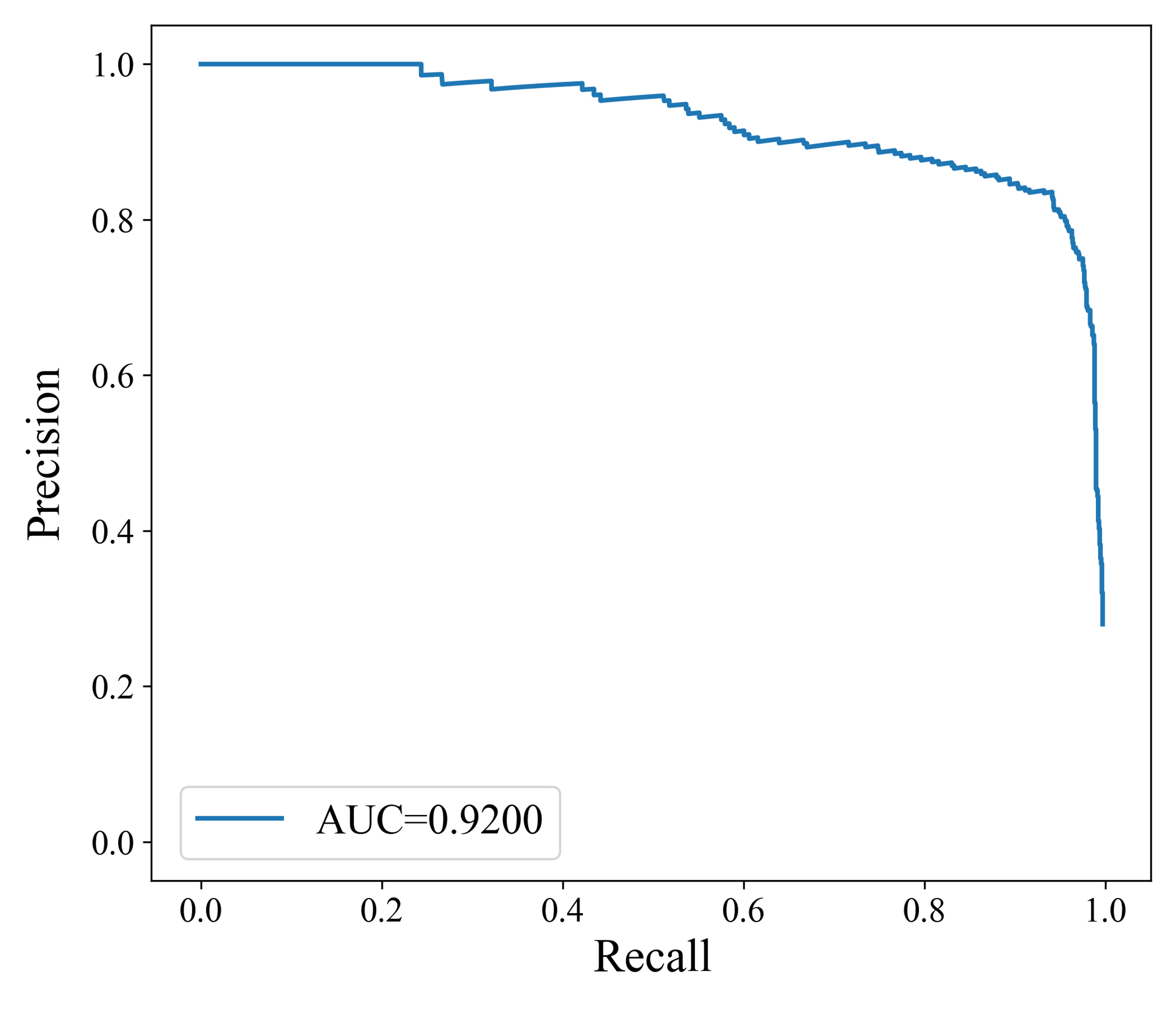} \vspace{-0.5pc} \\
    (b) PR curve
\end{tabular}

\caption{ROC and PR curves of ECL with the LCCT dataset.}
\label{fig:pr_roc_lcct}
\end{figure}

\begin{figure}[t]
\centering
\small\addtolength{\tabcolsep}{-5pt}
\renewcommand{\arraystretch}{1}

\begin{tabular}{c}
    \includegraphics[width=0.8\columnwidth]{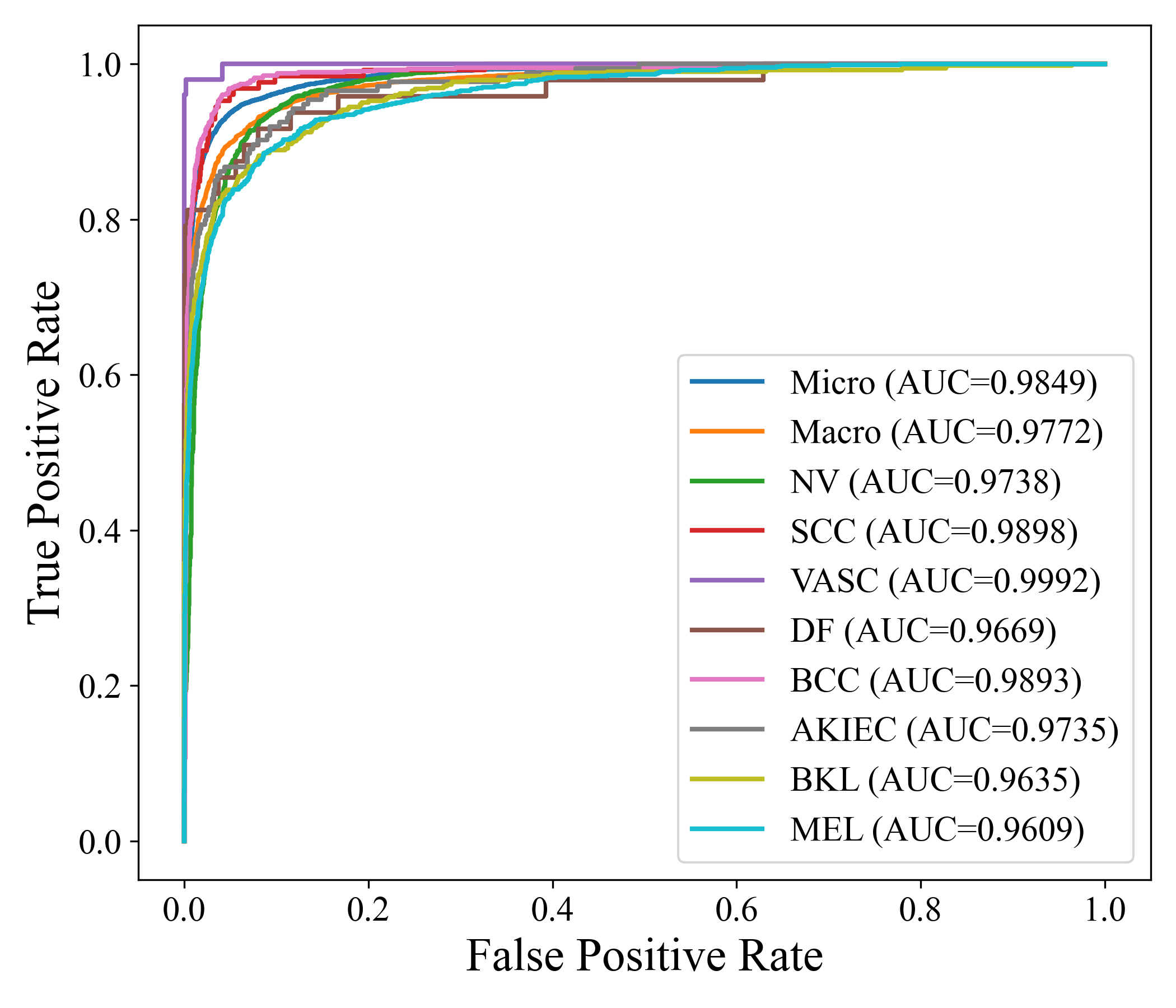} \vspace{-0.5pc} \\
    (a) ROC curve \\
    \includegraphics[width=0.8\columnwidth]{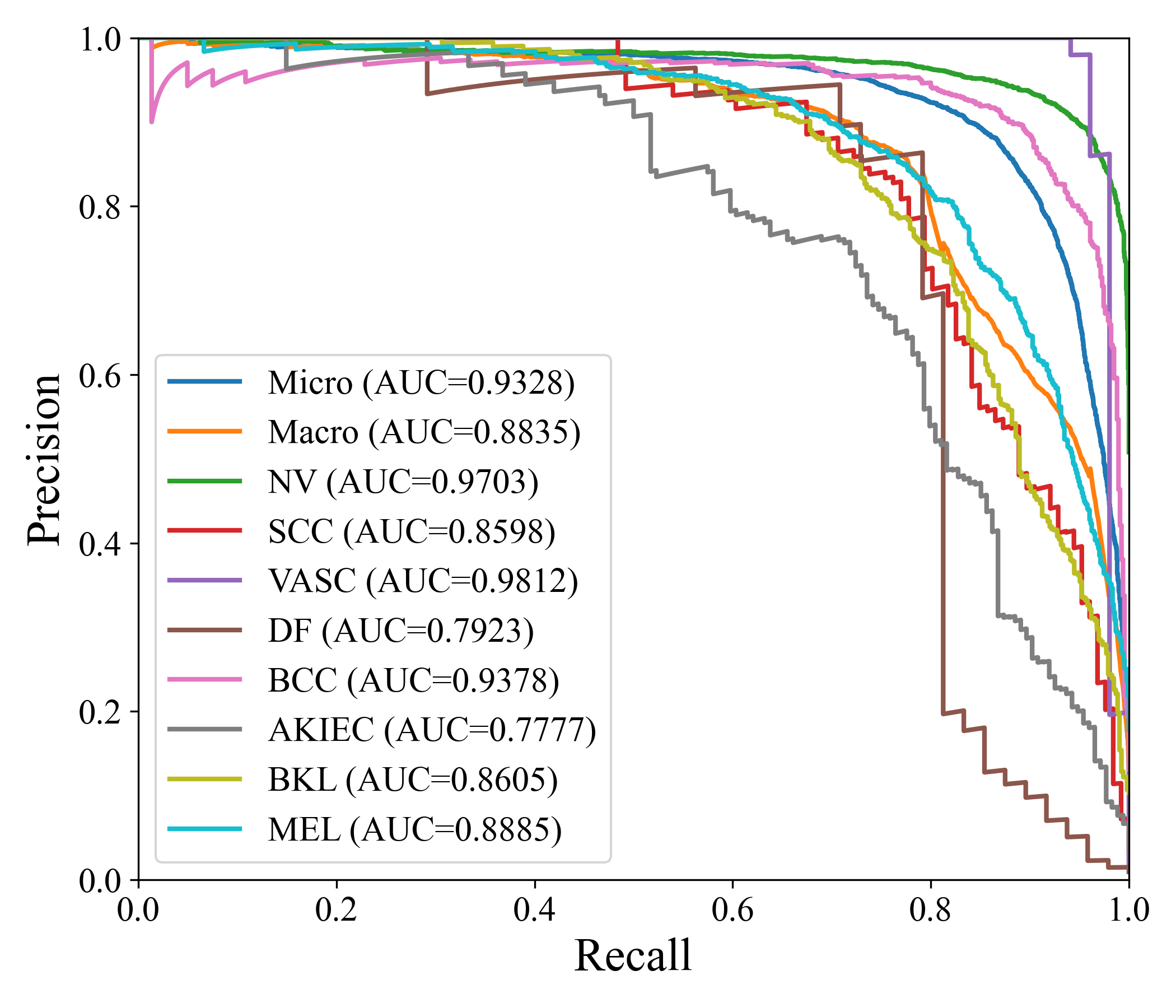} \vspace{-0.5pc} \\
    (b) PR curve
\end{tabular}

\caption{ROC and PR curves of ECL with the ISIC 2019 dataset.}
\label{fig:pr_roc_isic}
\end{figure}

Fig.~\ref{fig:confusion} visualizes the confusion matrices for the LCCT and ISIC 2019 datasets to examine whether ECL achieves balanced classification under class imbalance. 
In LCCT, the high‐risk class is the majority class and the low‐risk class is the minority class.
In ISIC 2019, \textsf{NV} is the only majority class, whereas the remaining classes--- \textsf{SCC}, \textsf{VASC}, \textsf{DF}, \textsf{BCC}, \textsf{AKIEC}, \textsf{BKL}, and \textsf{MEL}---serve as minority classes.

Across both datasets, 
the diagonal entries remain strong for all classes, indicating that ECL maintains balanced performance rather than collapsing toward the majority classes. 
Although minority classes naturally exhibit higher error rates, 
their misclassifications are largely confined to a few semantically similar class pairs rather than being distributed arbitrarily. 
This suggests that the remaining errors arise primarily from semantic ambiguity---not from a systematic bias toward majority classes.

Overall, these results show that ECL achieves relatively low false‐negative rates for minority classes while avoiding excessive false positives, 
demonstrating its ability to maintain a stable trade-off between sensitivity and specificity under class imbalance.

\subsection{ROC and PR curves for ECL}
\label{sup:pr&roc}

Fig.~\ref{fig:pr_roc_lcct} and Fig.~\ref{fig:pr_roc_isic} visualizes ROC and PR curves of ECL with the LCCT and ISIC 2019 datasets.
We report ROC curves to assess how well the model separates positive from negative samples across decision thresholds and PR curves to better reflect performance under class imbalance, where false positives can substantially degrade precision for minority classes.
For the multi-class setting, we report per-class curves as well as micro- and macro-averaged scores. 
The micro-averaged score computes a weighted average that reflects class frequency, whereas the macro-averaged score computes an unweighted average over classes, treating all classes equally.

Fig.~\ref{fig:pr_roc_lcct}(a) and Fig.~\ref{fig:pr_roc_lcct}(b) present the ROC and PR curves of ECL with the LCCT dataset, respectively, where the positive class corresponds to the low-risk (minority) class.
In Fig.~\ref{fig:pr_roc_lcct}(a), the ROC curve achieves a high area under the curve (AUC), with the true positive rate remaining high while the false positive rate stays low across a wide range of decision thresholds.
In Fig.~\ref{fig:pr_roc_lcct}(b), the PR curve achieves a high AUC and maintains high precision as recall increases across decision thresholds.
These results indicate that ECL clearly separates low-risk and high-risk classes while effectively limiting false positives under class imbalance.

Fig.~\ref{fig:pr_roc_isic}(a) and Fig.~\ref{fig:pr_roc_isic}(b) present the ROC and PR curves of ECL with the ISIC 2019 dataset, respectively.
Fig.~\ref{fig:pr_roc_isic}(a) shows that the per-class ROC curves achieve high AUC values, with true positive rates remaining high while false positive rates stay low across a wide range of decision thresholds.
Fig.~\ref{fig:pr_roc_isic}(b) shows the per-class PR curves. 
The majority class NV achieves a high AUC value, maintaining strong precision even at high recall. 
Notably, several minority classes such as VASC and BCC achieve comparable or even higher AUC values than NV, 
indicating that ECL can learn discriminative representations for certain minority classes despite class imbalance. 
For more challenging minority classes such as AKIEC and DF, 
precision decreases as recall increases, yet the curves show a gradual degradation rather than an abrupt collapse, suggesting a stable precision–recall trade-off under data imbalance. 
The gap between the micro- and macro-averaged PR curves suggests that ECL performs well on average, while indicating higher misclassification rates in a few challenging minority classes (e.g., AKIEC and DF).
Corresponding to the confusion-matrix analysis in Section~\ref{sup:confusion}, these misclassifications are largely concentrated in a few semantically similar class pairs, suggesting semantic ambiguity rather than a systematic bias toward the majority class.

Overall, the proposed ECL model effectively addresses data imbalance by limiting false positives through high precision in both binary and multi-class classification tasks, 
while still leaving room for improvement for some challenging minority classes in the multi-class settings.
These results suggest potential clinical utility for real-world medical applications.

\end{document}